\documentclass{article}

\usepackage{etoolbox}
\usepackage{comment}
\newtoggle{icml}
\togglefalse{icml}
\newcommand{\icml}[1]{\iftoggle{icml}{#1}{}}
\newcommand{\arxiv}[1]{\iftoggle{icml}{}{#1}}
\newcommand{\icmlfallback}{%
\newcommand{\icmltitlerunning}[1]{}%
\newcommand{\icmltitle}[1]{\title{##1}}%
\newcommand{\icmlsetsymbol}[2]{}%
\newenvironment{icmlauthorlist}{}{}%
\newcommand{\icmlauthor}[2]{}%
\newcommand{\icmlaffiliation}[2]{}%
\newcommand{\icmlcorrespondingauthor}[2]{}%
\newcommand{\icmlkeywords}[1]{}%
\newcommand{\printAffiliationsAndNotice}[1]{\author{}\maketitle}%
}
\icml{\IfFileExists{icml2025.sty}{\usepackage{icml2025}}{\IfFileExists{icml2026.sty}{\usepackage{icml2026}}{\icmlfallback}}}

\arxiv{
\usepackage[left=1in, right=1in, top=1in,
bottom=1in]{geometry}
  \usepackage{parskip}
}

\icml{
}

\PassOptionsToPackage{hypertexnames=false}{hyperref}  %

\arxiv{
\usepackage[dvipsnames]{xcolor}
}
\usepackage[colorlinks=true, linkcolor=blue!70!black, citecolor=blue!70!black,urlcolor=black,breaklinks=true]{hyperref}
\icml{
  \PassOptionsToPackage{dvipsnames}{xcolor}
}

\usepackage{amsmath}

\usepackage{microtype}
\usepackage{hhline}

\usepackage{amsthm}
\usepackage{mathtools}
\usepackage{bbm}
\usepackage{amsfonts}
\usepackage{amssymb}
\usepackage[nameinlink,capitalize]{cleveref}

\makeatletter
\newcommand{\neutralize}[1]{\expandafter\let\csname c@#1\endcsname\count@}
\makeatother

\arxiv{
\usepackage{algorithm}
}

\arxiv{
\usepackage{natbib}
\bibliographystyle{plainnat}
\bibpunct{(}{)}{;}{a}{,}{,}
}

\usepackage{xpatch}

\usepackage{thm-restate}
\usepackage{autonum}
\declaretheorem[name=Theorem,parent=section]{theorem}
\declaretheorem[name=Lemma,parent=section, numberlike=theorem]{lemma}
\declaretheorem[name=Assumption, parent=section, numberlike=theorem]{assumption}
\declaretheorem[name=Definition, parent=section, numberlike=theorem]{definition}

\declaretheorem[name=Claim, parent=section, numberlike=theorem]{claim}

\declaretheorem[name=Remark, parent=section, numberlike=theorem]{remark}
\declaretheorem[name=Proposition, parent=section, numberlike=theorem]{proposition}
\declaretheorem[name=Fact, parent=section, numberlike=theorem]{fact}

\usepackage{crossreftools}
\makeatletter
  \renewenvironment{proof}[1][Proof]%
  {%
   \par\noindent{\bfseries\upshape {#1.}\ }%
  }%
  {\qed\newline}
  \makeatother

\xpatchcmd{\proof}{\itshape}{\normalfont\proofnameformat}{}{}
\newcommand{\proofnameformat}{\bfseries}

\newcommand{\pref}[1]{\cref{#1}}
\newcommand{\pfref}[1]{Proof of \pref{#1}}

\renewcommand{\eqref}[1]{\texorpdfstring{\hyperref[#1]{(\ref*{#1})}}{(\ref*{#1})}}
\crefformat{equation}{#2Eq.\,(#1)#3}
\Crefformat{equation}{#2Eq.\,(#1)#3}

\Crefformat{figure}{#2Figure~#1#3}
\Crefformat{assumption}{#2Assumption~#1#3}

\Crefname{assumption}{Assumption}{Assumptions}

\crefname{fact}{Fact}{Facts}

\Crefformat{figure}{#2Figure #1#3}
\Crefformat{assumption}{#2Assumption #1#3}

\usepackage{xparse}

\ExplSyntaxOn
\DeclareDocumentCommand{\XDeclarePairedDelimiter}{mm}
 {
  \__egreg_delimiter_clear_keys: %
  \keys_set:nn { egreg/delimiters } { #2 }
  \use:x %
   {
    \exp_not:n {\NewDocumentCommand{#1}{sO{}m} }
     {
      \exp_not:n { \IfBooleanTF{##1} }
       {
        \exp_not:N \egreg_paired_delimiter_expand:nnnn
         { \exp_not:V \l_egreg_delimiter_left_tl }
         { \exp_not:V \l_egreg_delimiter_right_tl }
         { \exp_not:n { ##3 } }
         { \exp_not:V \l_egreg_delimiter_subscript_tl }
       }
       {
        \exp_not:N \egreg_paired_delimiter_fixed:nnnnn 
         { \exp_not:n { ##2 } }
         { \exp_not:V \l_egreg_delimiter_left_tl }
         { \exp_not:V \l_egreg_delimiter_right_tl }
         { \exp_not:n { ##3 } }
         { \exp_not:V \l_egreg_delimiter_subscript_tl }
       }
     }
   }
 }

\keys_define:nn { egreg/delimiters }
 {
  left      .tl_set:N = \l_egreg_delimiter_left_tl,
  right     .tl_set:N = \l_egreg_delimiter_right_tl,
  subscript .tl_set:N = \l_egreg_delimiter_subscript_tl,
 }

\cs_new_protected:Npn \__egreg_delimiter_clear_keys:
 {
  \keys_set:nn { egreg/delimiters } { left=.,right=.,subscript={} }
 }

\cs_new_protected:Npn \egreg_paired_delimiter_expand:nnnn #1 #2 #3 #4
 {%
  \mathopen{}
  \mathclose\c_group_begin_token
   \left#1
   #3
   \group_insert_after:N \c_group_end_token
   \right#2
   \tl_if_empty:nF {#4} { \c_math_subscript_token {#4} }
 }
\cs_new_protected:Npn \egreg_paired_delimiter_fixed:nnnnn #1 #2 #3 #4 #5
 {
  \mathopen{#1#2}#4\mathclose{#1#3}
  \tl_if_empty:nF {#5} { \c_math_subscript_token {#5} }
 }
\ExplSyntaxOff

\XDeclarePairedDelimiter{\supnorm}{
  left=\lVert,
  right=\rVert,
  subscript=\infty
  }

\newcommand{\noah}{\ngcomment}

\newcommand{\BR}{\mathbb{R}}
\newcommand{\E}{\mathbb{E}}
\newcommand{\Zhat}{\widehat{Z}}
\newcommand{\What}{\widehat{W}}

\usepackage[utf8]{inputenc} %
\usepackage[T1]{fontenc}    %
\usepackage{url}            %
\usepackage{booktabs}       %
\usepackage{amsfonts}       %
\usepackage{nicefrac}       %
\usepackage{microtype}      %
\usepackage{graphicx}       %

\usepackage{tocloft}            %

\usepackage{makecell}
\usepackage{amsmath}
\usepackage{amsthm}

\usepackage{upgreek}

\usepackage{enumitem}

\usepackage{breakcites}
\usepackage{mathrsfs}
\arxiv{
\usepackage{algorithm}
\usepackage{verbatim}
\usepackage[noend]{algpseudocode}

}

\usepackage{multicol}

\usepackage{colortbl}
\usepackage{subcaption}
\usepackage{setspace}

\usepackage{transparent}

\usepackage{inconsolata}
\usepackage[scaled=.90]{helvet}
\usepackage{xspace}

\usepackage{pifont}
\usepackage{bm}

\usepackage{tablefootnote}

\DeclareFontFamily{U}{jkpmia}{}
\DeclareFontShape{U}{jkpmia}{m}{it}{<->s*jkpmia}{}
\DeclareFontShape{U}{jkpmia}{bx}{it}{<->s*jkpbmia}{}
\DeclareMathAlphabet{\mathfrak}{U}{jkpmia}{m}{it}
\SetMathAlphabet{\mathfrak}{bold}{U}{jkpmia}{bx}{it}

\DeclarePairedDelimiter{\abs}{\lvert}{\rvert} %
\DeclarePairedDelimiter{\brk}{[}{]}
\DeclarePairedDelimiter{\crl}{\{}{\}}
\DeclarePairedDelimiter{\prn}{(}{)}

\DeclarePairedDelimiter{\floor}{\lfloor}{\rfloor}

\let\Pr\undefined

\DeclareMathOperator{\En}{\mathbb{E}}

\DeclareMathOperator{\Pr}{Pr}

\newcommand{\wt}[1]{\widetilde{#1}}
\newcommand{\wh}[1]{\widehat{#1}}

\def\ddefloop#1{\ifx\ddefloop#1\else\ddef{#1}\expandafter\ddefloop\fi}
\def\ddef#1{\expandafter\def\csname bb#1\endcsname{\ensuremath{\mathbb{#1}}}}
\ddefloop ABCDEFGHIJKLMNOPQRSTUVWXYZ\ddefloop
\def\ddefloop#1{\ifx\ddefloop#1\else\ddef{#1}\expandafter\ddefloop\fi}
\def\ddef#1{\expandafter\def\csname b#1\endcsname{\ensuremath{\mathbf{#1}}}}
\ddefloop ABCDEFGHIJKLMNOPQRSTUVWXYZ\ddefloop
\def\ddef#1{\expandafter\def\csname sf#1\endcsname{\ensuremath{\mathsf{#1}}}}
\ddefloop ABCDEFGHIJKLMNOPQRSTUVWXYZ\ddefloop
\def\ddef#1{\expandafter\def\csname c#1\endcsname{\ensuremath{\mathcal{#1}}}}
\ddefloop ABCDEFGHIJKLMNOPQRSTUVWXYZ\ddefloop
\def\ddef#1{\expandafter\def\csname h#1\endcsname{\ensuremath{\widehat{#1}}}}
\ddefloop ABCDEFGHIJKLMNOPQRSTUVWXYZ\ddefloop
\def\ddef#1{\expandafter\def\csname hc#1\endcsname{\ensuremath{\widehat{\mathcal{#1}}}}}
\ddefloop ABCDEFGHIJKLMNOPQRSTUVWXYZ\ddefloop
\def\ddef#1{\expandafter\def\csname t#1\endcsname{\ensuremath{\widetilde{#1}}}}
\ddefloop ABCDEFGHIJKLMNOPQRSTUVWXYZ\ddefloop
\def\ddef#1{\expandafter\def\csname tc#1\endcsname{\ensuremath{\widetilde{\mathcal{#1}}}}}
\ddefloop ABCDEFGHIJKLMNOPQRSTUVWXYZ\ddefloop
\def\ddefloop#1{\ifx\ddefloop#1\else\ddef{#1}\expandafter\ddefloop\fi}
\def\ddef#1{\expandafter\def\csname scr#1\endcsname{\ensuremath{\mathscr{#1}}}}
\ddefloop ABCDEFGHIJKLMNOPQRSTUVWXYZ\ddefloop

\newcommand{\ind}{\mathbbm{1}}    %

\newcommand{\eps}{\epsilon}

\newcommand{\ldef}{\vcentcolon=}
\newcommand{\rdef}{=\vcentcolon}

\newcommand{\nuhat}{\wh{\nu}}

\newcommand{\Ztil}{\wt{Z}}

\newcommand{\iidsim}{\overset{\text{i.i.d.}}{\sim}}
\newcommand{\Cact}{C_{\sf act}}
\newcommand{\Cchi}{C_{\chi^2}}

\newcommand{\tx}{\tilde{x}}

\newcommand{\einf}{\mathsf{C}_{\infty}}

\newcommand{\dirac}{\updelta}

\newcommand{\piref}{\pi_{\texttt{ref}}}
\newcommand{\pirefs}[1][h]{\pi_{\texttt{ref},#1}}

\newcommand{\model}{\mathsf{M}}
\newcommand{\prompt}{\texttt{p}}
\newcommand{\promptref}{\prompt_{\texttt{ref}}}
\newcommand{\promptstar}{\prompt^{\star}}
\renewcommand{\^}[1]{^{(#1)}}

\newcommand{\Vstar}{V^\star}

\newcommand{\Unif}{\mathsf{Unif}}

\newcommand{\var}{\mathrm{Var}}
\newcommand{\Var}{\var}

\renewcommand{\emptyset}{\varnothing}

\newcommand{\M}[1]{^{{\scriptscriptstyle M}}}  %

\newcommand{\pistar}{\pi^{\star}}

\newcommand{\pihat}{\wh{\pi}}

\renewcommand{\ind}[1]{^{{\scriptscriptstyle#1}}}

\newcommand{\bigoht}{\wt{O}}

\newcommand{\indic}{\mathbb{I}}

\renewcommand{\Pr}{\bbP}

\newcommand{\poly}{\mathrm{poly}}

\newcommand{\Dkl}[2]{D_{\mathsf{KL}}\prn*{#1\,\|\,#2}}

\newcommand{\Dchis}[2]{D_{\chi^2}\prn*{#1\dmid{}#2}}

\newcommand{\Dtv}[2]{D_{\mathsf{TV}}\prn*{#1,#2}}

\newcommand{\Dcov}[3][N]{\mathrm{Cov}_{#1}\prn*{#2\,\|\,#3}}
\newcommand{\Dcovac}[3][N]{\mathrm{Cov}_{#1}^{\mathsf{act}}\prn*{#2\,\|\,#3}}

\newcommand{\Ber}{\mathrm{Ber}}

\newcommand{\dmid}{\;\|\;}

\newcommand{\Vf}{V}

\newcommand{\Vhat}{\wh{\Vf}}

\newcommand{\unif}{\mathrm{unif}}

\newcommand{\supp}{\mathrm{supp}}

\newcommand{\muhat}{\widehat{\mu}}

\newcommand{\Vtil}{\wt{\Vf}}

\def\multiset#1#2{\ensuremath{\left(\kern-.3em\left(\genfrac{}{}{0pt}{}{#1}{#2}\right)\kern-.3em\right)}}

\newcommand{\iid}{i.i.d.\xspace}

\renewcommand{\emptyset}{\varnothing}

\newcommand{\phat}{\wh{p}}

\newcommand{\MA}{\mathcal{A}}

\DeclareMathOperator*{\EE}{\mathbb{E}}
\newcommand{\RR}{\mathbb{R}}
\newcommand{\NN}{\mathbb{N}}

\newcommand{\BP}{\mathbb{P}}

\newcommand{\Alg}{\texttt{Alg}}
\newcommand{\MX}{\mathcal{X}}

\newcommand{\MP}{\mathcal{P}}

\DeclareMathOperator{\Bin}{Bin}

\newcommand{\MS}{\mathcal{S}}

\newcommand{\ystar}{y^\star}

\newcommand{\SMCIND}{\text{SMC-IND}}
\DeclareMathOperator{\Geom}{Geom}

\newcommand{\DMC}{SMC-RS\xspace}
\newcommand{\DMCname}{Sequential Monte Carlo with Rejection Sampling\xspace}

\input{widebar}

\definecolor{ForestGreen}{RGB}{34,139,34}
\definecolor{BurntOrange}{RGB}{34,139,34}
\usepackage[suppress]{color-edits}
 \addauthor{df}{ForestGreen}
 \addauthor{ab}{red}
 \addauthor{dr}{purple}
 \addauthor{ak}{BurntOrange}
 \addauthor{ah}{olive}
 \addauthor{fc}{red}
 \addauthor{ng}{blue}
 \newcommand{\dfc}[1]{}
 \newcommand{\dhruv}[1]{}
 \newcommand{\ah}[1]{}

\arxiv{
\makeatletter
\let\OldStatex\Statex
\renewcommand{\Statex}[1][3]{%
  \setlength\@tempdima{\algorithmicindent}%
  \OldStatex\hskip\dimexpr#1\@tempdima\relax}
\makeatother
}

\icml{
\providecommand{\State}{\STATE}
\providecommand{\For}{\FOR}
\providecommand{\EndFor}{\ENDFOR}
\providecommand{\EndWhile}{\ENDWHILE}
\providecommand{\While}{\WHILE}
}

 \usepackage{accents}

\let\oldparagraph\paragraph

\renewcommand{\paragraph}[1]{\oldparagraph{#1.}}

\makeatletter
\g@addto@macro\appendix{%
  \crefalias{section}{appendixsection}%
  \crefalias{subsection}{appendixsubsection}%
  \crefalias{subsubsection}{appendixsubsubsection}%
}
\makeatother
\crefname{appendixsection}{Appendix}{Appendices}
\Crefname{appendixsection}{Appendix}{Appendices}
\crefname{appendixsubsection}{Appendix}{Appendices}
\Crefname{appendixsubsection}{Appendix}{Appendices}
\crefname{appendixsubsubsection}{Appendix}{Appendices}
\Crefname{appendixsubsubsection}{Appendix}{Appendices}

\arxiv{\title{Reject, Resample, Repeat: \\Understanding Parallel Reasoning in Language Model Inference}}
\arxiv{\author{\begin{tabular}{c}
Noah Golowich\\
{\small Microsoft Research}\\
{\small \texttt{nzg@cs.utexas.edu}}
\end{tabular}\and
\begin{tabular}{c}
Fan Chen\\
{\small MIT}\\
{\small \texttt{fanchen@mit.edu}}
\end{tabular}\and
\begin{tabular}{c}
Dhruv Rohatgi\\
{\small MIT}\\
{\small \texttt{drohatgi@mit.edu}}
\end{tabular}\and
\begin{tabular}{c}
Raghav Singhal\\
{\small NYU}\\
{\small \texttt{singhal.raghav@gmail.com}}
\end{tabular}\and
\begin{tabular}{c}
Carles Domingo-Enrich\\
{\small Microsoft Research}\\
{\small \texttt{carlesd@microsoft.com}}
\end{tabular}\and
\begin{tabular}{c}
Dylan J. Foster\\
{\small Microsoft Research}\\
{\small \texttt{dylanfoster@microsoft.com}}
\end{tabular}\and
\begin{tabular}{c}
Akshay Krishnamurthy\\
{\small Microsoft Research}\\
{\small \texttt{akshaykr@microsoft.com}}
\end{tabular}}}
\arxiv{\date{\today}}

\begin{document} 
\arxiv{\maketitle}

\icml{\icmltitlerunning{Understanding Particle Filtering Algorithms in Large Language Models}}

\icml{
\twocolumn[

  \icmltitle{Understanding Particle Filtering Algorithms in Large Language Models: \\ 
  Sequential Monte Carlo and Beyond}

  \icmlsetsymbol{equal}{*}

  \begin{icmlauthorlist}
    \icmlauthor{Firstname1 Lastname1}{equal,yyy}
    \icmlauthor{Firstname2 Lastname2}{equal,yyy,comp}
    \icmlauthor{Firstname3 Lastname3}{comp}
    \icmlauthor{Firstname4 Lastname4}{sch}
    \icmlauthor{Firstname5 Lastname5}{yyy}
    \icmlauthor{Firstname6 Lastname6}{sch,yyy,comp}
    \icmlauthor{Firstname7 Lastname7}{comp}
    \icmlauthor{Firstname8 Lastname8}{sch}
    \icmlauthor{Firstname8 Lastname8}{yyy,comp}
  \end{icmlauthorlist}

  \icmlaffiliation{yyy}{Department of XXX, University of YYY, Location, Country}
  \icmlaffiliation{comp}{Company Name, Location, Country}
  \icmlaffiliation{sch}{School of ZZZ, Institute of WWW, Location, Country}

  \icmlcorrespondingauthor{Firstname1 Lastname1}{first1.last1@xxx.edu}
  \icmlcorrespondingauthor{Firstname2 Lastname2}{first2.last2@www.uk}

  \icml{\icmlkeywords{Machine Learning, ICML}}

  \vskip 0.3in
]
}

\icml{
\printAffiliationsAndNotice{} }

\begin{abstract}
  Inference-time methods that aggregate and prune multiple samples have emerged as a powerful paradigm for steering large language models, yet we lack any principled understanding of their accuracy--cost tradeoffs. 
In this paper, we introduce a route to rigorously study such approaches using the lens of \emph{particle filtering} algorithms such as Sequential Monte Carlo (SMC). 
  Given a base language model and a \emph{process reward model} estimating expected terminal rewards, we ask: \emph{how accurately can we sample from a target distribution given some number of process reward evaluations?}  
  Theoretically, we identify (1) simple criteria enabling non-asymptotic guarantees for SMC; (2) algorithmic improvements to SMC; and (3) a fundamental limit faced by all particle filtering methods. %
  Empirically, we demonstrate that our theoretical criteria effectively govern the \emph{sampling error} of SMC, though not necessarily its final \emph{accuracy}, suggesting that theoretical perspectives beyond sampling may be necessary.
  
\end{abstract}

\section{Introduction}
\label{sec:intro}
Recent work on large language models (LLMs) has demonstrated the value of \emph{inference-time interventions} that \emph{steer} models towards higher-quality outputs. For instance, existing work has experimented with a combination of using \emph{parallel} generations from LLMs \citep{madaan2025rethinking,brown2024large}, techniques to \emph{aggregate} such generations \citep{wang2023selfconsistency,fu2025deep,zhao2025majority}, and ways to \emph{prune or compress} them to fit in the context window for subsequent model calls \citep{yang2024buffer,wu2025resum}. Even with \emph{no additional training}, such interventions can significantly improve performance on challenging tasks such as mathematical reasoning %
and question answering%
. However, these methods are largely ad hoc, since we lack a unifying theoretical framework that can \emph{explain} the benefits of different inference-time interventions and \emph{guide} algorithm design towards the most effective ones.

\begin{figure}[t!]
\centering
\includegraphics[width=0.3\textwidth]{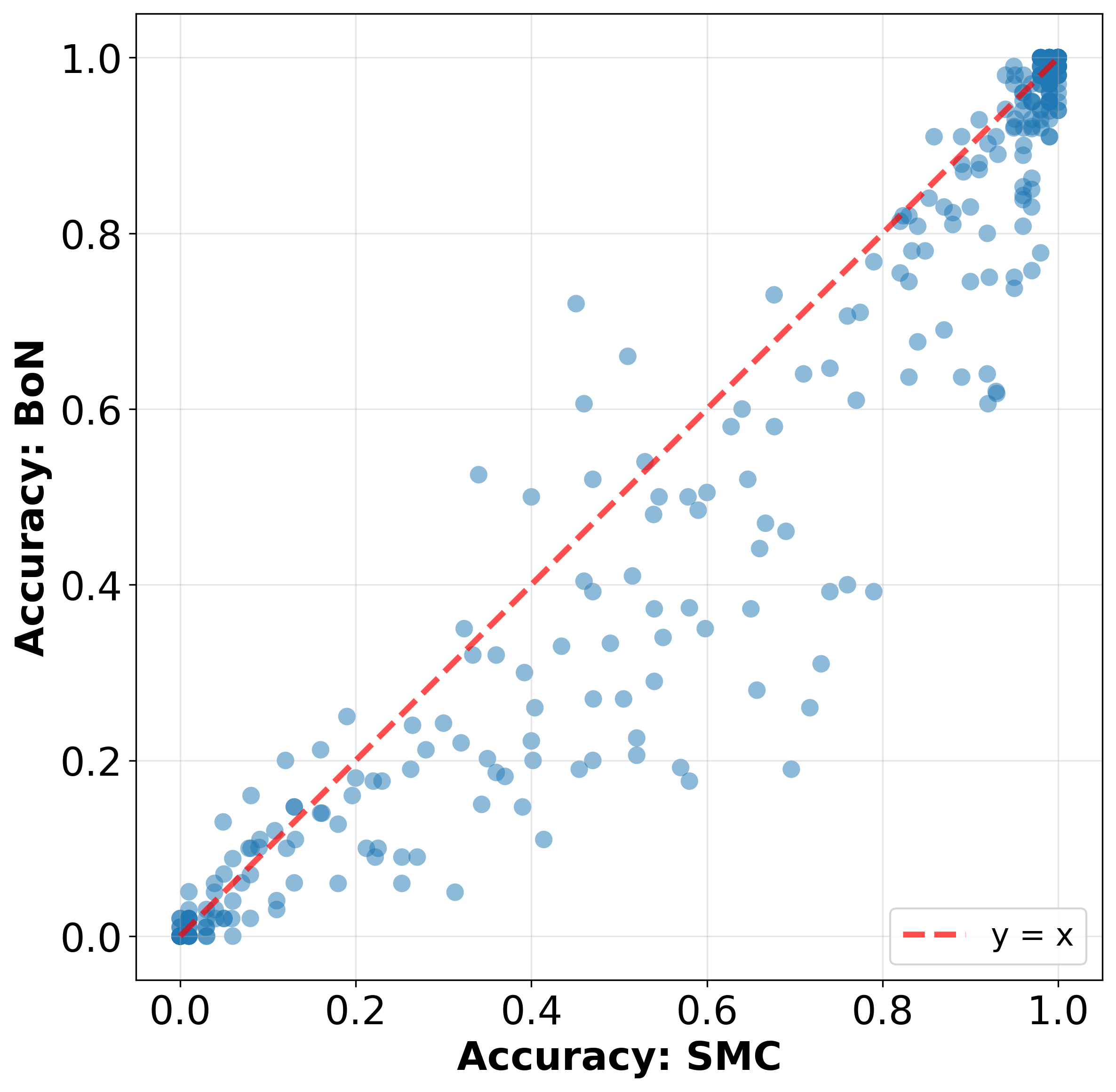}
\caption{Performance of SMC (with $N$ particles) vs. Best-of-$N$ on Math500 problems; here we take $N = 32$. Each point is a different problem; thus, SMC with $N$ particles improves performance over Best-of-$N$ on most problems. See \cref{sec:smc-math-experiments} for details.}%
\label{fig:smc-vs-bon}
\end{figure} 

In this paper, we ask:

\begin{center}
\emph{Can we analyze inference-time interventions such as aggregation and pruning of parallel generations in a principled manner? }%
\end{center}

To do so, we study the paradigm of \emph{guided LLM generation with a process reward model} \citep{wang2025value,puri2025probabilistic}, where the process reward model numerically scores partial generations.  %
In this paradigm, aggregation and pruning are naturally implemented by \emph{particle filtering} algorithms, including the popular \emph{Sequential Monte Carlo (SMC)} method, which use the process reward model to adaptively prune and replicate multiple \emph{particles} (namely, partial generations from the base model), as tokens are generated. These strategies can have robust empirical benefits over simpler interventions such as Best-of-$N$; see e.g. \cref{fig:smc-vs-bon} or \citet{puri2025probabilistic}.%

While particle filtering methods have been extensively applied to tasks ranging from posterior inference \citep{pitt1997filtering} and robot localization \citep{montemerlo2002fastslam} to steering LLMs \citep{lipkin2025fastcontrolled,lew2023smcsteering,loula2025syntactic,grand2025selfsteering,puri2025particlemc,zhao2024twisted,feng2024step} and diffusion models \citep{skreta2025feynman,singhal2025general}, few non-asymptotic \emph{guarantees} have been established for such methods. In this paper, we address this challenge, providing a user-friendly analysis of SMC and identifying benefits and limitations of more sophisticated particle filtering methods. %

\paragraph{Towards a theory of inference-time interventions} Increasingly, inference-time interventions utilize models not just for generation and scoring, but also more complex operations such as splicing \citep{fu2025deep} or summarizing \citep{yu2025limits} generations. While these operations are not directly captured by our setting, we believe that our work is a necessary first step towards developing a framework for understanding (and improving) such interventions. %

\subsection{Setting: Guided Generation with Imperfect Process Rewards}\label{sec:setting}

For a fixed prompt, a language model defines a distribution $\piref \in \Delta(\cA^H)$ over sequences of $H$ \emph{actions} in some set $\cA$.\footnote{For simplicity, we assume that the language model produces generations of fixed length.}
One can think of $\MA$ as the space of \emph{tokens}, or more generally as blocks of tokens. Given a sequence $a_{1:h}$, we can efficiently sample from $\piref(a_{h+1}=\cdot\mid{}a_{1:h})$, which enables autoregressive sampling from $\piref$.%

We model the task of \emph{steering} the language model towards some desired reward as a \emph{sampling} problem, in line with recent work on LLM inference and post-training \citep{rohatgi2025taming,geuter2025guided,foster2025foundation,xiong2024iterative,zhu2026power}. Given query access to a reward function $r^\star : \cA^H \to \BR_{\geq 0}$, we aim to sample from the \emph{tilted distribution} $\pistar_H(a_{1:H}) \propto \piref(a_{1:H}) r^\star(a_{1:H})$. For instance, in the context of mathematical reasoning, $r^\star(a_{1:H}) \in \{0,1\}$ measures whether $a_{1:H}$ correctly solves the prompt. %

\paragraph{Imperfect process rewards} Following \cite{rohatgi2025taming}, we theoretically formalize a process reward model (PRM) as an \emph{approximate value function}. Given a sequence $a_{1:h}$ of any length, the PRM provides an estimate $\Vhat(a_{1:h}) \in \RR_{\geq 0}$ of the 
\emph{expected reward} of $a_{1:h}$ under $\piref$, i.e., $\Vstar(a_{1:h}) := \E_{a'_{1:H} \sim \piref}[r^\star(a'_{1:H}) \mid a'_{1:h}=a_{1:h}]$. 

In practice, there has been extensive effort to train transformers specifically to be PRMs \citep{lightman2023let,wang2024mathshepherd,wang2025valueguided,zhang2025lessons} for certain problems, such as math reasoning. If such a fine-tuned PRM adapted to a desired reward function $r^\star$ is not available, there are several other popular choices for $\Vhat$---e.g. LLM-as-a-judge \citep{stephan2024calculation} or LLM log-probabilities \citep{karan2025reasoning,lew2023smcsteering}; see also \cref{sec:experiments}. %

The key algorithmic challenge in this setting is that the PRM is imperfect. To address this challenge, numerous parallel and sequential approaches have been proposed, including SMC \citep{puri2025probabilistic}, block rejection sampling \citep{mudgal2023controlled}, and backtracking \citep{botta2025on}. Thus far, provable guarantees are only known for backtracking \citep{rohatgi2025taming} --- an inherently sequential approach. We ask: \emph{given access to $\piref$ and $\Vhat$ as above, what structural properties enable approximate sampling from $\pistar_H$ \textbf{via efficient parallel methods such as particle filtering}?}

\subsection{Theoretical Contributions: A Principled Analysis of Particle Filtering Methods for LLM Inference}
\label{sec:theory-overview}

\newcommand{\bfb}[1]{\textbf{\color{blue} #1}}

\paragraph{Contribution I: Simple criteria for success of SMC}

 Our first result identifies two \bfb{key properties} that enable the success of SMC. For $h \in [H]$, let $\pistar_h, \pihat_h \in \Delta(\cA^h)$ denote the \emph{true} and \emph{approximate intermediate target distributions}, defined by $\pistar_h(a_{1:h}) \propto \piref(a_{1:h})\Vstar(a_{1:h})$ and $\pihat_h(a_{1:h}) \propto \piref(a_{1:h}) \Vhat(a_{1:h})$. Then we show the following:

\begin{theorem}[Corollary of \cref{thm:smc-chisq}]
\label{thm:main-intro}
Suppose that the following hold.\\ \bfb{(1)} Bounded \emph{action-level coverage}:  for all $h$, $a_{1:h+1}$, $\pistar(a_{h+1} \mid a_{1:h}) / \piref(a_{h+1} \mid a_{1:h}) \leq \Cact$.\\
\bfb{(2)} Bounded \emph{$\chi^2$-divergences} (these control the error of $\Vhat$ vs. $\Vstar$, per \cref{fact:chisq-avg}): $\Dchis{\pistar_h}{\pihat_h} \leq \Cchi$ for all $h$. 

Then SMC with $N$ particles samples from a distribution $\muhat$ satisfying $\Dtv{\muhat}{\pistar} \leq \sqrt{\frac{H^2 \Cact(\Cchi+1)}{N}}$. %
\end{theorem}
This bound strengthens the guarantees for the backtracking-based algorithm VGB of \citet{rohatgi2025taming}. Moreover, the parallel runtime of SMC is $O(H)$ whereas VGB uses time $\Omega(H^2)$ (and is inherently sequential). See \cref{sec:comparison} for additional discussion. %

\begin{remark}[Connection to literature on refinements of SMC]
Interestingly, the two \textbf{\color{blue} key properties} listed above map quite cleanly onto various efforts in prior work focused on refining SMC algorithms. The latter of the two conditions (that $\Vhat$ approximates $\Vstar$) is related to the line of work on \emph{twisting targets} for SMC. In particular, the distributions $\pihat_h$ introduced above are often referred to as (intermediate) target distributions, and prior work has explored how to choose (or learn) $\Vhat$ so as to bring $\pihat_h$ as close as possible to $\pistar_h$ for all $h \in [H]$ \citep{guarniero2017iterated,heng2017controlled,zhao2024twisted,lawson2022sixo}, which is a natural goal in light of \cref{thm:main-intro}.

Moreover, the first condition (on action-level coverage) is related to the extensive literature on changing the proposal distribution $\pirefs[h]$ to better approximate the \emph{locally optimal proposal distribution} \citep{naesseth2019highdim,doucet2000sequential}, which is defined as the distribution $\pihat_{h+1}(a_{h+1} = \cdot \mid a_{1:h})$. In the event that $\pihat_{h+1} \approx \pistar_{h+1}$, then the locally optimal proposal distribution has action-level coverage close to $1$, the smallest possible value. Thus, as long as $\Vhat$ is sufficiently accurate, modifying $\pirefs$ so as to approximate the locally optimal proposal is a way to decrease action-level coverage.\footnote{Indeed, using $\pihat_{h+1}$ as the proposal and appropriately modifying the weights in SMC yields a variant with improved convergence properties (see \cref{sec:dmc} and \cref{sec:dmc-restart}).} %
 Summarizing, \textbf{our result helps to unify the existing literature on variants of SMC by showing that they are effectively minimizing the key quantities \cref{thm:main-intro}}.
 \end{remark}
 
 \paragraph{Contribution II: Beyond Sequential Monte Carlo} Can particle filtering algorithms beyond SMC achieve stronger guarantees? First, we show that if $\pihat_h$ is close to $\pistar_h$ in a stronger $L_\infty$ sense, then it is possible to achieve \emph{exponential decay} in sampling error (rather than polynomial decay, as in \cref{thm:main-intro}) by wrapping SMC in an outer rejection sampling loop (\cref{thm:smc-linf}). %
 
Second, we observe that even when the PRM is \emph{perfect}, and hence $\Cchi=0$, SMC requires $\Omega(\sqrt{H})$ particles to achieve non-trivial sampling accuracy (\cref{prop:SMC-lower-exact}), even though there is a simple one-particle algorithm for this special case \citep{rohatgi2025taming}. We show that a modified algorithm--- \emph{Sequential Monte Carlo with Rejection Sampling (\DMC, \cref{alg:dmc-implementable})}---avoids this pathology, recovering \cref{thm:main-intro} as well as the one-particle guarantee (\cref{thm:dmc-chisq}).

 \paragraph{Contribution III: Limits of particle filtering} A key limitation of SMC (and our modifications) is that even when $\piref$ has bounded action-level coverage (i.e. $\Cact=O(1)$) and the approximate value function $\Vhat$ is within a constant factor of the truth (and hence $\Cchi=O(1)$), our algorithm requires maintaining $\poly(H)$ particles and hence requires superlinear work. In \cref{thm:pf-lb-main} we show that fully avoiding this lower bound would require \emph{lookahead}: any \emph{myopic} method needs at least $\Omega(\log H/\log \log H)$ particles (\cref{thm:pf-lb-main}). Closing this gap is open.

\paragraph{Addendum: backtracking as a form of particle filtering} In \cref{app:coupling}, we show an interesting connection between particle filtering algorithms and the backtracking-based approach VGB \citep{rohatgi2025taming} --- previously, the only algorithm for inference-time steering of LLMs that was known to be robust to PRM errors. VGB is a (one-particle) Markov chain method that generates tokens but sometimes also \emph{deletes} the most recent token. We show that the execution of VGB can be \emph{coupled} with the execution of a certain particle filtering method, potentially providing a new perspective on the benefits of both methods.%

\subsection{Empirical Contributions: Does the Theory Predict the Performance of SMC in LLMs?}
We investigate empirically whether the quantities identified by \cref{thm:main-intro} indeed predict the performance of SMC on LLM sampling problems. First, we consider a simple ``prompt switching'' task in which the base model $\piref$ corresponds to sampling from an LLM given some prompt $\promptref$, and $\pistar$ corresponds to sampling from that LLM given another prompt $\promptstar$ (see \cref{sec:experiments} for details). By varying $\promptref, \promptstar$, and the PRM $\Vhat$, we can control  for each quantity in \cref{thm:main-intro}. First, in \cref{fig:mean-ratio-intro}, we consider a dataset where $\Vhat = \Vstar$ for all prompts in the dataset (so $\Cchi=0$) but the action level coverage varies, and observe a strong correlation between the action-level coverage (measured via a KL-divergence proxy) and the sampling error of SMC. Second, in \cref{fig:inv-phi-intro}, we consider a dataset where the action-level coverage is identical across all prompts but the degree of accuracy of $\Vhat$ with respect to $\Vstar$ varies, and again we observe a strong correlation between $\Dkl{\pistar_h}{\pihat_h}$ and the sampling error of SMC. (We report the KL divergence instead of the $\chi^2$ divergence since the latter requires prohibitively large sample complexity to estimate.)

\begin{figure}[t!]
\centering
\begin{subfigure}[b]{\icml{0.24}\arxiv{0.49}\textwidth}
\centering
\includegraphics[width=\textwidth]{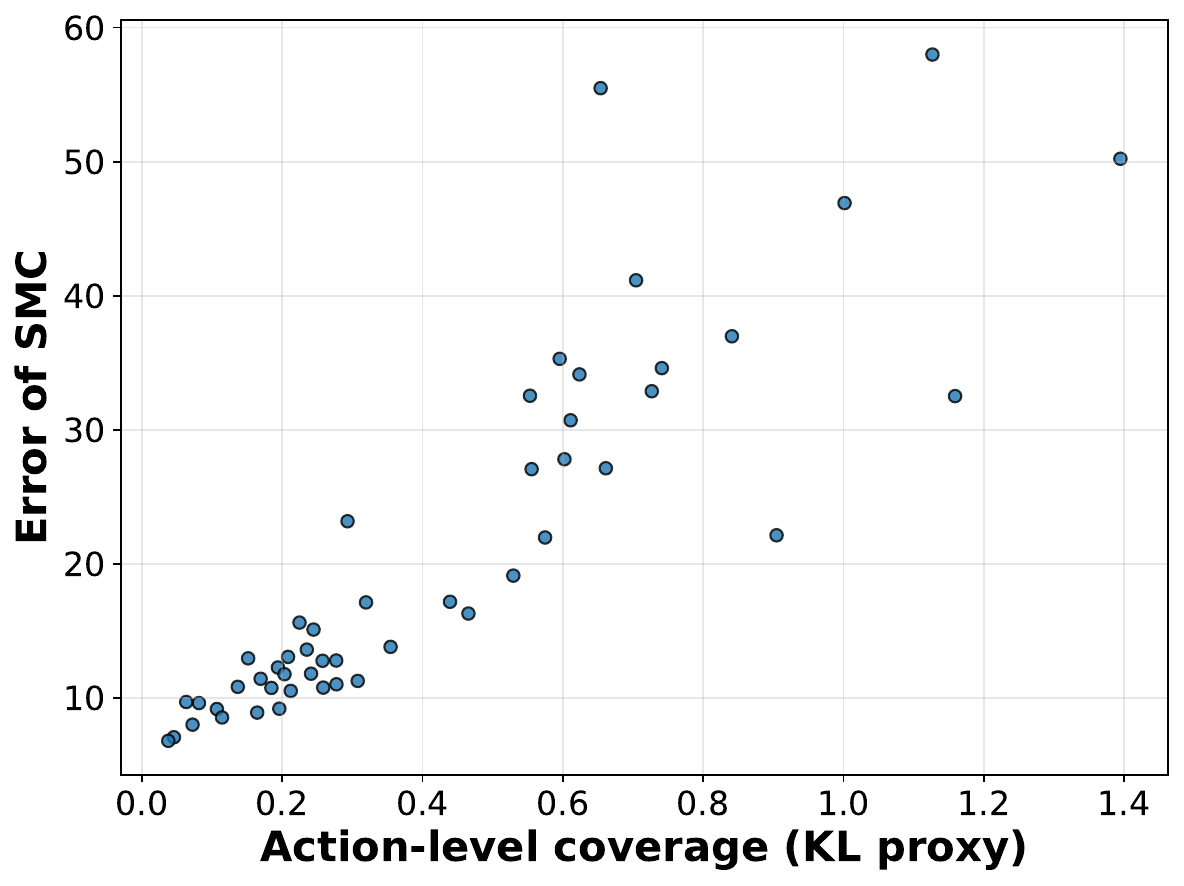}
\caption{Action-level coverage.}
\label{fig:mean-ratio-intro}
\end{subfigure}\hfill
\begin{subfigure}[b]{\icml{0.24}\arxiv{0.49}\textwidth}
\centering
\includegraphics[width=\textwidth]{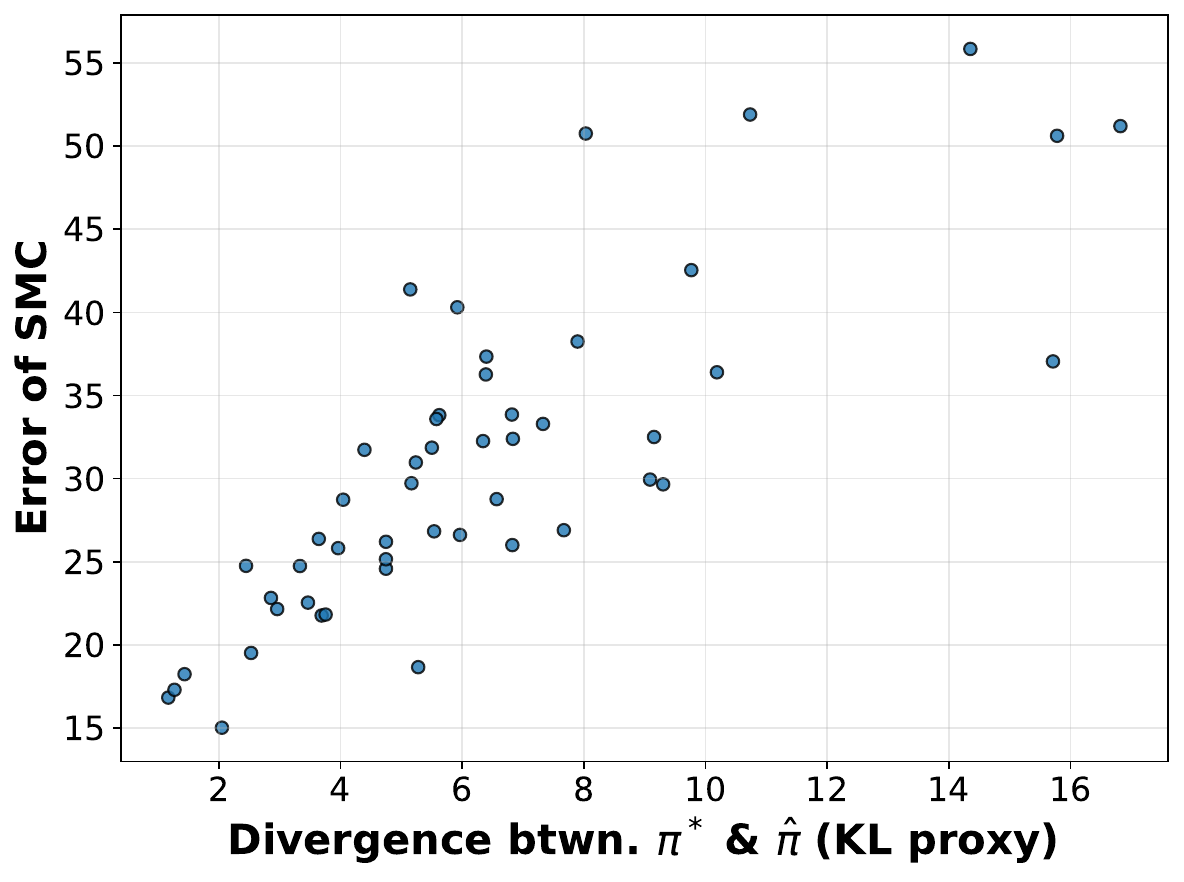}
\caption{$\Dkl{\pistar_h}{\pihat_h}$.}
\label{fig:inv-phi-intro}
\end{subfigure}
\caption{Empirical validation for our theory on the \emph{prompt-switching task} (see \cref{sec:experiments}). (a) We vary the action-level coverage (as measured by a KL-divergence proxy) across many prompts while keeping $\Vhat = \Vstar$ fixed, and observe that action-level coverage predicts the sampling error of SMC. (b) We fix $\piref = \pistar$ (so that action-level coverage is fixed) and observe that $\Dkl{\pistar_h}{\pihat_h}$ predicts the sampling error of SMC.}
\label{fig:empirical-validation}
\end{figure}

\paragraph{Evaluating SMC on math reasoning tasks} Finally, we consider the performance of SMC on challenging math reasoning benchmarks including AIME and Math500. In \cref{fig:smc-vs-bon}, we plot, for each problem in Math500, the accuracy of SMC with $N$ particles against that of Best-of-$N$ (with $N = 32$); each of these quantities was estimated by running the requisite algorithm 100 times per problem. As most points lie below the line $y = x$, we see that SMC  improves performance over Best-of-$N$ in a \emph{uniform sense over all problems}, which extends prior observations in \cite{puri2025probabilistic,chatziveroglou2025adecodingtokenefficientinferencescaling} which observed that SMC improves upon Best-of-$N$ in an \emph{on-average} sense over such benchmarks. Despite this strong evidence for superiority of SMC, as we discuss further in \cref{sec:smc-math-experiments}, the quality of the PRM $\Vhat$ (as measured by an estimate of the chi-square divergence between $\pihat_h$ and $\pistar_h$) does not seem to positively correlate with performance, as often \emph{larger divergence} leads to \emph{higher accuracy}. \textbf{\emph{We leave as an intriguing open question the development of a more refined framework that captures performance on such benchmarks.}}

\paragraph{Roadmap} \cref{sec:prelim} introduces the formal setting for our theoretical results. \cref{sec:particle-theory} contains our main theoretical results, and \cref{sec:experiments} contains our experimental results. %

\section{Preliminaries}\label{sec:prelim}

Our theory applies to a setting that generalizes \cref{sec:setting}. It should capture settings such as guidance for diffusion models, though we do not investigate such settings concretely. 

Let $\MX_0, \MX_1, \ldots, \MX_H$ be arbitrary sets, and let $\MX := \MX_0 \sqcup \MX_1 \sqcup \cdots \sqcup \MX_H$ be their disjoint union. Let $\piref = (\pirefs)_{h=0}^{H-1}$ be a collection of transition kernels $\pirefs : \MX_h \to \Delta(\MX_{h+1})$. Without loss of generality, assume that $\MX_0 = \{ \perp \}$ is a single element. The transition kernels induces a Markov chain $x_0 \to x_1 \to x_2 \cdots \to x_H$, where $x_{h+1} \sim \pirefs(\cdot \mid x_h)$ for $0 \leq h \leq H-1$; we drop the subscript $h$ when clear from context. We write $x_{1:H} \sim \piref$ to denote a sample from the Markov chain. Let $\pi_h \in \Delta(\MX_h)$ be the induced marginal distribution of $x_h$. 

\begin{remark}[Special case: autoregressive generation]\label{remark:autoregressive} In the specialization to LLMs discussed in \cref{sec:intro}, we would have $\MX_h = \cA^h$, and $\pirefs(\cdot \mid a_{1:h})$ is the transition kernel which samples $a_{h+1}$ (given $a_{1:h}$) and appends it to $a_{1:h}$. 
\end{remark}

As in the setup of \cref{sec:intro}, given a terminal reward function $r^\star : \MX_H \to \BR_{\geq 0}$, the \emph{value function} $\Vstar : \MX\to \BR_{\geq 0}$ is defined, for $x \in \MX_h$, by 
 $\Vstar(x) = \En_{x_{1:H} \sim \piref}[r^\star(x_H) \mid x_h = x]$. %
 Note that $\Vstar(x_H) = r^\star(x_H)$ for $x_H \in \MX_H$. 

For $h \in [H]$, define the distribution $\pistar_h \in \Delta(\MX_h)$ by $\pistar_h(x)\propto \pi_h(x)\cdot \Vstar(x)$. %
One may check that the distributions $\pistar_h$ are induced by the Markov chain with kernel 
\begin{align}
  \pistar(x_{h+1}\mid x_h):=\frac{\pirefs(x_{h+1}\mid x_h) \Vstar(x_{h+1})}{\Vstar_h(x_h)},\label{eq:pistar-kernel}
\end{align}
as $\sum_{x_h} \pistar_h(x_h) \pistar(x_{h+1}\mid x_h) = \pistar_{h+1}(x_{h+1})$.\footnote{We note the following slight abuse of notation: we use $\pistar$ with no subscript to denote the transition kernel in \cref{eq:pistar-kernel}, and $\pistar_h$ with a subscript to denote the corresponding intermediate \emph{distributions}.}

\paragraph{Partial Reward Model}
We extend the definition of the PRM to our more general setting in the natural way: the PRM is denoted $\Vhat : \MX \to \BR_{\geq 0}$, and satisfies $\Vhat(x_H)= r^\star(x_H)$ for $x_H \in \MX_H$. We define $\Zhat_h := \E_{x \sim \pi_h}[\Vhat(x)]$ for $h \in [H]$, and let $\pihat_h\in\Delta(\cX_h)$ be defined by $\pihat_h(x)=\pi_h(x)\cdot \Vhat(x)/\Zhat_h$.

\section{Theoretical Analysis of Particle Filtering Methods}
\label{sec:particle-theory}

\subsection{Simple Criteria for Success of SMC}\label{sec:SMC}

\begin{algorithm}
  \caption{Sequential Monte Carlo}\label{alg:smc}
\begin{algorithmic}
\State \textbf{Input:} Transition kernel $\piref$, value function $\Vhat$.
\State \textbf{Parameter:} Number of particles $N \geq 1$. 
\State Initialize $\nuhat_0=\dirac_{\perp}$.
\For{$h=1,\cdots,H$}
\State Sample $N$ particles $\tx_{h-1}\ind{1},\ldots,\tx_{h-1}\ind{N}\iidsim \nuhat_{h-1}(\cdot{})$.
\State Generate $x_h\ind{i}\sim \pirefs[h-1](\cdot\mid \tx_{h-1}\ind{i})$ independently for $i \in [N]$. 
\State Define weights $w_h\ind{i}=\frac{\Vhat(x_h\ind{i})}{\Vhat(\tx_{h-1}\ind{i})}$.
\State Define weighted empirical measure $\nuhat_h\ldef{} \frac{1}{W_h}\sum_{i=1}^{N} w_h\ind{i}\dirac_{x_h\ind{i}}$, where $W_h = \sum_{i=1}^{N} w_h\ind{i}$.
\EndFor
\State Define $\What_h := \prod_{i=1}^h \frac{W_i}{N}$.
\State \textbf{Output Option 1:} Output $x_H \sim \nuhat_H$. 
\State \textbf{Output Option 2:} With probability $\min\crl*{\frac{\What_H}{ 2\einf\Vhat(\perp)}, 1}$, output $x\sim \nuhat_H$, and otherwise restart the algorithm. 
\end{algorithmic}
\end{algorithm}

The first criterion is bounded \emph{action-level coverage}: 

\begin{assumption}\label{asmp:act-cov}
There exists parameter $\Cact\geq 1$ such that for all $h \geq 0$, $\Vstar_{h+1}(x_{h+1})\leq \Cact \Vstar_h(x_h)$ for any $x_{h+1}\in \supp(\piref(\cdot\mid x_h))$. 
\end{assumption}

We call this condition \emph{action-level coverage} because
$
    \frac{\Vstar(x_{h+1})}{\Vstar(x_h)}=\frac{\pistar(x_{h+1}\mid x_h)}{\pirefs(x_{h+1}\mid x_h)}
$
is exactly the density ratio between the conditional distributions $\pistar$ and $\pirefs$.\footnote{ We remark that \cref{asmp:act-cov}
can be relaxed to the \emph{average} sense (cf. \cref{thm:smc-coverage-full}).}

The second criterion is bounded $\chi^2$-divergences between $\pistar_h$ and $\pihat_h$ for all $h$. Quantitatively, the two criteria yield the following guarantee for SMC, which directly implies \cref{thm:main-intro}:%

\begin{theorem}
  \label{thm:smc-chisq}
Under \cref{asmp:act-cov}, SMC with $N$ particles achieves
\begin{align}
  \Dtv{\En[\nuhat_H]}{\pistar_H}\leq \sqrt{\frac{\Cact}{N}}\prn*{H+\sum_{h=1}^{H-1} \sqrt{\Dchis{\pistar_h}{\pihat_h}}}.
\end{align}
\end{theorem}

We remark that bounding $\Dchis{\pistar_h}{\pihat_h}$ only requires \emph{average-case} closeness between $\Vhat$ and $\Vstar$ (\cref{fact:chisq-avg}). Existing theory for SMC \citep{schweizer2012non,zhu2026power} can recover an analogous bound that requires \emph{worst-case} closeness:

\begin{assumption}\label{asmp:L-inf}
  There exists parameter $\einf\geq 1$ such that $\einf^{-1}\leq \frac{\Vstar(x)}{\Vhat(x)}\leq \einf$ for any $x\in \cX$. 
\end{assumption}

While the proof technique of \citet{schweizer2012non} can relax this assumption to some extent, it cannot recover \cref{thm:smc-chisq} (see \cref{sec:compare-smc}). Moreover, as we show below, we can strengthen \cref{thm:smc-chisq} further to tolerate \emph{heavy-tailed errors} in $\Vhat$.

\subsubsection{Handling Tail Behavior}

While the $\chi^2$-divergences in \cref{thm:smc-chisq} provide an easy-to-interpret upper bound, $\Dchis{\pistar_h}{\pihat_h}$ is sensitive to the tail behavior of $\pihat_h$ with respect to $\pistar_h$ (equivalently, the tail behavior of $\Vhat$ relative to $\Vstar$); in the case of heavy-tailed errors, it can blow up and render \cref{thm:smc-chisq} vacuous. 

However, the goal of SMC is to produce a sample from an \emph{approximately} correct distribution, so the tail behavior should only appear as an error probability. This motivates the following notion of distributional closeness~\citep{chen2025coverage}, which can be bounded in terms of $\chi^2$-divergence but not vice versa—--in particular, it can be arbitrarily small even when $\Dchis{\pistar_h}{\pihat_h} = \infty$.

\begin{definition}\label{def:coverage}
For parameter $M\geq 1$, $h\in[H]$, 
we define
\begin{align}
  \Dcov[M]{\pistar_h}{\pihat_h}=\bbP_{x\sim \pistar_h}\prn*{ \frac{\pistar_h(x)}{\pihat_h(x)}\geq M }.
\end{align}
\end{definition}

\begin{theorem}
  \label{thm:smc-coverage}
Under \cref{asmp:act-cov}, SMC with $N$ particles achieves the following for any $M\geq 1$:
\begin{align}
  \Dtv{\En[\nuhat_H]}{\pistar_H}\leq H\sqrt{\frac{M\Cact}{N}}+\sum_{h=1}^{H-1} \Dcov[M]{\pistar_h}{\pihat_h}.
\end{align}
\end{theorem}
Let us compare the guarantee of \cref{thm:smc-coverage} with that of \cref{thm:smc-chisq}: the coverage terms $\Dcov[M]{\pistar_h}{\pihat_h}$ in \cref{thm:smc-coverage} play the role of the chi-square divergences $\Dchis{\pistar_h}{\pihat_h} $ in \cref{thm:smc-chisq}, capturing the amount by which $\pihat_h$ deviates from $\pistar_h$. A key difference, though, is that for any fixed value of $M$, the TV error between $\En[\nuhat_H]$ and $\pistar_H$ in \cref{thm:smc-coverage} does not approach $0$ as $N \to \infty$, i.e., we pay for the mass $\Dcov[M]{\pistar_h}{\pihat_h}$ \emph{additively}. 

\subsection{Beyond SMC: Faster Convergence under $L_\infty$-condition}

Below, we show that a simple modification to SMC (Output Option 2 of \cref{alg:smc}) yields exponential (rather than polynomial) convergence when $\Vhat$ and $\Vstar$ are close in a worst-case sense (\cref{asmp:L-inf}).

\begin{theorem}
  \label{thm:smc-linf}
Under \cref{asmp:act-cov,asmp:L-inf}, \cref{alg:smc} with \textbf{Output Option 2} samples from a distribution $\mu$ that satisfies $\Dtv{\mu}{\pistar_H}\leq \delta$,
as long as $N \geq \Omega(H\einf^4 \Cact^2 \log(\einf H/\delta))$, using a number of outer loop steps that is bounded by $O(\einf^2 \log(1/\delta))$.
\end{theorem}

The basic idea is to wrap SMC in an outer rejection sampling loop. This qualitatively matches the convergence rate of the backtracking-based method VGB of \citet{rohatgi2025taming}, which requires the same assumptions. %

\subsection{Beyond SMC: Near-perfect PRM}\label{sec:dmc}

Even when $\Vhat=\Vstar$ (i.e. the PRM is perfect), the upper bound in \cref{thm:smc-chisq} still scales as $\frac{H}{\sqrt{N}}$. In fact, this is a fundamental limitation of SMC, not an artifact of the analysis: even when $\Vhat=\Vstar$ and $\Cact=2$, SMC cannot achieve $o(1)$ sampling error without at least $N \geq \Omega(\sqrt{H})$ particles (\cref{prop:SMC-lower-exact}).\footnote{We leave the problem of determining the optimal $H$-dependence as an open question.} Note that $\Vhat=\Vstar$, the one-particle algorithm which simply samples $x_{h} \sim \pihat(x_h \mid x_{h-1})$ succeeds in exactly sampling from $\pistar$, so there is indeed hope to improve the upper bound of $H/\sqrt{N}$.

\begin{algorithm}[t]
  \caption{\DMCname}\label{alg:dmc-implementable}
\begin{algorithmic}
\State \textbf{Input:} Transition kernel $\piref$, value function $\Vhat$. 
\State \textbf{Parameter:} initial sample size $N\geq 1$, parameter $\eta\geq1$. 
\State Initialize $\MS_0$ to be $N$ copies of $ \perp$.
\For{$h=1,\cdots,H$}
\State Initialize a multi-set $\cS_h=\crl{}$.
\While{$|\MS_h| < N$}
\State Sample $x_{h-1}\sim \mathrm{Unif}(\cS_{h-1})$.
\State Sample $x_h\sim \pirefs[h-1](\cdot\mid x_{h-1})$.
\State With probability $\frac{\Vhat(x_h)}{\eta\Vhat(x_{h-1})}$, update $\cS_h\leftarrow \cS_h \cup \crl{x_h}$.
\EndWhile
\EndFor
\State \textbf{Output:} $x \sim \nuhat_H=\Unif(\MS_H)$. 
\end{algorithmic}
\end{algorithm}

To mitigate this issue, we introduce \DMCname (\DMC, \cref{alg:dmc-implementable}). 
Intuitively, \DMC fixes the issue that, in order to sample new particles, SMC \emph{normalizes} the weighted empirical measure, which introduces interference between particles. \DMC instead uses rejection sampling, so that regardless of the other particles, the conditional distribution $x_h\mid{}x_{h-1}$ is 
\begin{align}
    \textstyle
    \pihat(x_h\mid{}x_{h-1})\propto \piref(x_h\mid{}x_{h-1})\Vhat(x_h).
\end{align}
This can be regarded as a refinement of the base proposal $\piref(\cdot\mid{}x_{h-1})$ with the knowledge of $\Vhat$.
In particular, when $N=1$, \DMC reduces to \emph{action-level rejection sampling} \citep{yang2021fudge}, which, due to \cref{eq:pistar-kernel}, is an exact sampler when $\Vhat=\Vstar$ (see \cref{sec:comparison} for the definition).
We analyze \DMC under the following action-level coverage condition, which we view as conceptually equivalent to \cref{asmp:act-cov} (and formally equivalent, up to a factor of $\einf^2$ under the $L_\infty$ condition of \cref{asmp:L-inf}). In the event that we make the weaker assumption of controlling chi-square divergence between $\pistar$ and $\pihat$ (as in \cref{thm:dmc-chisq} below), \cref{asmp:act-cov-Vhat} is needed in order to ensure that a bounded value of the parameter $\eta$ in \cref{alg:dmc-implementable} suffices. %

\newcommand{\Cacthat}{\wh{C}_{\mathsf{act}}}
\begin{assumption}\label{asmp:act-cov-Vhat}
There exists parameter $\Cacthat\geq 1$ such that for all $h \geq 0$, 
$\Vhat_{h+1}(x_{h+1})\leq \Cacthat \Vhat_h(x_h)$ for any $x_{h+1}\in \supp(\piref(\cdot\mid x_h))$. 
\end{assumption}

\begin{theorem}\label{thm:dmc-chisq}
Under \cref{asmp:act-cov-Vhat}, suppose that \DMC (\cref{alg:dmc-implementable}) is instantiated with $\eta\geq \Cacthat$. Then
\DMC with $N$ particles achieves
\begin{align}
  \Dtv{\En[\nuhat_H]}{\pistar_H}\leq \frac{1}{\sqrt{N}}\sum_{h=1}^H \sqrt{\Dchis{\pistar_h}{\pihat_h}},
\end{align}
and has time complexity $O(NH\eta)$ in expectation.
\end{theorem}

In particular, suppose $\Vhat$ is near-accurate so that $\Dchis{\pistar_h}{\pihat_h}\leq \eps^2$ for all $h\in[H]$. Then \DMC achieves \icml{$\Dtv{\En[\nuhat_H]}{\pistar_H}\leq \frac{\eps H}{\sqrt{N}}.$}
\arxiv{
\begin{align}
    \Dtv{\En[\nuhat_H]}{\pistar_H}\leq \frac{\eps H}{\sqrt{N}}.
\end{align}}
If $\eps \leq \frac{1}{H}$, then \DMC achieves sampling error $o(1)$ with only $N=O(1)$ particles. This is in sharp contrast to SMC, where by \cref{prop:SMC-lower-exact}, $N\geq \Omega(\sqrt{H})$ particles are necessary even when $\Vhat=\Vstar$ and $\Cact=O(1)$ (note that in this setting $\Cacthat=\Cact = O(1)$ as well, so \DMC does indeed succeed).

Finally, \DMC admits analogues of \cref{thm:smc-coverage} (heavy-tailed errors) and \cref{thm:smc-linf} (fast-rate convergence); see \cref{thm:dmc-coverage} and \cref{thm:dmc-linf-full}, respectively.

\subsection{Limits of Particle Filtering Methods}\label{sec:limits}

A key downside of SMC (and the discussed variants) is that, when the PRM is imperfect, avoiding error amplification requires the number of particles $N$ to grow at least linearly in the horizon $H$. In this section, we show that some horizon dependence is necessary for all \emph{myopic} particle filtering methods. Informally, a myopic particle filtering method is any algorithm that maintains a set of particles $\cS_h$ at each step $h$, and does not use the PRM data from later steps $k>h$ to determine the set $\cS_h$. SMC and \DMC are two examples of myopic methods; see \cref{app:pf-lb} for a formal definition.

We prove that even with mild (constant-factor) imperfections in the PRM, any myopic particle filtering algorithm requires nearly $\Omega(\log H)$ particles to even obtain non-trivial \emph{coverage} of $\pistar_H$:

\begin{theorem}\label{thm:pf-lb-main}
Let $N(H) := \log(H)/(4\log \log(H))$. There is \textbf{no myopic particle filtering algorithm} $\Alg$ with the following guarantee for all sufficiently large $H$: $\Alg$ uses at most $N(H)$ particles, and under \cref{asmp:act-cov,asmp:L-inf} with $\Cact = 2$ and $\einf = e^3$, the output distribution $\nu$ of $\Alg$ satisfies $\nu(\cE) \geq H^{-1/5}$
for all events $\cE \subset \cX_H$ with $\pistar_H(\cE) \geq 1/2$.
\end{theorem}

As an immediate corollary, no such algorithm can achieve $\Dtv{\nu}{\pistar_H}\leq 1/3$. In contrast, with a perfect PRM (i.e. $\einf = 1$), a single particle achieves zero error \citep{rohatgi2025taming}. Moreover, \cref{thm:pf-lb-main} is particularly striking since, under \cref{asmp:L-inf}, every intermediate distribution $\pihat_h$ has constant coverage of $\pistar_h$ (i.e. $\pihat_h(x)/\pistar_h(x) \geq \einf^{-2}$ for all $x\in\cX_h$). \cref{thm:pf-lb-main} shows that it is impossible for myopic methods to inductively ``maintain'' this coverage without a super-constant number of particles --- and raises the question of whether some form of \emph{lookahead} can improve computational efficiency.

\section{Experiments in Language Models}
\label{sec:experiments}
Next, we investigate to what extent the quantities from \cref{sec:SMC} %
predict the performance of SMC on sampling problems involving LLMs.

\subsection{Controlled Setting: Prompt-Switching Task}
\label{sec:prompt-switching}
A key challenge in evaluating SMC and understanding its performance as a function of the aforementioned quantities is the difficulty of \emph{measuring them}, as well as measuring the performance of SMC itself. Most notably, we need to measure quantities involving the target distribution $\pistar_h$, which is typically not readily available. While prior work (e.g., \citet{zhao2024twisted}) has studied how to evaluate the performance of SMC using the algorithm's estimate of the partition function $Z := \E_{\piref}[\Vstar(x_H)]$, such techniques are only guaranteed to provide tight bounds if SMC is accurate.  Moreover, these methods do not directly address the challenge of estimating other quantities such as action-level coverage. 
Thus, in this section, we consider an alternative approach, in a setting which will conveniently satisfy that $\piref, \pistar$ are both efficiently sampleable. %

In particular, we consider the \emph{prompt-switching task}, where $\piref$ and $\pistar$ correspond to the output distribution of a language model under different prompts. Formally, we fix a language model $\model$, and consider the setting of \cref{sec:intro} where $\MA$ is the token space of $\model$ and $\MX_h = \MA^h$. %
Each instance of the prompt-switching task is determined by 3 prompts, $(\promptref, \promptstar, \prompt)$, as follows. 
Given choices of these prompts, we let $\piref$ be the output distribution of $\model$ for the prompt $\promptref$, and we denote it as $\piref(\cdot) = \model(\cdot \mid \promptref)$. We then choose the terminal reward function to be $r^\star(a_{1:H}) := \frac{ \model(a_{1:H} \mid \promptstar)}{ \model(a_{1:H} \mid \promptref)}$, which gives that, for any intermediate step $h$, $V^\star(a_{1:h}) = {\model(a_{1:h} \mid \promptstar) / \model(a_{1:h} \mid \promptref)}$. It follows that $\pistar_h(a_{1:h}) = \model(a_{1:h} \mid \promptstar)$. Since it is tractable to sample from $\model$ and compute sequence-level probabilities, sampling from $\piref, \pistar$ as well as computing $\Vstar$ is tractable. Finally, the prompt $\prompt$ determines $\Vhat$, in a way that will be discussed below.

\paragraph{Evaluation} 
To evaluate SMC on the prompt-switching task, we consider the following metric, inspired by more complex ``LLM-as-judge'' techniques and also the use of log-probabilities to characterize the output distributions of LLMs \citep{golowich2025sequenceslogitsreveallow}. Given two distributions $\nu, \nu'$ over sequences of length $H$, and a set of prompts $\MP$, we define the \emph{$\MP$-logprob discrepancy} $\mathrm{LPD}_\MP(\nu, \nu')$ between $\nu$ and $\nu'$ as the difference between the mean token-level log-probabilities assigned to sequences sampled from $\nu$ and $\nu'$ under the prompts in $\MP$ (summed over all positions $h$). Formally, for $\nu, \nu' \in \Delta(\MA^H)$, we have:
\begin{align}
\mathrm{LPD}_\MP(\nu, \nu') :=&  \sum_{\prompt \in \MP} \sum_{h=1}^H \left|\E_{a_{1:H} \sim \nu} \left[ \log \model(a_h \mid \prompt, a_{1:h-1}) \right]\right.\icml{\nonumber\\
&\quad } \left.  - \E_{a_{1:H}' \sim \nu'} \left[ \log \model(a_h' \mid \prompt, a_{1:h-1}') \right]\right|.\nonumber
\end{align}
In words, the $\MP$-logprob discrepancy uses the prompts in $\MP$ as \emph{test functions} to measure how different the two distributions $\nu, \nu'$ are. It can be thought of as a very simple (albeit more stable) ``LLM-as-judge'' metric, where instead of \emph{generating} from $\model$ to evaluate $\nu, \nu'$, we simply use the log-probabilities assigned by the model $\model$. In our evaluations, we choose $\MP = \{\promptref, \promptstar\}$, which is natural in light of the fact that the two most salient prompts in the prompt switching task are $\promptref$ and $\promptstar$.

\subsection{Evaluating Dependence on PRM Accuracy}
\label{sec:exp-prm-accuracy}
To evaluate the influence of PRM accuracy on SMC performance, we fix the action-level coverage between $\piref$ and $\pistar$ (in particular, by fixing $\promptref=\promptstar$) and vary the PRM $\Vhat$. In particular, we consider a number $k = 50$ of instances of the prompt-switching task as described above:  each instance $i \in [k]$ is specified by three prompts $((\promptref)\^i,(\promptstar)\^i,\prompt\^i)$, as well as a scalar $\alpha$ used in \cref{eq:vhat-interp} below. In this subsection, the reference $(\promptref)\^i$ and target $(\promptstar)\^i$ prompts are shared across all instances, so we denote them by $\promptref, \promptstar$, respectively.
Moreover, we in fact have $\promptref = \promptstar$ --- this prompt asks the model to write a short story (e.g., ``Write a scene about a dragon negotiating peace with the last human kingdom.'').\footnote{While in this setting it is trivial to sample from $\promptstar$, we emphasize that our main goal here is studying the performance of SMC in terms of PRM accuracy as opposed to solving a challenging sampling problem.} Each instance $i \in [k]$ has a different value of $\Vhat$, denoted $\Vhat\^i$, which is given as follows: we fix $k$ prompts $\prompt\^1, \ldots, \prompt\^k$, which are each identical to $\promptref$ except they ask for the story in $k$ different styles (e.g., ``$\ldots$ Tell it as a news article.''). For each such prompt $\prompt\^i$, we let $\pi\^i$ denote the distribution $\pi\^i(a_{1:H}) = \model(a_{1:H} \mid \prompt\^i)$, which should be interpreted as a ``perturbed'' version of $\pistar$. %
Then we define
\begin{align}
\Vhat\^i(a_{1:h}) := \frac{\pistar(a_{1:h})}{\piref(a_{1:h})} \cdot \left( \frac{\pi\^i(a_{1:h})}{\pistar(a_{1:h})}\right)^{(1-h/H) \cdot \alpha},\label{eq:vhat-interp}
\end{align}
for some parameter $\alpha > 0$ controlling the degree of accuracy of $\Vhat\^i$. To get some intuition behind this definition, suppose $\alpha = 1$ for simplicity: for small $h$, $\Vhat\^i(a_{1:h}) \approx {\pi\^i(a_{1:h})}/{\pirefs[h](a_{1:h})}$, while at the final step $h =H$,  $\Vhat\^i(a_{1:H}) = \Vstar(a_{1:H})$. In particular, at early steps, $\Vhat_h\^i$ looks the value function of the ``perturbed'' target distribution $\pi\^i$, meaning that for small $h$, $\Vhat_h\^i$ is ``misleading'' in that it ``points towards'' $\pi\^i$ instead of $\pistar$. For larger $h$, we attenuate the noise so as to ensure that sampling from $\Vstar_H$ is information-theoretically possible (i.e., that $\Vhat_H\^i = \Vstar_H$).

Natural variation in the choice of the prompts $\prompt\^i$ (which were generated by GPT-5) leads the resulting functions $\Vhat\^i_h$ to have varying levels of accuracy with respect to $\Vstar_h$. Define $\pihat\^i_h\in\Delta(\MA^h)$ by $\pihat\^i_h(a_{1:h}) \propto \piref(a_{1:h}) \cdot \Vhat\^i(a_{1:h})$. While estimating $\Dchis{\pistar_h}{\pihat\^i_h}$ is in principle possible, we found that the ratios $\Vstar(a_{1:h}) / \Vhat\^i(a_{1:h})$ needed to estimate the $\chi^2$ divergence %
are prohibitively high in variance. 
Thus, we instead estimate $\Dkl{\pistar_h}{\pihat\^i_h}$ to measure the divergence between $\pistar_h$ and $\pihat_h\^i$, which is much more stable because KL divergence can be written as an expectation of \emph{log}-probabilities (see \cref{sec:exp-prm-accuracy-details} for further details). %

In our experiments, we use Qwen3-0.6B as the language model $\model$ and set the horizon to $H = 64$ tokens. For each choice of $i \in [k]$ we run SMC with $N=32$ particles given the PRM $\Vhat\^i$, and sample from its output distribution $\nuhat_H\^i$. In \cref{fig:inv-phi-intro}, we plot the values of $\Dkl{\pistar_h}{\pihat_h\^i}$ against the sampling error of SMC, measured via $\MP$-logprob discrepancies $\mathrm{LPD}_\MP(\E[\nuhat_H\^i], \pistar_H)$ (estimated by running SMC for 200 trials). %
We observe a clear correlation between the two quantities, thus validating our theoretical findings that PRM accuracy (as measured by the divergence between $\pihat_h$ and $\pistar_h$) predicts SMC's sampling error.  
In \cref{sec:exp-prm-accuracy-details} we repeat the experiment for additional choices of $\promptref, \promptstar, \prompt\^i, \alpha$ and observe similar results.

\subsection{Evaluating Dependence on Action-level Coverage}\label{sec:exp-prm-action}
We next evaluate the influence of action-level coverage on SMC performance. We consider the same setup as in the previous subsection, except for each instance $i \in [k]$ we take the target distribution, denoted $(\pistar)\^i$, to be the same as $\pi\^i$. In particular, each instance is determined by the tuple $(\promptref, \prompt\^i, \prompt\^i)$. %
Moreover, we take $\alpha = 0$, so that $\Vhat\^i = (\Vstar)\^i$. %
In this setting, the \emph{action-level coverage} depends on $i$. We again use Qwen3-0.6B, take the horizon to be $H = 32$ tokens, and run SMC with $N = 32$ particles.

While the action-level coverage between $(\pistar)\^i$ and $\piref$ (in the $L_\infty$ sense of \cref{asmp:act-cov}) is extremely large and essentially impossible to compute, we note that $\Dkl{(\pistar)\^i}{\piref}$ can be viewed as an ``average-case proxy'' for action-level coverage in light of the chain rule, which gives \[\Dkl{(\pistar)\^i}{\piref} = \sum_{h=1}^H \E_{a_{1:h} \sim (\pistar)\^i_h}\left[ \log \frac{(\pistar)\^i_h(a_{h} \mid a_{1:h-1})}{\pirefs[h](a_{h} \mid a_{1:h-1})} \right].\] 

\Cref{fig:mean-ratio-intro} plots $\frac{1}{H}\Dkl{(\pistar)\^i}{\piref}$ against the $\MP$-logprob discrepancies $\mathrm{LPD}_\MP(\E[\nuhat_H\^i], \pistar_H)$; we again observe a clear correlation between the two quantities, in line with our theory. %

\subsection{Dependence on the Number of Particles}
\label{sec:experiments-particles}
Finally, in \cref{fig:pswitch-wallclock-scaling}, we show the performance of SMC as a function of the number of particles $N$ for the prompt switching task, using a setup similar to that of \cref{sec:exp-prm-accuracy}. In particular, we consider 50 different prompt tuples $(\promptref\^i, (\promptstar)\^i, \prompt\^i)$, and for each tuple we use a PRM $\Vhat\^i$ defined from $\prompt\^i$ as in \cref{eq:vhat-interp} with $\alpha = 1$. In addition to SMC, we also plot the performance of (a) sequential importance sampling using $\Vhat\^i$ (\cref{sec:comparison}), and (b) the ``best-of-$N$'' algorithm which samples $N$ sequences from $\piref$ and outputs the one with the highest value according to the terminal reward $(\Vstar)\^i(a_{1:H})$. We observe that SMC consistently outperforms both baselines, and moreover that its performance improves as $N$ increases; the latter conclusion is in line with our theoretical results. 

\begin{figure}[ht]
\centering
\includegraphics[width=\icml{0.4}\arxiv{0.6}\textwidth]{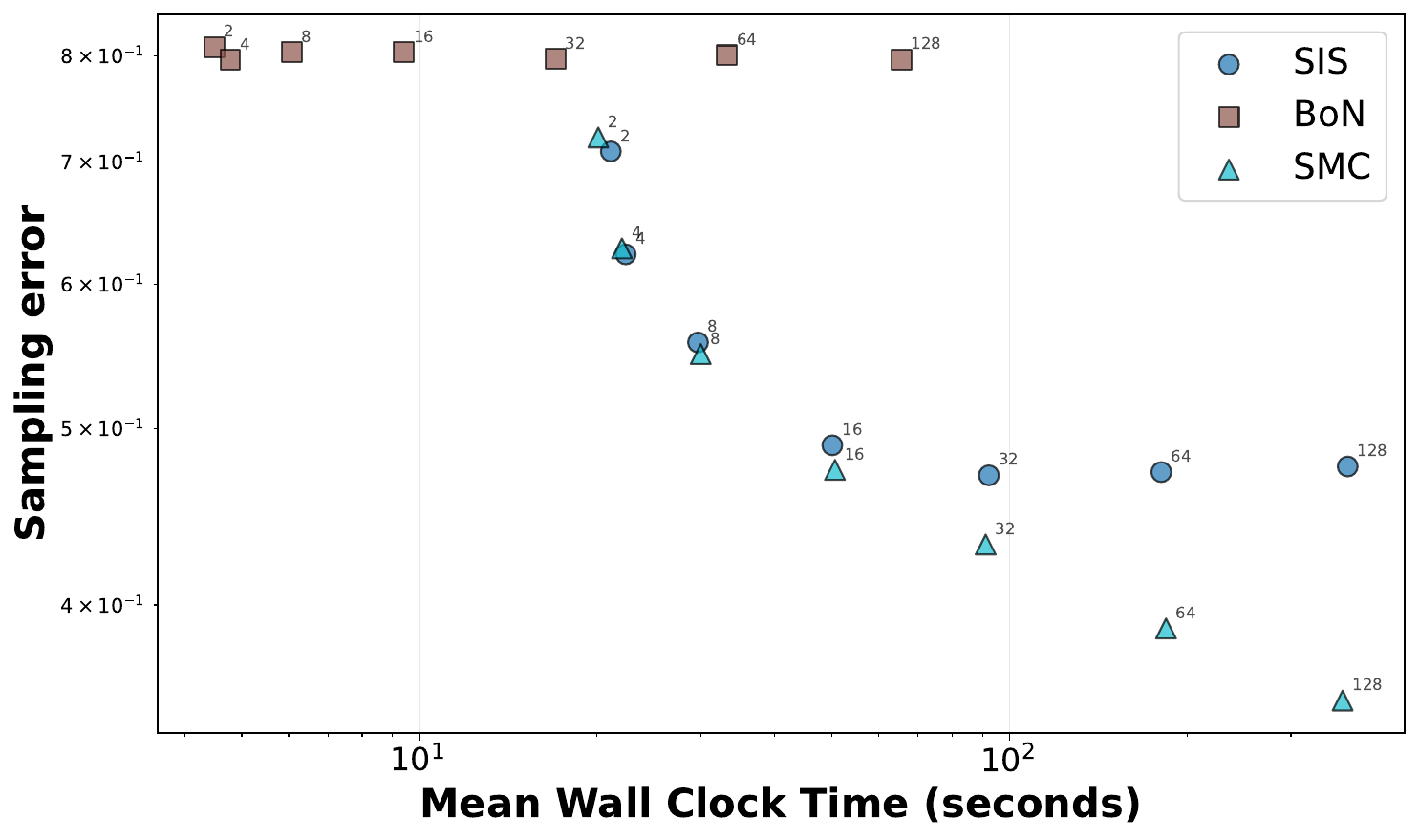}
\caption{Performance of SMC on the prompt-switching task (\cref{sec:experiments-particles}). Each point represents SMC with some number of particles $N \in \{2,4,8,\ldots,128\}$. We average the sampling error and wall-clock time over all data points and all trials per data point.  SMC consistently outperforms both sequential importance sampling (SIS) and Best-of-$N$ (BoN) baselines.}
\label{fig:pswitch-wallclock-scaling}
\end{figure}

\subsection{SMC on Math Problem-solving Tasks}
\label{sec:smc-math-experiments}
Finally, we study the use of SMC to improve the performance of LLMs on math problem-solving tasks, including AIME and MATH500. In this setting, the base distribution $\piref$ corresponds to the output distribution of a language model prompted to solve a math problem, and the terminal reward function $r^\star(a_{1:H}) \in \{0,1\}$ is the indicator that the answer $a_{1:H}$ is correct. Thus, $\pistar$ is simply the base model distribution \emph{conditioned on producing a correct answer}, and $\Vstar(a_{1:h})$ is the probability that $\piref$ produces a correct answer conditioned on the prefix $a_{1:h}$.

There are several open-source process reward models fine-tuned for math problems, which purport to yield good estimates of $\Vstar$. In our experiments, we take the PRM $\Vhat$ to be Qwen2.5-Math-PRM-7B, and the base model $\piref$ to be Qwen2.5-1B-Instruct. In the application of SMC to math reasoning tasks, since the goal is simply to find \emph{some} correct answer, it is typical to combine SMC with Best-of-$N$ sampling in the following manner: at the end of SMC, we select the single particle with the highest value $\Vhat(a_{1:H})$. Note that this is distinct from the standard Best-of-$N$ sampling procedure, in which $N$ samples are drawn \emph{independently} from $\piref$ and the one with highest $\Vhat$-value is selected.

\paragraph{SMC improves performance on math} Many papers \citep{puri2025probabilistic,chatziveroglou2025adecodingtokenefficientinferencescaling} have reported that SMC typically improves the performance over Best-of-$N$ (as well as other baselines). However, such results are typically reported \emph{on-average} over all problems in the dataset, which fails to account for the hetereogeneity between individual problems. For instance, it was not clear if SMC improves performance over Best-of-$N$ on \emph{all} problems, or if SMC can actually hurt performance (compared to Best-of-$N$) on a non-negligible fraction of the dataset.\footnote{Indeed, much prior work has highlighted that SMC can suffer from pathologies such as mode collapse\noah{cite}, which might decrease the diversity of responses from which Best-of-$N$ can choose.} Certainly such a distinction is salient for any theoretical explanation for SMC's performance.

In \cref{fig:smc-vs-bon}, we estimate the performance of SMC and Best-of-$N$ for each problem in Math500 (using 100 trials per problem), and plot the accuracy of SMC against that of Best-of-$N$. While there is variability in the relative performance of the two algorithms, we see that for the vast majority of problems, SMC's accuracy is at least as large as that of Best-of-$N$. In \cref{fig:smc-vs-bon-aime}, we observe a similar conclusion for AIME24 and AIME25.

\begin{figure}[ht]
\centering
\includegraphics[width=\icml{0.45}\arxiv{0.9}\textwidth]{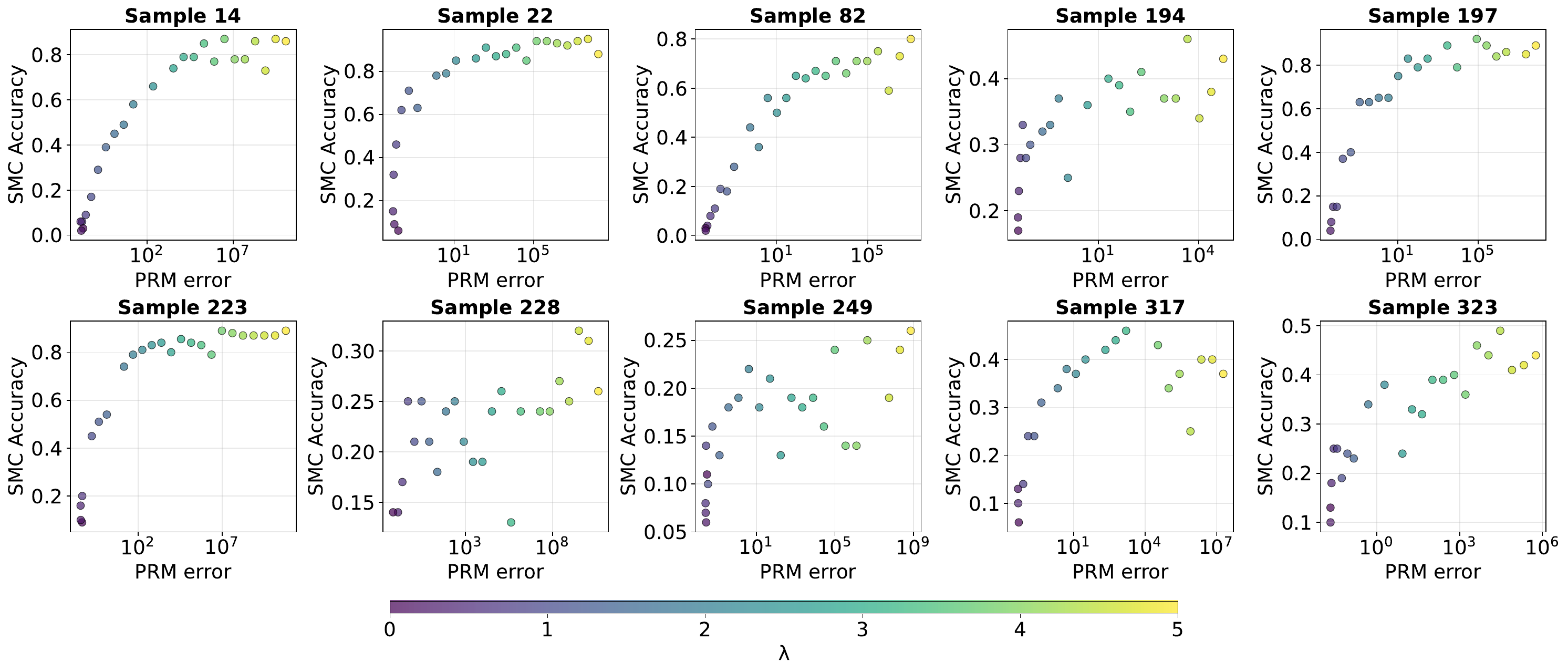}
\caption{Influence of PRM error (measured by $\Dchis{\pistar_h}{\pihat_h}$) on SMC accuracy for 10 Math500 problems. Different colors correspond to different values of the inverse temperature parameter $\lambda$ parametrizing $\Vhat\^\lambda$. The plots show SMC accuracy with respect to a \emph{random} particle (not the best one) selected at the end of SMC.}
\label{fig:smc-temperature}
\vspace{-0.5cm}
\end{figure}

\paragraph{SMC and PRM accuracy} Having established that SMC leads to a consistent increase in accuracy across most Math500 and AIME problems, we next turn to the question of what factors influence its performance \emph{on a given problem}. In light of the fact that significant effort has been spent designing better PRMs $\Vhat$ \citep{lightman2023let,wang2024mathshepherd,wang2025valueguided,zhang2025lessons}, a natural follow-up question from our theory is whether the divergence between $\pistar_h$ and $\pihat_h$ (as measured by, e.g., the $\chi^2$-divergence $\Dchis{\pistar_h}{\pihat_h}$) predicts the performance of SMC. 

To address this question, we consider a range of PRMs $\Vhat\^\lambda$ by varying the inverse temperature parameter $\lambda$ used by the PRM to produce its final probabilities.  %
In \cref{fig:smc-temperature}, we select 10 Math500 problems, and for each value of $\lambda$, we plot an empirical estimate of $\Dchis{\pistar_h}{\pihat_h}$ %
against the accuracy of SMC given $\Vhat\^\lambda$. Interestingly, we observe that larger $\chi^2$-divergence often leads to \emph{higher} accuracy, which is in contrast to our theoretical findings which predict the opposite. %
We suspect that larger values of $\lambda$ (lower temperature) lead to better performance since they more effectively weed out incorrect partial solutions; however, they will also potentially lead $\pihat_h$ to miss some modes of $\pistar_h$, which causes the $\chi^2$-divergence to blow up. Thus, in order to accurately capture the performance of SMC on such problems, we need a modification to our framework that aims to not approximate $\pistar_H$ in a distributional sense but rather to simply cover some portion of its mass.

\section{Discussion}
In this paper, we have established theoretical guarantees for particle filtering algorithms under varioues assumptions relating the reference distribution $\piref$ to the target distribution $\pistar$ and the PRM $\Vhat$ (via the induced distributions $\pihat$). We highlight below some intriguing directions for future work: 
\begin{itemize}
  \item There are various polynomial gaps between our upper and lower bounds, including their the dependence on horizon $H$, as discussed in \cref{sec:theory-overview}, as well as the decay of the total variation distance as a function of $N$: whereas our upper bounds of \cref{thm:smc-chisq,thm:dmc-chisq} decay as $1/\sqrt{N}$, the lower bound of \cref{prop:SMC-lower-exact} decays as only $1/N$. 
\item Can we develop metrics for the performance of SMC (and other inference-time sampling algorithms) which are weaker than total variation distance but still meaningfully capture ``usefulness'' of outputs? For instance, for math problems, typically one only desires a correct answer, and the particular phrasing of the solution may not matter too much. 
\item Can we develop provable guarantees for particle filtering algorithms or other types of inference-time sampling procedures which rely on weaker notions of discrepancy between $\pistar_h$ and $\pihat_h$? The values of $\Dchis{\pistar_h}{\pihat_h}$ seen in, e.g., the experiments in \cref{sec:exp-prm-accuracy,sec:exp-prm-action,sec:experiments-particles} are extremely large (in particular, much larger than the values of $N$ seen in our experiments), yet particle filtering algorithms still are able to effectively make use of $\Vhat$. 
\end{itemize}

\icml{\section*{Impact Statement}

This paper presents work whose goal is to advance the field of machine learning. There are many potential societal consequences of our work, none of which we feel must be specifically highlighted here.
}

\bibliography{refs.bib}

@article{brown2024large,
  title={Large language monkeys: Scaling inference compute with repeated sampling},
  author={Brown, Bradley and Juravsky, Jordan and Ehrlich, Ryan and Clark, Ronald and Le, Quoc V and R{\'e}, Christopher and Mirhoseini, Azalia},
  journal={arXiv preprint arXiv:2407.21787},
  year={2024}
}

@inproceedings{fu2025deep,
  title={Deep Think with Confidence},
  author={Fu, Yichao and Wang, Xuewei and Tian, Yuandong and Zhao, Jiawei},
  booktitle={NeurIPS 2025 Workshop on Efficient Reasoning},
  year={2025}
}

@inproceedings{
wang2023selfconsistency,
title={Self-Consistency Improves Chain of Thought Reasoning in Language Models},
author={Xuezhi Wang and Jason Wei and Dale Schuurmans and Quoc V Le and Ed H. Chi and Sharan Narang and Aakanksha Chowdhery and Denny Zhou},
booktitle={The Eleventh International Conference on Learning Representations },
year={2023},
url={https://openreview.net/forum?id=1PL1NIMMrw}
}

@article{zhao2025majority,
  title={The majority is not always right: Rl training for solution aggregation},
  author={Zhao, Wenting and Aggarwal, Pranjal and Saha, Swarnadeep and Celikyilmaz, Asli and Weston, Jason and Kulikov, Ilia},
  journal={arXiv preprint arXiv:2509.06870},
  year={2025}
}

@article{wu2025resum,
  title={ReSum: Unlocking Long-Horizon Search Intelligence via Context Summarization},
  author={Wu, Xixi and Li, Kuan and Zhao, Yida and Zhang, Liwen and Ou, Litu and Yin, Huifeng and Zhang, Zhongwang and Yu, Xinmiao and Zhang, Dingchu and Jiang, Yong and others},
  journal={arXiv preprint arXiv:2509.13313},
  year={2025}
}

@article{yang2024buffer,
  title={Buffer of thoughts: Thought-augmented reasoning with large language models},
  author={Yang, Ling and Yu, Zhaochen and Zhang, Tianjun and Cao, Shiyi and Xu, Minkai and Zhang, Wentao and Gonzalez, Joseph E and Cui, Bin},
  journal={Advances in Neural Information Processing Systems},
  volume={37},
  pages={113519--113544},
  year={2024}
}

@inproceedings{
mudgal2023controlled,
title={Controlled Decoding from Language Models},
author={Sidharth Mudgal and Jong Lee and Harish Ganapathy and YaGuang Li and Tao Wang and Yanping Huang and Zhifeng Chen and Heng-Tze Cheng and Michael Collins and Trevor Strohman and Jilin Chen and Alex Beutel and Ahmad Beirami},
booktitle={Forty-first International Conference on Machine Learning},
year={2024},
url={https://openreview.net/forum?id=bVIcZb7Qa0}
}

@article{wang2025value,
  title={Value-Guided Search for Efficient Chain-of-Thought Reasoning},
  author={Wang, Kaiwen and Zhou, Jin Peng and Chang, Jonathan and Gao, Zhaolin and Kallus, Nathan and Brantley, Kiant{\'e} and Sun, Wen},
  journal={arXiv preprint arXiv:2505.17373},
  year={2025}
}

@article{puri2025probabilistic,
  title={A probabilistic inference approach to inference-time scaling of llms using particle-based monte carlo methods},
  author={Puri, Isha and Sudalairaj, Shivchander and Xu, Guangxuan and Xu, Kai and Srivastava, Akash},
  journal={arXiv preprint arXiv:2502.01618},
  year={2025}
}

@inproceedings{
botta2025on,
title={On the Query Complexity of Verifier-Assisted Language Generation},
author={Edoardo Botta and Yuchen Li and Aashay Mehta and Jordan T. Ash and Cyril Zhang and Andrej Risteski},
booktitle={Forty-second International Conference on Machine Learning},
year={2025},
url={https://openreview.net/forum?id=9oIjvaDhoN}
}

@inproceedings{yang2021fudge,
  title={FUDGE: Controlled Text Generation With Future Discriminators},
  author={Yang, Kevin and Klein, Dan},
  booktitle={Proceedings of the 2021 Conference of the North American Chapter of the Association for Computational Linguistics: Human Language Technologies},
  pages={3511--3535},
  year={2021}
}

@misc{golowich2025sequenceslogitsreveallow,
      title={Sequences of Logits Reveal the Low Rank Structure of Language Models}, 
      author={Noah Golowich and Allen Liu and Abhishek Shetty},
      year={2025},
      eprint={2510.24966},
      archivePrefix={arXiv},
      primaryClass={cs.LG},
      url={https://arxiv.org/abs/2510.24966}, 
}

@article{rohatgi2025taming,
  title={Taming imperfect process verifiers: A sampling perspective on backtracking},
  author={Rohatgi, Dhruv and Shetty, Abhishek and Saless, Donya and Li, Yuchen and Moitra, Ankur and Risteski, Andrej and Foster, Dylan J},
  journal={arXiv preprint arXiv:2510.03149},
  year={2025}
}

@article{madaan2025rethinking,
  title={Rethinking thinking tokens: Llms as improvement operators},
  author={Madaan, Lovish and Didolkar, Aniket and Gururangan, Suchin and Quan, John and Silva, Ruan and Salakhutdinov, Ruslan and Zaheer, Manzil and Arora, Sanjeev and Goyal, Anirudh},
  journal={arXiv preprint arXiv:2510.01123},
  year={2025}
}

@inproceedings{cerou2011nonasymptotic,
  title={A nonasymptotic theorem for unnormalized Feynman--Kac particle models},
  author={C{\'e}rou, F and Del Moral, P and Guyader, A},
  booktitle={Annales de l'IHP Probabilit{\'e}s et statistiques},
  volume={47},
  number={3},
  pages={629--649},
  year={2011},
  organization={Gauthier-Villars}
}

@article{whiteley2012linear,
  title={Linear variance bounds for particle approximations of time-homogeneous Feynman--Kac formulae},
  author={Whiteley, Nick and Kantas, Nikolas and Jasra, Ajay},
  journal={Stochastic Processes and their Applications},
  volume={122},
  number={4},
  pages={1840--1865},
  year={2012},
  publisher={Elsevier}
}

@incollection{del2004feynman,
  title={Feynman-kac formulae},
  author={Del Moral, Pierre},
  booktitle={Feynman-Kac Formulae: Genealogical and Interacting Particle Systems with Applications},
  pages={47--93},
  year={2004},
  publisher={Springer}
}

@article{schweizer2012non,
  title={Non-asymptotic error bounds for sequential MCMC methods in multimodal settings},
  author={Schweizer, Nikolaus},
  journal={arXiv preprint arXiv:1205.6733},
  year={2012}
}

@article{guarniero2017iterated,
  title={The Iterated Auxiliary Particle Filter},
  author={Guarniero, Peter and Johansen, Adam M. and Lee, Anthony},
  journal={Journal of the American Statistical Association},
  volume={112},
  number={520},
  pages={1636--1647},
  year={2017}
}

@article{heng2017controlled,
  title={Controlled Sequential Monte Carlo},
  author={Heng, Jeremy and Bishop, Adrian and Deligiannidis, George and Doucet, Arnaud},
  journal={arXiv preprint arXiv:1708.08396},
  year={2017}
}

@article{lawson2022sixo,
  title={SIXO: Smoothing Inference with Twisted Objectives},
  author={Lawson, J. and Raventos, L. and Warrington, A. and Linderman, S.},
  journal={CoRR},
  volume={abs/2206.05952},
  year={2022}
}

@article{naesseth2019highdim,
  title={High-dimensional Filtering using Nested Sequential Monte Carlo},
  author={Naesseth, C. A. and Lindsten, F. and Sch{\"o}n, T. B.},
  journal={IEEE Transactions on Signal Processing},
  volume={67},
  number={16},
  pages={4177--4188},
  year={2019},
  doi={10.1109/TSP.2019.2926035}
}

@article{doucet2000sequential,
  title={On Sequential Monte Carlo sampling methods for Bayesian filtering},
  author={Doucet, Arnaud and Godsill, Simon J. and Andrieu, Christophe},
  journal={Statistics and Computing},
  volume={10},
  number={3},
  pages={197--208},
  year={2000}
}

@misc{chatziveroglou2025adecodingtokenefficientinferencescaling,
      title={A*-Decoding: Token-Efficient Inference Scaling}, 
      author={Giannis Chatziveroglou},
      year={2025},
      eprint={2505.13672},
      archivePrefix={arXiv},
      primaryClass={cs.AI},
      url={https://arxiv.org/abs/2505.13672}, 
}

@techreport{pitt1997filtering,
  title={Filtering via Simulation: Auxiliary Particle Filters},
  author={Pitt, Michael K. and Shephard, Neil},
  institution={Nuffield College, University of Oxford},
  number={W13},
  year={1997}
}

@inproceedings{montemerlo2002fastslam,
  title={FastSLAM: A Factored Solution to the Simultaneous Localization and Mapping Problem},
  author={Montemerlo, Michael and Thrun, Sebastian and Koller, Daphne and Wegbreit, Ben},
  booktitle={Proceedings of the {AAAI} Conference on Artificial Intelligence},
  year={2002}
}

@article{lightman2023let,
  title={Let's Verify Step by Step},
  author={Lightman, Hunter and Kosaraju, Vineet and Burda, Yura and Edwards, Harri and Baker, Bowen and Lee, Teddy and Leike, Jan and Schulman, John and Sutskever, Ilya and Cobbe, Karl},
  journal={arXiv preprint arXiv:2305.20050},
  year={2023}
}

@article{wang2024mathshepherd,
  title={Math-Shepherd: Verify and Reinforce {LLMs} Step-by-step without Human Annotations},
  author={Wang, Peiyi and Li, Lei and Shao, Zhihong and Xu, R. X. and Dai, Damai and Li, Yifei and Chen, Deli and Wu, Y. and Sui, Zhifang},
  journal={arXiv preprint arXiv:2312.08935},
  year={2024}
}

@article{lew2023smcsteering,
  title={Sequential Monte Carlo Steering of Large Language Models using Probabilistic Programs},
  author={Lew, Alexander K. and Tan, Zhi-Xuan and Grand, Gabriel and Mansinghka, Vikash K.},
  journal={arXiv preprint arXiv:2306.03081},
  year={2023}
}

@article{zhao2024twisted,
  title={Probabilistic Inference in Language Models via Twisted Sequential Monte Carlo},
  author={Zhao, Stephen and Brekelmans, Rob and Makhzani, Alireza and Grosse, Roger},
  journal={arXiv preprint arXiv:2404.17546},
  year={2024}
}

@article{xiong2024iterative,
  title={Iterative Preference Learning from Human Feedback: Bridging Theory and Practice for {RLHF} under {KL}-constraint},
  author={Xiong, Wei and Dong, Hanze and Ye, Chenlu and Wang, Ziqi and Zhong, Han and Ji, Heng and Jiang, Nan and Zhang, Tong},
  journal={arXiv preprint arXiv:2312.11456},
  year={2024}
}

@article{stephan2024calculation,
  title={From Calculation to Adjudication: Examining {LLM} Judges on Mathematical Reasoning Tasks},
  author={Stephan, Andreas and Zhu, Dawei and Assenmacher, Matthias and Shen, Xiaoyu and Roth, Benjamin},
  journal={arXiv preprint arXiv:2409.04168},
  year={2024}
}

@article{feng2024step,
  title={Step-by-Step Reasoning for Math Problems via Twisted Sequential Monte Carlo},
  author={Feng, Shengyu and Kong, Xiang and Ma, Shuang and Zhang, Aonan and Yin, Dong and Wang, Chong and Pang, Ruoming and Yang, Yiming},
  journal={arXiv preprint arXiv:2410.01920},
  year={2024}
}

@article{singhal2025general,
  title={A General Framework for Inference-time Scaling and Steering of Diffusion Models},
  author={Singhal, Raghav and Horvitz, Zachary and Teehan, Ryan and Ren, Mengye and Yu, Zhou and McKeown, Kathleen and Ranganath, Rajesh},
  journal={arXiv preprint arXiv:2501.06848},
  year={2025}
}

@article{zhang2025lessons,
  title={The Lessons of Developing Process Reward Models in Mathematical Reasoning},
  author={Zhang, Zhenru and Zheng, Chujie and Wu, Yangzhen and Zhang, Beichen and Lin, Runji and Yu, Bowen and Liu, Dayiheng and Zhou, Jingren and Lin, Junyang},
  journal={arXiv preprint arXiv:2501.07301},
  year={2025}
}

@article{puri2025particlemc,
  title={A Probabilistic Inference Approach to Inference-Time Scaling of {LLMs} using Particle-Based Monte Carlo Methods},
  author={Puri, Isha and Sudalairaj, Shivchander and Xu, Guangxuan and Xu, Kai and Srivastava, Akash},
  journal={arXiv preprint arXiv:2502.01618},
  year={2025}
}

@article{skreta2025feynman,
  title={Feynman-Kac Correctors in Diffusion: Annealing, Guidance, and Product of Experts},
  author={Skreta, Marta and Akhound-Sadegh, Tara and Ohanesian, Viktor and Bondesan, Roberto and Aspuru-Guzik, Al{\'a}n and Doucet, Arnaud and Brekelmans, Rob and Tong, Alexander and Neklyudov, Kirill},
  journal={arXiv preprint arXiv:2503.02819},
  year={2025}
}

@article{foster2025foundation,
  title={Is a Good Foundation Necessary for Efficient Reinforcement Learning? The Computational Role of the Base Model in Exploration},
  author={Foster, Dylan J. and Mhammedi, Zakaria and Rohatgi, Dhruv},
  journal={arXiv preprint arXiv:2503.07453},
  year={2025}
}

@article{lipkin2025fastcontrolled,
  title={Fast Controlled Generation from Language Models with Adaptive Weighted Rejection Sampling},
  author={Lipkin, Benjamin and LeBrun, Benjamin and Vigly, Jacob Hoover and Loula, {Jo\~ao} and MacIver, David R. and Du, Li and Eisner, Jason and Cotterell, Ryan and Mansinghka, Vikash K. and O'Donnell, Timothy J. and Lew, Alexander K. and Vieira, Tim},
  journal={arXiv preprint arXiv:2504.05410},
  year={2025}
}

@inproceedings{grand2025selfsteering,
  title={Self-Steering Language Models},
  author={Grand, Gabriel and Tenenbaum, Joshua B. and Mansinghka, Vikash K. and Lew, Alexander K. and Andreas, Jacob},
  booktitle={Conference on Language Modeling ({COLM})},
  year={2025},
  note={arXiv:2504.07081}
}

@inproceedings{loula2025syntactic,
  title={Syntactic and Semantic Control of Large Language Models via Sequential Monte Carlo},
  author={Loula, {Jo\~ao} and LeBrun, Benjamin and Du, Li and Lipkin, Ben and Pasti, Clemente and Grand, Gabriel and Liu, Tianyu and Emara, Yahya and Freedman, Marjorie and Eisner, Jason and Cotterell, Ryan and Mansinghka, Vikash K. and Lew, Alexander K. and Vieira, Tim and O'Donnell, Timothy J.},
  booktitle={International Conference on Learning Representations},
  year={2025},
  note={arXiv:2504.13139}
}

@inproceedings{wang2025valueguided,
  title={Value-Guided Search for Efficient Chain-of-Thought Reasoning},
  author={Wang, Kaiwen and Zhou, Jin Peng and Chang, Jonathan and Gao, Zhaolin and Kallus, Nathan and Brantley, Kiante and Sun, Wen},
  booktitle={Advances in Neural Information Processing Systems},
  year={2025},
  note={arXiv:2505.17373}
}

@article{geuter2025guided,
  title={Guided Speculative Inference for Efficient Test-Time Alignment of {LLMs}},
  author={Geuter, Jonathan and Mroueh, Youssef and Alvarez-Melis, David},
  journal={arXiv preprint arXiv:2506.04118},
  year={2025}
}

@article{yu2025limits,
  title={On the Limits of Test-Time Compute: Sequential Reward Filtering for Better Inference},
  author={Yu, Yue and Di, Qiwei and Gu, Quanquan and Zhou, Dongruo},
  journal={arXiv preprint arXiv:2512.04558},
  year={2025}
}

@article{chen2025coverage,
  title={The Coverage Principle: How Pre-Training Enables Post-Training},
  author={Chen, Fan and Huang, Audrey and Golowich, Noah and Malladi, Sadhika and Block, Adam and Ash, Jordan T and Krishnamurthy, Akshay and Foster, Dylan J},
  journal={arXiv preprint arXiv:2510.15020},
  year={2025}
}

@article{marion2018finite,
  title={Finite Sample $ L\_2 $ Bounds for Sequential Monte Carlo and Adaptive Path Selection},
  author={Marion, Joseph and Schmidler, Scott C},
  journal={arXiv preprint arXiv:1807.01346},
  year={2018}
}

@article{lee2024convergence,
  title={Convergence bounds for sequential Monte Carlo on multimodal distributions using soft decomposition},
  author={Lee, Holden and Santana-Gijzen, Matheau},
  journal={arXiv preprint arXiv:2405.19553},
  year={2024}
}

@misc{karan2025reasoning,
      title={Reasoning with Sampling: Your Base Model is Smarter Than You Think}, 
      author={Aayush Karan and Yilun Du},
      year={2025},
      eprint={2510.14901},
      archivePrefix={arXiv},
      primaryClass={cs.LG},
      url={https://arxiv.org/abs/2510.14901}, 
}

@article{zhu2026power,
  title={On the Power of (Approximate) Reward Models for Inference-Time Scaling},
  author={Zhu, Youheng and Lu, Yiping},
  journal={arXiv preprint arXiv:2602.01381},
  year={2026}
}
\icml{\IfFileExists{icml2026.bst}{\bibliographystyle{icml2026}}{\bibliographystyle{plainnat}}}

\appendix
\onecolumn

\renewcommand{\contentsname}{Contents of Appendix}
\addtocontents{toc}{\protect\setcounter{tocdepth}{2}}
{
  \hypersetup{hidelinks}
  \tableofcontents
}

\clearpage

\section{Comparison with Prior Work}\label{sec:comparison}
In this section we discuss how our theoretical guarantees compare with prior work. We specify to the autoregressive generation setting, as the more general setting has not been theoretically studied in the past.

\paragraph{Action-level rejection sampling / sequential importance sampling \citep{yang2021fudge}} This algorithm samples $a_{1:H}$ autoregressively via the conditional distributions
\[\pihat(a_{h+1}\mid{}a_{1:h}) \propto \piref(a_{1:h+1}\mid{}a_{1:h}) \Vhat(a_{1:h+1}).\]
It can be implemented in expected time $O(H\Cacthat)$ by rejection sampling / importance sampling \cite{rohatgi2025taming}, and by \cref{eq:pistar-kernel}, it exactly samples from $\pistar$ if $\Vhat=\Vstar$. However, even under assumption \cref{asmp:L-inf} with $\einf = 1+\epsilon$, it can incur $\Omega(\epsilon\sqrt{H})$ sampling error \cite{rohatgi2025taming}, which is vacuous if e.g. $\epsilon = \Theta(1)$.

\paragraph{VGB \citep{rohatgi2025taming}} This algorithm is defined by a random walk on the \emph{tree of autoregressive generations}, which has state space $\cA^0 \sqcup \dots\sqcup \cA^H$ and in which the parent of $a_{1:h}$ is $a_{1:h-1}$. The first main result of \citet{rohatgi2025taming} is that the random walk samples from $\pistar$, up to total variation error $\delta$, in time roughly $O(H^2 \Cact \einf^4 \log(1/\delta))$. This guarantee is essentially matched by the guarantee of \cref{thm:smc-linf} for SMC, except the language model evaluations in SMC can be parallelized into $H$ rounds, whereas the language model evaluations in VGB are inherently sequential.\footnote{The results of \cref{app:coupling} suggest a potential approach to parallelizing VGB via particle filtering, but it is unclear in what generality such a simulation would be efficient.}

The second main result of \citet{rohatgi2025taming} shows that under the following average-case closeness condition between $\Vhat$ and $\Vstar$, as well as bounded action-level coverage, VGB samples from a distribution that \emph{approximately covers} $\pistar$. 

\begin{assumption}[Average-case assumption of \citet{rohatgi2025taming}]\label{asmp:rohatgi-avg}
\[\max\left\{ \En^{\pistar}\left[\frac{\Vhat(a_{1:h})}{\Vstar(a_{1:h})}\right], \En^{\pistar}\left[\frac{\Vstar(a_{1:h})}{\Vhat(a_{1:h})}\right]\right\} \leq C.\]
\end{assumption}

We make the following observation:

\begin{fact}\label{fact:chisq-avg}
For each $0 \leq h \leq H$, it holds that
\[\Dchis{\pistar_h}{\pihat_h} = \En^{\pistar}\left[\frac{\Vhat(a_{1:h})}{\Vstar(a_{1:h})}\right] \cdot \En^{\pistar}\left[\frac{\Vstar(a_{1:h})}{\Vhat(a_{1:h})}\right] - 1.\]
\end{fact}

Thus, under \cref{asmp:rohatgi-avg} as well as bounded action-level coverage, \cref{thm:smc-chisq} shows that SMC with $N \leq (C+1)^2 H^2 / \epsilon^2$ particles samples from a distribution that is $\epsilon$-close in total variation distance to $\pistar$. Note that closeness in total variation distance is a stronger guarantee than approximate coverage.

\section{Technical Tools}
\begin{proposition}\label{prop:MZ-inequality}
For any $R\geq 0$, independent random variables $X_1,\cdots,X_n$, it holds that
\begin{align}
  \En\abs*{\sum_{i=1}^n (X_i-\En[X_i])}\leq \sqrt{nR\max_{i}\En\abs{X_i}}+2n\max_i \En\brk*{\abs{X_i}\indic\crl{\abs{X_i}\geq R}}.
\end{align}
\end{proposition}

\begin{proof}
We define $Y_i=\min\crl{\max\crl{X_i,-R},R}$ and $Z_i=X_i-Y_i$. Then, we can bound
\begin{align}
  \En\abs*{\sum_{i=1}^n (X_i-\En[X_i])}\leq&~
  \En\abs*{\sum_{i=1}^n (Y_i-\En[Y_i])}+\sum_{i=1}^n \En\abs{Z_i-\En[Z_i]}.
\end{align}
Note that $\En\abs{Z_i-\En[Z_i]}\leq 2\En\abs{Z_i}\leq 2\En\brk*{\abs{X_i}\indic\crl{\abs{X_i}\geq R}}$. Further,
\begin{align}
  \prn*{\En\abs*{\sum_{i=1}^n (Y_i-\En[Y_i])}}^2\leq&~ \En\prn*{\sum_{i=1}^n (Y_i-\En[Y_i])}^2 \\
  =&~\sum_{i=1}^n \En(Y_i-\En[Y_i])^2
  \leq \sum_{i=1}^n \En Y_i^2\leq R\sum_{i=1}^n \En\abs{X_i},
\end{align}
where we use $\abs{Y_i}\leq \min\crl{\abs{X_i},R}$. The proof is then completed by combining the inequalities above.
\end{proof}

\section{Proofs for SMC}\label{appdx:SMC}
\paragraph{Notation}
We define $\cH_h = \crl{(\tx_{\ell}^{i}, x_\ell\ind{i})}_{\ell\in\brk{h},i\in\brk{N}}$ to be the full history of the execution of SMC (\cref{alg:smc}) up to the end of iteration $h$. We write $\En_h[\cdot]$ for the conditional distribution with respect to $\cH_h$. We write $Z := \E_{x_{1:H} \sim \piref}[r^\star(x_H)]$ and observe that $Z = \E_{x \sim \pi_h}[\Vstar(x)]$ for all $h$.

\subsection{Useful Lemmas for SMC}
We first prove the following lemmas.

\begin{lemma}
\label{clm:induction-wh} 
For any step $h$ and any function $g : \MX_h \to \BR$, we have that
  \begin{align}
    \label{eq:inductive-step-smc}
    \En_{h-1}\brk*{\What_h\cdot\En_{\nuhat_h}\brk*{g}} = \What_{h-1}\cdot\En_{x_{h-1} \sim \nuhat_{h-1}}\brk*{
     \frac{ \En_{x_h \sim \piref(\cdot \mid x_{h-1})}\brk*{\Vhat(x_h) \cdot g(x_{h})}}{\Vhat(x_{h-1})}
    }.
  \end{align}
\end{lemma}

The next lemma follows from applying \cref{clm:induction-wh} inductively.

\begin{lemma}\label{lem:SMC-unbiased}
  For any $g_h:\cX_h\to\RR$, it holds that
  \begin{align}
    \label{eq:SMC-unbiased-g}
        \En\brk*{\What_h\cdot\En_{\nuhat_h}\brk*{g_h}} = \En_{x_h \sim \pi_h}\brk*{\Vhat(x_h)g_h(x_h)}.
  \end{align}
  In particular, $\En[\What_h]=\Zhat_h = \En\brk{\Vhat(x_h)}$ for $h\geq 1$ and
  \begin{align}
    \label{eq:SMC-R-unbiased-nu}
        \En\brk*{\What_H \nuhat_H(x)} = Z\cdot \pistar_H(x), \qquad \forall x\in\cX_H.
  \end{align}
\end{lemma}

Finally, we upper bound the TV error of \cref{alg:smc} by the deviation of $\What_H$.
\begin{lemma}\label{lem:SMC-TV}
  For \cref{alg:smc}, it holds that
  \begin{align}
    \Dtv{\En[\nuhat_H]}{\pistar_H}\leq \frac1{2Z}\En\abs{\What_H-Z}.
  \end{align}
\end{lemma}

\begin{proof}[\pfref{clm:induction-wh} and \cref{lem:SMC-unbiased}]
  To prove \cref{eq:inductive-step-smc}, observe that
  \begin{align}
    &\En_{h-1}\brk*{\What_h\cdot\En_{\nuhat_h}\brk*{g}} \\
    &= \What_{h-1}\cdot\En_{h-1}\brk*{\frac{1}{N}
    \sum_{i=1}^{N}w_h\ind{i}g(x_{h}\ind{i})
    }\\
    &= \What_{h-1}\cdot\En_{h-1}\brk*{
    \frac{1}{N}\sum_{i=1}^{N}\frac{\Vhat(x_h\ind{i})}{\Vhat(\tx_{h-1}\ind{i})} g(x_{h}\ind{i})
    }.
  \end{align}
  Since $\crl*{(\tx_{h-1}\ind{i},x_h\ind{i})}$ are \iid given $\cH_{h-1}$, for each $i \in [N]$, we have
  \begin{align}
\E_{h-1} \left[ \frac{\Vhat(x_h\ind{i})}{\Vhat(\tx_{h-1}\ind{i})} g(x_h\ind{i})\right] = \E_{\tx_{h-1}\sim \hat\nu_{h-1}} \E_{x_h \sim \piref(\cdot \mid \tx_{h-1})}\brk*{ \frac{\Vhat(x_h)}{\Vhat(\tx_{h-1})} g(x_h) },
  \end{align}
  which establishes \cref{eq:inductive-step-smc}. Next, we argue that \cref{eq:SMC-unbiased-g} holds inductively (noting that the base case is trivial since $\What_0 = \Vhat(\perp)$ and $\nuhat_0 = \delta_\perp$), using \cref{eq:inductive-step-smc}:
  \begin{align}
  \E \left[ \What_h \cdot \E_{\nuhat_H}[g_h] \right] =& \E \left[ \What_{h-1} \cdot \E_{x_{h-1} \sim \nuhat_{h-1}}\brk*{
     \frac{ \E_{x_h \sim \piref(\cdot \mid x_{h-1})}\brk*{\Vhat(x_h) \cdot g_h(x_{h})}}{\Vhat(x_{h-1})}
    } \right] \\
    =& \E \left[ \Vhat(x_{h-1}) \cdot \frac{\E_{x_h \sim \piref(\cdot \mid x_{h-1})}[\Vhat(x_h) g_h(x_h)]}{\Vhat(x_{h-1})} \right]\\
    =& \E[\Vhat(x_h) g_h(x_h)],
  \end{align}
  as desired. We get that $\E[\What_h] = \E[\Vhat(x_h)]$ by using $g_h \equiv 1$. Finally, \cref{eq:SMC-R-unbiased-nu} follows by setting $g_H(x) = \indic\crl{x_H = x}$ in \cref{eq:SMC-unbiased-g} and using that $\Vhat(x_H) = \Vstar(x_H)$ for $x_H \in \MX_H$. 
\end{proof}

\begin{proof}[\pfref{lem:SMC-TV}]
By \cref{eq:SMC-R-unbiased-nu}, we can write
\begin{align}
  2Z\Dtv{\En[\nuhat_H]}{\pistar_H}=&~
  \sum_{x\in\cX_H} \abs{Z\En[\nuhat_H(x)]-Z\pistar_H(x)}
  =\sum_{x\in\cX_H} \abs*{Z\En[\nuhat_H(x)]-\En[\What_H \nuhat_H(x)]} \\
  \leq&~ \sum_{x\in\cX_H} \En\brk*{\abs*{\What_H-Z}\nuhat_H(x)}
  =\En\brk*{ \abs*{\What_H-Z}\cdot \sum_{x\in\cX_H} \nuhat_H(x) }
  =\En\abs*{\What_H-Z}.
\end{align}
\end{proof}

\subsection{Proof of \cref{thm:smc-chisq}}\label{appdx:pf-smc-chisq}

We prove a slightly stronger upper bound.
\begin{proposition}\label{prop:smc-chisq}
SMC (\cref{alg:smc}) with $N$ particles achieves
\begin{align}
    \Dtv{\En[\nuhat_H]}{\pistar_H}\leq \frac{1}{\sqrt{N}}\sum_{h=1}^H \sqrt{\En_{\piref}[\Vhat(x_{h-1})]\En_{\piref}\brk*{ \frac{\Vstar(x_h)^2}{\Vhat(x_{h-1})} }}.
\end{align}
\end{proposition}

\cref{thm:smc-chisq} then follows from the fact that $\Vstar(x_h)\leq \Cact \Vstar(x_{h-1})$ almost surely under $x_h\sim \piref(\cdot\mid{}x_{h-1})$ and
\begin{align}
\En_{\piref}[\Vhat(x_{h-1})]\En_{\piref}\brk*{ \frac{\Vstar(x_{h-1})^2}{\Vhat(x_{h-1})} }
=1+\Dchis{\pistar_{h-1}}{\pihat_{h-1}}.
\end{align}
\qed

\begin{proof}[\pfref{prop:smc-chisq}]
We define $\rho(x)=\frac{\Vstar(x)}{\Vhat(x)}$ and
\begin{align}
  M_h=\What_h \cdot \En_{x_h\sim \nuhat_h}[\rho(x_h)].
\end{align}
Then \cref{eq:inductive-step-smc} with $g(x) = \rho(x)$ implies $(M_h)$ is a martingale with $M_0=Z$ and $M_H=\What_H$. We can then decompose
\begin{align}
  \En\abs{\What_H-Z}\leq \sum_{h=1}^H \En\abs{M_h-M_{h-1}}.
\end{align}
Next, we note that
\begin{align}
  {M_h-M_{h-1}}=&~\What_{h-1} \prn*{\frac{1}{N}\sum_{i=1}^N w_h\ind{i}\rho(x_h\ind{i}) - \En_{x_{h-1}\sim \nuhat_{h-1}}[\rho(x_{h-1})] } \\
  =&~ \What_{h-1} \prn*{\frac{1}{N}\sum_{i=1}^N \frac{\Vstar(x_h\ind{i})}{\Vhat(\tx_{h-1}\ind{i})} - \En_{x_{h-1}\sim \nuhat_{h-1}}[\rho(x_{h-1})] }.
\end{align}
Since the tuples $(\tx_{h-1}\ind{i}, x_h\ind{i})$ are \iid given $\cH_{h-1}$, we can bound
\begin{align}
  \frac{1}{\What_{h-1}^2} \En_{h-1}(M_h-M_{h-1})^2 \leq&~\frac1N \Var_{x_{h-1}\sim \nuhat_{h-1}, x_h\sim \piref(\cdot\mid{}x_{h-1}) }\brk*{ \frac{\Vstar(x_h)}{\Vhat(x_{h-1})} } \\
  \leq&~ \frac1N \En_{x_{h-1}\sim \nuhat_{h-1} }\brk*{ \frac{\En_{x_h \sim \piref(\cdot \mid x_{h-1})}[\Vstar(x_h)^2]}{\Vhat(x_{h-1})^2} } \rdef \Delta_{h-1}.
\end{align}
Then, by \cref{lem:SMC-unbiased}, it holds that
\begin{align}
  \En[\What_{h-1}]\En[\What_{h-1}\Delta_{h-1}]
  =\frac1N\En_{\piref}[\Vhat(x_{h-1})]\En_{\piref}\brk*{ \frac{\Vstar(x_h)^2}{\Vhat(x_{h-1})} }. 
\end{align}
Therefore, we can bound
\begin{align}
  \En\abs{\What_H-Z}\leq&~ \sum_{h=1}^H \En\brk*{ \En_{h-1}\abs{M_h-M_{h-1}} }\\
  \leq&~ \sum_{h=1}^H \En\brk*{ \What_{h-1}\sqrt{\Delta_{h-1}} } 
  \leq \sum_{h=1}^H \sqrt{\En[\What_{h-1}] \cdot \En\brk*{ \What_{h-1}\Delta_{h-1}} }  \\
  \leq&~ \frac{1}{\sqrt{N}}\sum_{h=1}^H \sqrt{\En_{\piref}[\Vhat(x_{h-1})]\En_{\piref}\brk*{ \frac{\Vstar(x_h)^2}{\Vhat(x_{h-1})} }}.
\end{align}
Combining the inequalities above completes the proof.
\end{proof}

\subsection{Insufficiency of Existing Analysis}\label{sec:compare-smc}

Most non-asymptotic analyses of SMC assume conditions at least as strong as the $L_\infty$ closeness in \cref{asmp:L-inf}~\citep{del2004feynman,cerou2011nonasymptotic,whiteley2012linear}. Another line of work~\citep{marion2018finite,lee2024convergence} imposes assumptions on the transition kernel, which are therefore incomparable to those in this paper. 

We note that these approaches might fail to capture SMC convergence when we assume only $\chi^2$-closeness (\cref{thm:smc-chisq}). This is because many explicitly or implicitly control the variance $(\Var(\What_h))_{h\in[H]}$~\citep[etc.]{cerou2011nonasymptotic,whiteley2012linear,schweizer2012non}. In contrast, our proof of \cref{thm:smc-chisq} proceeds via the absolute deviation $\En\abs{M_H-Z}$. The seemingly more natural attempt to study the variance $\En\prn{M_H-Z}^2$ must confront the fact that $M_H=\What_H=\prod_{h=1}^H\prn*{\frac1N\sum_{i=1}^N w_h\ind{i}}$ is a sequential product. As the following lemma suggests, $\Var(\What_h)$ need not be controlled when the $\chi^2$ divergences $(\Dchis{\pistar_h}{\pihat_h})_{h\in[H]}$ are bounded. 

\begin{lemma}\label{lem:smc-var-lower}
Suppose that $\Vhat(\perp)=1$.
For any $h\in[H]$, it holds that
\begin{align}
  \Var(\What_h)=\En(\What_h-\Zhat_h)^2\geq \frac{\Var_{x_1\sim \pi_1}(\En[\Vhat(x_h)\mid x_1])}{N}.
\end{align}
In particular, for $H\geq2$, there exists a problem instance $(\Vstar, \Vhat)$ such that $\Dchis{\pistar_h}{\pihat_h}\leq 2$, and both \cref{asmp:act-cov} and \cref{asmp:act-cov-Vhat} holds with $\Cact=2$, but SMC with $N$ particles incurs
\begin{align}
  \max_{h\in[H]} \Var(\What_h) \geq \frac{2^{\floor{H/2}}-2}{N}.
\end{align}
\end{lemma}

Therefore, if convergence of SMC is established via upper bounds on $\Var(\What_h)$, the rate can be vacuous unless we have $N\geq \exp(\Omega(H))$ particles. For example, 
\citet{schweizer2012non} study SMC through the following relation (cf. Lemma 2.1 therein) for $f:\cX_H\to[0,1]$:
\begin{align}
  \En\prn*{\En_{\nuhat_H}[f]-\En_{\pistar_H}[f]}^2
  \leq 2\Var(\What_H\cdot \En_{\nuhat_H}[f])+2\Var(\What_H).
\end{align}
Consequently, \citet{schweizer2012non} establish upper bounds of the form
\begin{align}
  N\cdot \Var(\What_H\cdot \En_{\nuhat_H}[f])\leq C_0(f)+C_1(f)\cdot \eps_N,
\end{align} 
where $C_0, C_1$ are two problem-dependent functionals, and $\eps_N$ is defined so that
\begin{align}
  \eps_N\geq \max_{h\in[H]}\Var(\What_h).
\end{align}
This implies that the analysis in \citet{schweizer2012non} cannot recover the upper bound of \cref{thm:smc-chisq} (and \cref{thm:smc-coverage}).

\begin{proof}[\pfref{lem:smc-var-lower}]
For any $h\leq \ell$, we define $\Vtil_\ell(x_h)= \En[\Vhat(x_\ell)\mid{} x_h]$ and $\rho_\ell(x_h)=\frac{\Vtil_\ell(x_h)}{\Vhat(x_h)}$, and we define $M_{h,\ell}=\What_h \En_{\nuhat_h}[\rho_\ell]$. Then it is straightforward to verify that $(M_{h,\ell})_{h\in[\ell]}$ is a martingale with $M_{0,\ell}=\Zhat_\ell$ and $M_{\ell,\ell}=\What_\ell$. 

Therefore, we can lower bound
\begin{align}
  \En(\What_\ell-\Zhat_\ell)^2\geq &~\En(M_{1,\ell}-\Zhat_\ell)^2 \\
  =&~\En\brk*{ \prn*{\frac{1}{N}\sum_{i=1}^N \frac{\Vtil_\ell(x_1\ind{i})}{\Vhat(\tx_{0}\ind{i})} - \En_{\nuhat_{0}}[\rho_\ell] }^2 } \\
  =&~\frac{1}{N}\prn*{ \En_{x_{1}\sim \pi_1}[\Vtil_\ell(x_1)^2] - \Zhat_\ell^2} ,
\end{align}
where we note that $\tx_0\ind{i}=\perp$ for $i\in[N]$.

Now, consider $n=\floor{H/2}$. Consider the following simple example:
\begin{itemize}
  \item For each $h\in[H]$, $\cX_h=\crl{(h,0),(h,1)}$, and $P((h+1,i)\mid (h,i))=1$ for $i\in\crl{0,1}$. 
  \item The initial distribution is given by $\pi_1((1,0))=\frac1{2^n}=\pi_1((1,1))$.
  \item The reward function is $r\equiv 1$, and hence $\Vstar(x)=1$ for all $x\in\cX$.
  \item For each $h\in[H]$, we set $\Vhat((h,1))=1$ and $\Vhat((h,0))=2^h$ for $h\leq n$, and $\Vhat((h,0))=2^{(2n-h)_+}$ for $h\geq n$.
\end{itemize}
Then, it is clear that both \cref{asmp:act-cov} and \cref{asmp:act-cov-Vhat} hold with $\Cact=2$, and $\pihat_h$ is given by $\pihat((h,0))=\frac{p\Vhat((h,0))}{p\Vhat((h,0))+1-p}$, and hence $\pihat((h,0))\in[p,\frac12]$. This immediately implies $\Dchis{\pistar_h}{\pihat_h}\leq 2$ for any $h\in[H]$.

However, we can then calculate $\Vtil_n((1,0))=2^n$, and hence $\En[\Vtil_n(x_1)^2]\geq p\cdot 2^{2n}=2^n$. This gives the desired statement.
\end{proof}

\subsection{\pfref{thm:smc-coverage}}

We prove the following general version of \cref{thm:smc-coverage}. 
\begin{theorem}
  \label{thm:smc-coverage-full}
SMC with $N$ particles achieves the following for any $M\geq 1$, $\eta\geq 1$, 
\begin{align}
  \Dtv{\En[\nuhat_H]}{\pistar_H}\leq H\sqrt{\frac{M\eta}{N}}+\sum_{h=1}^H\brk*{ \Dcov[M]{\pistar_h}{\pihat_h}+\Dcovac[\eta,h]{\pistar}{\piref} },
\end{align}
where we define %
\begin{align}
  \Dcovac[\eta,h]{\pistar}{\piref}\ldef&~ \bbP_{x_{h-1}\sim \pistar_{h-1}, x_h\sim \pistar(\cdot\mid x_{h-1})}\prn*{ \frac{\pistar(x_h\mid x_{h-1})}{\piref(x_h\mid x_{h-1})}\geq \eta }\\
  =&~\bbP_{x_{h-1}\sim \pistar_{h-1}, x_h\sim \pistar(\cdot\mid x_{h-1})}\prn*{ \frac{\Vstar(x_h)}{\Vstar(x_{h-1})}\geq \eta }.
\end{align}
\end{theorem}

The quantity $\Dcovac[\eta,h]{\pistar}{\piref}$ is the action-level analogue of \cref{def:coverage}.
Under \cref{asmp:act-cov}, this additional term is $0$ when $\eta \geq \Cact$, recovering \cref{thm:smc-coverage}.

\begin{proof}
We follow the notation of the proof of \cref{thm:smc-chisq} (\cref{appdx:pf-smc-chisq}). Recall that
\begin{align}
  {M_h-M_{h-1}}
  =&~ \What_{h-1} \prn*{\frac{1}{N}\sum_{i=1}^N \frac{\Vstar(x_h\ind{i})}{\Vhat(\tx_{h-1}\ind{i})} - \En_{x_{h-1}\sim \nuhat_{h-1}}[\rho(x_{h-1})] },
\end{align}
where the tuples $(\tx_{h-1}\ind{i}, x_h\ind{i})$ are \iid given $\cH_{h-1}$. 
By \cref{prop:MZ-inequality}, for any fixed parameter $R>0$, we can bound
\begin{align}
  &~ \frac{1}{\What_{h-1}} \En_{h-1}\abs{M_h-M_{h-1}} \\
  \leq&~ \sqrt{\frac{R}{N}\En_{x_{h-1}\sim \nuhat_{h-1} }\brk*{ \frac{\Vstar(x_{h-1})}{\Vhat(x_{h-1})} }}+2\En_{x_{h-1}\sim \nuhat_{h-1}, x_h\sim \piref(\cdot\mid x_{h-1}) }\brk*{ \frac{\Vstar(x_h)}{\Vhat(x_{h-1})}\indic\crl*{\frac{\Vstar(x_h)}{\Vhat(x_{h-1})}\geq R} } \\
\end{align}
By \cref{lem:SMC-unbiased}, we have
\begin{align}
  &~\En\brk*{\What_{h-1}\sqrt{\En_{x_{h-1}\sim \nuhat_{h-1} }[\rho(x_{h-1})]}}\leq \sqrt{\En[\What_{h-1}]\En[\What_{h-1}\En_{x_{h-1}\sim \nuhat_{h-1} }[\rho(x_{h-1})]]}\leq \sqrt{\Zhat_{h-1}\cdot Z}, 
\end{align}
and similarly, for any $\eta>0$,
\begin{align}
  &~\En\brk*{\What_{h-1}\En_{x_{h-1}\sim \nuhat_{h-1}, x_h\sim \piref(\cdot\mid x_{h-1}) }\brk*{ \frac{\Vstar(x_h)}{\Vhat(x_{h-1})}\indic\crl*{\frac{\Vstar(x_h)}{\Vhat(x_{h-1})}\geq R} }} \\
  =&~ \En_{x_{h-1}\sim \pi_{h-1}, x_h\sim \piref(\cdot\mid x_{h-1})}\brk*{ \Vstar(x_h)\indic\crl*{\frac{\Vstar(x_h)}{\Vhat(x_{h-1})}\geq R} } \\
  \leq&~ \En_{x_{h-1}\sim \pi_{h-1}, x_h\sim \piref(\cdot\mid x_{h-1})}\brk*{ \Vstar(x_h)\indic\crl*{\frac{\Vstar(x_{h-1})}{\Vhat(x_{h-1})}\geq \eta^{-1}R} +\Vstar(x_h)\indic\crl*{\frac{\Vstar(x_h)}{\Vstar(x_{h-1})}\geq \eta}} \\
  =&~ Z\cdot \brk*{ \bbP_{x_{h-1}\sim \pistar_{h-1}}\prn*{ \frac{\Vstar(x_{h-1})}{\Vhat(x_{h-1})}\geq \eta^{-1}R }+\bbP_{x_{h-1}\sim \pistar_{h-1}, x_h\sim \pistar(\cdot\mid x_{h-1})}\prn*{ \frac{\Vstar(x_h)}{\Vstar(x_{h-1})}\geq \eta } },
\end{align}
where we use $\pi_{h-1}(x_{h-1})\cdot \Vstar(x_{h-1})=\pistar_{h-1}(x_{h-1})\cdot Z$ for $x_{h-1}\in\cX_{h-1}$ and $\piref(x_h\mid{}x_{h-1})\cdot \Vstar(x_h)=\pistar(x_h\mid{}x_{h-1})\cdot Z$.
Therefore, we can choose $R=M\eta \frac{Z}{\Zhat_{h-1}}$ to derive
\begin{align}
  \frac{1}{Z} \En\abs{M_h-M_{h-1}}
  \leq&~ \sqrt{\frac{M\eta}{N}}+2\bbP_{x_{h-1}\sim \pistar_{h-1}}\prn*{ \frac{\Vstar(x_{h-1})}{\Vhat(x_{h-1})}\geq M\cdot \frac{Z}{\Zhat_{h-1}} } \\
  &~ +2\bbP_{x_{h-1}\sim \pistar_{h-1}, x_h\sim \pistar(\cdot\mid x_{h-1})}\prn*{ \frac{\Vstar(x_h)}{\Vstar(x_{h-1})}\geq \eta }.
\end{align}
Taking summation completes the proof.
\end{proof}

\subsection{\pfref{thm:smc-linf}}\label{appdx:pf-smc-linf}

We first prove the following lemma.
\begin{lemma}
  \label{lem:linf-smc-concentration}
Under \cref{asmp:L-inf} and \cref{asmp:act-cov}, for any $\delta\in(0,1)$, if we run SMC (\cref{alg:smc}) with $N = \Omega (H \einf^4 \Cact^2 \log(H/\delta))$ particles, then with probability at least $1-\delta$, it holds that $|\What_H - Z| \leq Z/2$.
\end{lemma}
\begin{proof}
We follow the notation of the proof of \cref{thm:smc-chisq} (\cref{appdx:pf-smc-chisq}). Recall that
\begin{align}
  {M_h-M_{h-1}}
  =&~ \What_{h-1} \prn*{\frac{1}{N}\sum_{i=1}^N \frac{\Vstar(x_h\ind{i})}{\Vhat(\tx_{h-1}\ind{i})} - \En_{x_{h-1}\sim \nuhat_{h-1}}[\rho(x_{h-1})] },
\end{align}
where the tuples $(\tx_{h-1}\ind{i}, x_h\ind{i})$ are \iid given $\cH_{h-1}$. 
Note that by \cref{asmp:act-cov}, it holds that almost surely 
\begin{align}
  0\leq \frac{\Vstar(x_h\ind{i})}{\Vhat(\tx_{h-1}\ind{i})}=\frac{\Vstar(x_h\ind{i})}{\Vstar(\tx_{h-1}\ind{i})}\cdot \frac{\Vstar(\tx_{h-1}\ind{i})}{\Vhat(\tx_{h-1}\ind{i})}\leq \Cact \einf.
\end{align}

Using the fact that for a random variable $Y \in [0,A]$, we can bound $\E[e^{\lambda(Y - \E[Y])}] \leq e^{\lambda^2 A^2/8}$, %
we have that for any $\lambda \in \BR$, %
\begin{align}
& \E_{h-1}[\exp(\lambda(M_h - M_{h-1}))] \\
= & \E_{h-1}\brk*{\exp\prn*{\lambda \What_{h-1} \cdot \prn*{\frac{1}{N}\sum_{i=1}^N \frac{\Vstar(x_h\ind{i})}{\Vhat(\tx_{h-1}\ind{i})} - \E_{\nuhat_{h-1}}[\rho(x_{h-1})] }}}\\
 \leq & {\exp\prn*{\frac{1}{N}\lambda^2 \einf^2 \Cact^2 \What_{h-1}^2}} \leq  {\exp\prn*{\frac{1}{N}\lambda^2 \einf^4 \Cact^2 M_{h-1}^2}}.
\end{align}

By Freedman's inequality, for any $h \in [H]$, it follows that with probability $1-\delta$, we have that  
\begin{align}
|M_h - Z| \leq \sum_{i=1}^{h-1} \frac{\lambda \einf^4 \Cact^2 M_{i}^2}{N} + \frac{\log (2/\delta)}{\lambda}.
\end{align}
Choosing $\lambda = 4 \log(2H/\delta) / Z$ and using the union bound, it follows that with probability $1-\delta$, for all $h \in [H]$,
\begin{align}
|M_h - Z| \leq&~ \frac{4 \log(2/\delta) \einf^4 \Cact^2}{N} \sum_{i=0}^{h-1} \frac{ M_{i}^2}{ Z} + \frac{Z}{4}.\label{eq:freedman-mh}
\end{align}
Since we have chosen $N = \Omega(H \log(H/\delta) \einf^4 \Cact^2)$, it follows from induction on $h$ that in the event that \cref{eq:freedman-mh} holds, we have that $|M_h - Z| \leq Z/2$ for all $h \in [H]$. In particular, this implies that $|\What_H - Z| \leq Z/2$ with probability at least $1-\delta$.
\end{proof}

\begin{proof}[\pfref{thm:smc-linf}]
In any given round, let $\phat$ be the distribution of the output of \cref{alg:smc} with Option 2 (when \cref{alg:smc} restarts we regard it as outputing the symbol ``restart''). Then, the probability that the algorithm outputs any $x \in \MX_H$ is given by
\begin{align}
\phat(x) :=  \E\brk*{\min \left\{ 1, \frac{\What_H}{2\einf \Zhat_1}\right\} \cdot  \nuhat_H(x)} \leq \frac{Z \cdot \pistar_H(x)}{2\einf \Zhat_1} ,
\end{align}
where we have used \cref{lem:SMC-unbiased}.
Moreover, using that $2\einf \Zhat_1 \geq 2Z \geq \What_H$ with probability $1-\frac{\delta}{2\einf^2}$ by \cref{lem:linf-smc-concentration}, for any subset $X \subset \MX_H$, 
\begin{align}
\phat(X) =   \E\brk*{\min \left\{ 1, \frac{\What_H}{2\einf \Zhat_1}\right\} \cdot  \nuhat_H(X)} \geq \frac{Z \cdot \pistar_H(x)}{2\einf\Zhat_1} - \frac{\delta}{2\einf^2}. %
\end{align}
In particular, it holds that $\frac{Z}{2\einf \Zhat_1}-\frac{\delta}{4\einf^2}\leq \phat(x\neq \text{restart})\leq \frac{Z}{2\einf \Zhat_1}$, and hence using $\Zhat_1 / Z \leq \einf$, we know that for any $X \subset \MX_H$
\begin{align}
  \pistar_H(X)-\delta\leq \phat(X\mid{}\text{not restart})\leq (1+\delta)\pistar_H(X),
\end{align}
and hence $\Dtv{\pistar_H}{\phat(\cdot\mid{}\text{not restart})}\leq \delta$. Finally, using the guarantee of rejection sampling, we know the number of outer loop steps is bounded by $O(\einf^2\log(1/\delta))$ 
by noting that $\Zhat_1 / Z \leq \einf$. 
\end{proof}

\subsection{Lower Bound for SMC with Exact Value Function}

In the following proposition, we show that the horizon dependence of SMC when $\Vhat=\Vstar$ is not an artifact of our analysis.
\begin{proposition}\label{prop:SMC-lower-exact}
There exists a problem instance where $\Vhat=\Vstar$ and \cref{asmp:act-cov} holds with $\Cact=2$, however SMC with $N$ particles must incur
\begin{align}
  \Dtv{\En[\nuhat_H]}{\pistar_H}\geq c\min\crl*{1,\frac{\sqrt{H}}{N}}.
\end{align}
\end{proposition}

\begin{proof}
We construct a problem instance as follows. Fix $\lambda>0$. We consider $\cX_h=\crl{0,1}^h$, and define the transition kernel $P$ and value function $\Vhat=\Vstar$ as follows:
\begin{itemize}
  \item For any $x$, let $\piref(\cdot\mid{}x)=\unif(\crl{(x,0), (x,1)})$.
  \item Let $\Vhat=\Vstar:\cX\to \RR$ be defined recursively as
\begin{align}
  \Vstar(\perp)=0, \qquad \Vstar(x,0)=\Vstar(x), \quad
  \Vstar(x,1)=(1+\lambda)\Vstar(x).
\end{align}
  \item Then, it holds that $\Cact\leq 1+\lambda$, and $\pistar_H=\Ber\prn*{p}^{\otimes H}$, where $p=\frac{1+\lambda}{2+\lambda}$.
\end{itemize}
Under this construction, we can rewrite SMC as follows (where $\nuhat_0=\delta_{\perp}$). 
\begin{itemize}
\item For $h\geq 1$, sample $N$ particles $\tx_{h-1}\ind{1},\ldots,\tx_h\ind{N}\iidsim \nuhat_{h-1}(\cdot{})$ and generates $a_h\ind{i}\sim \unif(\crl{0,1})$ independently. 
\item Assign weights $w_h\ind{i}=1+\lambda a_h\ind{i}$ to particle $x_h\ind{i}=(\tx_{h-1}\ind{i},a_h\ind{i})$.
\item Define weighted empirical measure $\nuhat_h\ldef{} \frac{1}{W_h}\sum_{i=1}^{N} w_h\ind{i}\dirac_{x_h\ind{i}}$, where $W_h=N+\lambda \sum_{i=1}^N a_h\ind{i}$.
\end{itemize}

We make the following claims, proven below.

\begin{claim}\label{clm:SMC-lower-tensor}
It holds that $\En[\nuhat_H]=\Ber(p_N)^{\otimes H}$, where $p_N$ is defined in \cref{pfeq:SMC-lower-pN}.
\end{claim}

\begin{claim}\label{clm:SMC-lower-gap}
It holds that $p_N\leq \frac{1+\lambda}{2+\lambda} - \frac{c_\lambda}{N}$, where $c_\lambda=\frac{2\lambda(1+\lambda)}{(2+\lambda)^3}$.
\end{claim}

Now, we can lower bound
\begin{align}
  \Dtv{\En[\nuhat_H]}{\pistar_H}=\Dtv{\Bin(H,p)}{\Bin(H,p_N)}\geq c\min\crl{1,\sqrt{H}\abs{p-p_N}},
\end{align}
where $c>0$ is an absolute constant. Choosing $\lambda=1$ completes the proof of \cref{prop:SMC-lower-exact}.
\end{proof}

We prove the auxiliary claims:

\begin{proof}[\pfref{clm:SMC-lower-tensor}]
  By definition, for any $x\in\cX_{h-1}$ and $a\in\crl{0,1}$,
\begin{align}
  \nuhat_{h}(x,1)=\frac{1}{W_h}\sum_{i=1}^{N} w_h\ind{i}\indic\crl{\tx_{h-1}\ind{i}=x, a_h\ind{i}=a}.
\end{align}
Note that $\crl{a_h\ind{i}}$ and $\crl{\tx_{h-1}\ind{i}}$ are independent, we can calculate
\begin{align}
  \En[\nuhat_{h}(x,1)]
  =&~ \En\brk*{\frac{1}{W_h}\sum_{i=1}^{N} w_h\ind{i}\indic\crl{a_h\ind{i}=1}\cdot \indic\crl{\tx_{h-1}\ind{i}=x}} \\
  =&~\En\brk*{ \nuhat_{h-1}(x) \cdot \frac{(1+\lambda)\sum_{i=1}^N a_h\ind{i}}{N+\lambda \sum_{i=1}^N a_h\ind{i}} }.
\end{align}
Note that $A_h=\sum_{i=1}^N a_h\ind{i}\sim \Bin(N,\frac12)$ is independent of $\nuhat_{h-1}$. Then, we know 
\begin{align}\label{pfeq:SMC-lower-pN}
  \En[\nuhat_{h}(x,1)]
  =p_N\En\brk*{ \nuhat_{h-1}(x)  }, \qquad
  p_N=(1+\lambda )\En_{A\sim \Bin(N,\frac12)}\brk*{\frac{A}{N+\lambda A}}.
\end{align}
Applying this equation recursively gives the desired result.
\end{proof}

\begin{proof}[\pfref{clm:SMC-lower-gap}]
  We let $Y=\frac{2A}{N}-1$ so that $Y_h$ symmetrically distributed on $[-1,1]$. Then we can write
\begin{align}
  \frac{N}{N+\lambda A}=\frac{2}{2+\lambda(Y+1)}=\frac{2}{2+\lambda}\cdot \frac{1}{1+\frac{\lambda}{2+\lambda} Y}=\frac{2}{2+\lambda}\sum_{k\geq 0} (-1)^k \prn*{\frac{\lambda Y}{2+\lambda}}^k.
\end{align}
Note that $\En Y^k=0$ for odd integer $k$ and $\En Y^2=\frac1N$, we can conclude
\begin{align}
  \En\brk*{\frac{N}{N+\lambda A}}\geq \frac{2}{2+\lambda}\prn*{1+\prn*{\frac{\lambda}{2+\lambda}}^2\frac{1}{N}}.
\end{align}
Reorganizing yields
\begin{align}
  p_N=(1+\lambda )\En\brk*{\frac{A}{N+\lambda A}}
  =\frac{1+\lambda}{\lambda}\prn*{1-\En\brk*{\frac{N}{N+\lambda A}}}
  \leq \frac{1+\lambda}{2+\lambda} - \frac{2\lambda(1+\lambda)}{(2+\lambda)^3N}.
\end{align}
\end{proof}

\section{Proofs for \DMC}\label{appdx:DMC}

Analogously to \cref{appdx:SMC}, 
we introduce the following quantity for \cref{alg:dmc-implementable}:
\begin{align}
  W_h\ldef \En_{x_{h-1}\sim \nuhat_{h-1}, x_h\sim \piref(\cdot\mid x_{h-1})}\brk*{ \frac{\Vhat(x_h)}{\Vhat(x_{h-1})} }, \qquad 
  \What_h\ldef \prod_{\ell=1}^h W_\ell.
\end{align}
We define $\cH_h = \crl{\cS_1,\cdots,\cS_{h}}$ to be the full history of the execution of \DMC (\cref{alg:dmc-implementable}) up to the end of iteration $h$. We write $\En_h[\cdot]$ to be the conditional distribution with respect to $\cH_h$.

\begin{lemma}%
  \label{lem:DMC-unbiased}
  For step $h\in[H]$ of \DMC (\cref{alg:dmc-implementable}), given $\cH_{h-1}$, the set $\cS_h$ consists of i.i.d. samples from the distribution $\muhat_h$ given by
\begin{align}
  \muhat_h(x_h)=\frac{1}{W_h}\En_{x_{h-1}\sim \nuhat_{h-1}}\brk*{ \piref(x_h\mid x_{h-1}) \frac{\Vhat(x_h)}{\Vhat(x_{h-1})} }, \qquad \forall x_h\in\cX_h.
\end{align}
  Hence, for any $g_h:\cX_h\to\RR$, it holds that
  \begin{align}
    \label{eq:DMC-unbiased-g}
        \En\brk*{\What_h\cdot\En_{\nuhat_h}\brk*{g_h}}= \En\brk*{\What_h\cdot\En_{\muhat_h}\brk*{g_h}} = \frac{1}{\Zhat_1}\En_{x_h \sim \pi_h}\brk*{\Vhat(x_h)g_h(x_h)}.
  \end{align}
  
\end{lemma}
\begin{proof}[\pfref{lem:DMC-unbiased}]
  The first claim follows immediately from the analysis of rejection sampling: For any $x_h\in\cX_h$, by Bayes' rule, we have
  \begin{align}
    \bbP(x_h\mid \text{accept})\propto \bbP(\text{accept}\mid x_h)\cdot \bbP(x_h)\propto \muhat_h(x_h).
  \end{align}

  To prove the second claim, we prove the following result, from which \cref{eq:DMC-unbiased-g} follows inductively: 
    \begin{align}
    \label{eq:DMC-unbiased}
        \En_{h-1}\brk*{W_h\cdot\En_{\nuhat_h}\brk*{g_h}} = \En_{h-1}\brk*{W_h\cdot\En_{\muhat_h}\brk*{g_h}} = \En_{x_{h-1}\sim \nuhat_{h-1}, x_h\sim \piref(\cdot\mid x_{h-1})}\brk*{ \frac{\Vhat(x_h)g_h(x_h)}{\Vhat(x_{h-1})} }.
  \end{align}
\end{proof}

In particular, $\En[\What_h]=\Zhat_h/\Zhat_1$ for $h\geq 1$ and
  \begin{align}
    \label{eq:DMC-R-unbiased-nu}
        \En\brk*{\What_H \nuhat_H(x)} = Z\cdot \pistar_H(x)/\Zhat_1, \qquad \forall x\in\cX_H.
  \end{align}
Further, we define $\rho(x)=\frac{\Vstar(x)}{\Vhat(x)}$ and
\begin{align}
  M_h=\What_h \cdot \En_{x_h\sim \nuhat_h}[\rho(x_h)].
\end{align}
Then \cref{lem:DMC-unbiased} with $g(x) = \rho(x)$ implies $(M_h)$ is a martingale with $M_0=Z$ and $M_H=\What_H$. 
Therefore, similar to \cref{lem:SMC-TV}, we have the following upper bound.
\begin{lemma}\label{lem:DMC-TV}
  
  For \cref{alg:dmc-implementable}, it holds that
  \begin{align}
    \Dtv{\En[\nuhat_H]}{\pistar_H}\leq \frac1{2Z}\En\abs{\What_H-Z}
    \leq \frac1{2Z}\sum_{h=1}^H \En\abs{M_h-M_{h-1}}.
  \end{align}
\end{lemma}

\subsection{\pfref{thm:dmc-chisq}}\label{appdx:pf-dmc-chisq}

Note that $M_h-M_{h-1}=\What_{h}\prn*{\En_{\nuhat_h}[\rho]-\En_{\muhat_h}[\rho]}$, and $\nuhat_h=\unif(\cS_h)$, where $\cS_h$ consists of i.i.d samples from $\muhat_h$. 
Then, we can bound
\begin{align}
  \En_{h-1}\abs{\En_{\nuhat_h}[\rho_h]-\En_{\muhat_h}[\rho_h]}\leq &~ 
  \sqrt{\En_{h-1}\abs*{\frac1N\sum_{x_h\in\cS_h}\rho(x_h)-\En_{\muhat_h}[\rho]}^2 }
  =\sqrt{\frac{\En_{\muhat_h}[\rho(x_h)^2]-\prn*{\En_{\muhat_h}[\rho(x_h)]}^2 }{N}},
\end{align}
and hence
\begin{align}
  \sqrt{N}\En\abs{M_h-M_{h-1}}
  \leq&~ \En\brk*{ \What_{h}\sqrt{\En_{\muhat_h}[\rho(x_h)^2]-\prn*{\En_{\muhat_h}[\rho(x_h)]}^2 }} \\
  \leq&~ \sqrt{\En[\What_h]\prn*{\En\brk*{\What_h\En_{\muhat_h}[\rho(x_h)^2]}-\En\brk*{\What_h\prn*{\En_{\muhat_h}[\rho(x_h)]}^2}} }\\
  \leq&~ \sqrt{\Zhat_h\En_{x_h\sim \pi_h}\brk*{ \frac{\Vstar(x_h)^2}{\Vhat(x_h)} }-Z^2} \\
  =&~Z\sqrt{\En_{x_h\sim \pistar_h}\brk*{ \frac{\pistar(x_h)}{\pihat(x_h)} }-1} = Z\sqrt{\Dchis{\pistar_h}{\pihat_h}},
\end{align}
where the third line uses $\En[\What_h]\En\brk*{\What_h\prn*{\En_{\muhat_h}[\rho(x_h)]}^2}\geq \prn*{\En[\What_h\En_{\muhat_h}[\rho(x_h)]]}^2=Z^2$ (Cauchy inequality and \cref{lem:DMC-unbiased}), and the last line follows from the definition of $\pistar$, $\pihat$, and $\chi^2$-divergence.
Taking summation completes the proof.
\qed

\subsection{Analysis of \DMC with Heavy-tailed Errors}

Parallel to SMC, we show that \DMC also enjoys an upper bound that depends only on coverage (\cref{def:coverage}). 
\begin{theorem}\label{thm:dmc-coverage}
Under \cref{asmp:act-cov-Vhat}, suppose that \DMC (\cref{alg:dmc-implementable}) is instantiated with $\eta\geq \Cacthat$. It holds that for any $M\geq 1$, \cref{alg:dmc-implementable} achieves
\begin{align}
  \Dtv{\En[\nuhat_H]}{\pistar_H}\leq H\sqrt{\frac{M}{N}}+\sum_{h=1}^H\Dcov[M]{\pistar_h}{\pihat_h}.
\end{align}
\end{theorem}
The convergence rate of \cref{thm:dmc-coverage} improves upon that of SMC (\cref{thm:smc-coverage}) by a factor of $\Cact$, though the same factor reappears in the time complexity. The expected time complexity bound matches \cref{thm:dmc-chisq}.

\begin{proof}
We follow the proof of \cref{thm:dmc-chisq} (\cref{appdx:pf-dmc-chisq}).
By \cref{prop:MZ-inequality}, it holds that
\begin{align}
  \En_{h-1}\abs{\En_{\nuhat_h}[\rho_h]-\En_{\muhat_h}[\rho_h]}=&~ 
  \En_{h-1}\abs*{\frac1N\sum_{x_h\in\cS_h}\rho(x_h)-\En_{\muhat_h}[\rho]} \\
  \leq&~ \sqrt{\frac{R}{N}\En_{x_h\sim \muhat_h}[\rho(x_h)]}+2R\En_{x_h\sim \muhat_h}[\rho(x_h)\indic\crl{\rho(x_h)\geq R}],
\end{align}
where we use $\cS_h$ consists of $N$ i.i.d samples from $\muhat_h$. Therefore, we can bound
\begin{align}
  \En\abs{M_h-M_{h-1}}
  \leq 2\sqrt{\frac{R}{N}}\En\brk*{\What_h\sqrt{\En_{x_h\sim \muhat_h}[\rho(x_h)]}}+2\En\brk*{\What_h\En_{x_h\sim \muhat_h}[\rho(x_h)\indic\crl{\rho(x_h)\geq R}]}
\end{align}
Further, by \cref{lem:DMC-unbiased}, we have
\begin{align}
  \En\brk*{\What_h\sqrt{\En_{x_h\sim \muhat_h}[\rho(x_h)]}}\leq \sqrt{\En[\What_h]\En[\What_h\En_{x_h\sim \muhat_h}[\rho(x_h)]]}\leq \sqrt{\Zhat_h \cdot Z},
\end{align}
and 
\begin{align}
  \En\brk*{\What_h\cdot \En_{x_h\sim \muhat_h}[\rho(x_h)\indic\crl{\rho(x_h)\geq R}]}=Z\cdot \bbP_{x\sim \pistar_h}\prn*{\rho(x_h)\geq R}.
\end{align}
Choosing $R=M\frac{Z}{\Zhat_h}$ gives 
\begin{align}
  \En\abs{M_h-M_{h-1}}
  \leq 2Z\sqrt{\frac{M}{N}}+2Z\Dcov[M]{\pistar_h}{\pihat_h}.
\end{align}
Taking summation completes the proof.
\end{proof}

\subsection{Faster Convergence for \DMC}\label{sec:dmc-restart}

In this section, we present and analyze a modification of \DMC, which we call \DMC with restart (\cref{alg:dmc-restart}). Under the following assumption, we prove that \cref{alg:dmc-restart} achieves exponential convergence (analogous to \cref{thm:smc-linf} for SMC).
\begin{assumption}\label{asmp:finite-piref}
There are at most $\eta$ values of $x_{h+1}$ for each $x_h$ so that $\piref(x_{h+1} \mid x_h) > 0$, and these are efficiently enumerable.
\end{assumption}

\begin{algorithm}
  \caption{Diffusion Monte Carlo}\label{alg:dmc-restart}
\begin{algorithmic}
\State \textbf{Input:} Transition kernel $\piref$, value function $\Vhat$.
\State \textbf{Parameter:} initial sample size $N\geq 1$. %
\State Initialize $\MS_0$ to be $N$ copies of $ \perp$.
\For{$h=1,\cdots,H$}
\State For each $x_{h-1}\in\cS_{h-1}$, compute $\Vtil(x_{h-1})=\En_{\piref}[\Vhat(x_h)\mid{} x_{h-1}]$. 
\State Define $W_{h-1}=\frac1N\sum_{x_{h-1}\in \cS_{h-1}} \frac{\Vtil(x_{h-1})}{\Vhat(x_{h-1})}$.
\State Define the weighted empirical measure $\wt{\mu}_{h-1}=\frac{1}{NW_{h-1}} \sum_{x_{h-1}\in \cS_{h-1}} \frac{\Vtil(x_{h-1})}{\Vhat(x_{h-1})} \dirac_{x_{h-1}}$.
\State Generates $\tx_{h-1}\ind{i}\sim \wt{\mu}_{h-1}, x_h\ind{i}\sim \pihat(\cdot\mid{}\tx_{h-1}\ind{i})$ independently for $i\in[N]$.
\State Define $\cS_h=\crl{x_h\ind{i}}_{i\in[N]}$ and $\nuhat_h=\Unif(\cS_h)$.
\EndFor
\State Define $\What_H=\prod_{h=1}^H W_h$.
\State \textbf{Output:} With probability $\min\crl*{\frac{\What_H}{ \Ztil_1 }, 1}$, output $x\sim \nuhat_H$, and otherwise restart the algorithm. 
\end{algorithmic}
\end{algorithm}

We prove the following guarantee:
\begin{theorem}\label{thm:dmc-linf-full}
Suppose that \cref{asmp:L-inf} and \cref{asmp:finite-piref} hold, and we denote $\eps\ldef \einf-1$. Further, suppose that we have $Z\leq \Ztil_1\leq J Z$. Then \cref{alg:dmc-restart} with \textbf{Output Option 2} samples from a distribution $\mu$ that satisfies
\begin{align}
  \Dtv{\mu}{\pistar_H}\leq \delta,
\end{align}
as long as $N \geq \Omega(H\einf^2 \eps^2 \log(JH/\delta))$, using a number of outer loop steps that is bounded by $O(J^2\log(1/\delta))$.
\end{theorem}

Under \cref{asmp:finite-piref}, we can always choose $\Ztil_1=\einf \Zhat_1$, so that $Z\leq \Ztil_1\leq \einf^2 Z$. Further, we can run \cref{alg:dmc-restart} 1 time solely for estimating $Z$, with choice $\Ztil_1=2\What_H$ satisfies $Z\leq \Ztil\leq 4Z$ with high probability.

Similar to the proof of \cref{thm:smc-linf}, we first state and prove the following lemma (as an analogue to \cref{lem:linf-smc-concentration}). \cref{thm:dmc-linf-full} then follows from the argument in \cref{appdx:pf-smc-linf}.
\begin{lemma}\label{lem:linf-dmc-concentration}
Under the condition of \cref{thm:dmc-linf-full}, for any $\delta\in(0,1)$, if we run \DMC (\cref{alg:dmc-restart}) with $N = \Omega (H \einf^2 \eps^2 \log(H/\delta))$ particles, then with probability at least $1-\delta$, it holds that $|\What_H - Z| \leq Z/2$.
\end{lemma}

\begin{proof}[\pfref{lem:linf-dmc-concentration}]
We follow the proof of \cref{thm:dmc-chisq} (\cref{appdx:pf-dmc-chisq}). Recall that
\begin{align}
  \frac{M_h-M_{h-1}}{\What_h}
  =&~ \frac1N\prn*{ \sum_{i=1}^N\rho(x_h\ind{i})-\En_{\muhat_h}[\rho] },
\end{align}
where $x_h\ind{1},\cdots,x_h\ind{N}$ are $N$ \iid samples from $\muhat_h$ (given $\cH_{h-1}$). 
Note that by \cref{asmp:L-inf}, it holds that $\rho(x_h)-1\in [-\eps,\eps]$ for any $x_h\in\cX_h$. 
This implies that $\E_{x_h\sim \muhat_h}[e^{\lambda(\rho(x_h) - \E_{\muhat_h}[\rho])}] \leq e^{\lambda^2 \eps^2/8}$ for $\lambda\in\RR$, and hence
\begin{align}
& \E_{h-1}[\exp(\lambda(M_h - M_{h-1}))] \\
= & \E_{h-1}\brk*{\exp\prn*{\lambda \What_h \cdot \prn*{\frac1N\prn*{ \sum_{i=1}^N\rho(x_h\ind{i})-\En_{\muhat_h}[\rho] }}}}\\
 \leq & {\exp\prn*{\frac{1}{N}\lambda^2 \eps^2 \What_{h}^2}} \leq  {\exp\prn*{\frac{1}{N}\lambda^2 \einf^2 \eps^2 M_{h-1}^2}},
\end{align}
where we use $M_{h-1}=\En_{h-1}[M_h]=\What_{h}\En_{\muhat_h}[\rho]\geq \einf^{-1}\What_h$.

By Freedman's inequality, it follows that for any $h \in [H]$, with probability $1-\delta$, we have that 
\begin{align}
|M_h - Z| \leq \sum_{i=1}^{h-1} \frac{\lambda \einf^2 \eps^2 M_{i}^2}{N} + \frac{\log (2/\delta)}{\lambda}.
\end{align}
Choosing $\lambda = 4 \log(2H/\delta) / Z$ and using the union bound, it follows that with probability $1-\delta$, for all $h \in [H]$,
\begin{align}
|M_h - Z| \leq&~ \frac{4 \log(2/\delta) \einf^2 \eps^2}{N} \sum_{i=0}^{h-1} \frac{ M_{i}^2}{ Z} + \frac{Z}{4}.
\end{align}
Since we have chosen $N = \Omega(H \log(H/\delta) \einf^2 \eps^2)$, it follows from induction on $h$ that in the event that \cref{eq:freedman-mh} holds, we have that $|M_h - Z| \leq Z/2$ for all $h \in [H]$. In particular, this implies that $|\What_H - Z| \leq Z/2$ with probability at least $1-\delta$.
\end{proof}

\section{A Unified Perspective on Backtracking and Particle Filtering}\label{app:coupling}

\newcommand{\Fdmc}{\mathcal{F}_{\mathsf{DMC}}} 
\newcommand{\Fvgb}{\mathcal{F}_{\mathsf{VGB}}}
\newcommand{\Sdmc}{\mathcal{S}_{\mathsf{DMC}}}
\newcommand{\Svgb}{\mathcal{S}_{\mathsf{VGB}}}
\newcommand{\Pdown}{\BP_{\mathsf{down}}}
\newcommand{\aux}{\mathfrak{s}}
\newcommand{\Tauto}{\mathcal{T}}

Our main theoretical results demonstrated that particle filtering methods like SMC can achieve strong guarantees under the same assumptions on $\Vhat$ as the backtracking-based method VGB introduced by \cite{rohatgi2025taming} --- while improving parallel runtime from $\bigoht(H^2)$ to $\bigoht(H)$. In this section, we show that the connection goes deeper: for \emph{any} problem instance, the execution of VGB can be \emph{coupled} with the execution of \SMCIND, a variant of SMC in which the children of each particle are sampled independently.

We defer formal definitions until after stating the result (\cref{thm:dmc-vgb-main}), but we make several remarks for context. Henceforth, we specify to the autoregressive setting (\cref{remark:autoregressive}) in which VGB was defined \cite{rohatgi2025taming}. Recall that VGB implements a single-particle random walk in which the allowed transitions can not only add a token to the current particle (i.e. $a_{1:h} \mapsto a_{1:h+1}$) but also backtrack (i.e. $a_{1:h} \mapsto a_{1:h-1}$). For technical reasons, we augment the state space of the random walk with a state $\aux$ that is only adjacent to the empty string $\perp$. See \cref{sec:smcind,sec:vgb} for the formal definitions of \SMCIND{} and the VGB random walk.

\begin{theorem}\label{thm:dmc-vgb-main}
Let $\Sdmc$ be the (random) multiset of particles visited by \SMCIND{} with $N$ initial particles. Let $\Svgb$ be the (random) multiset of particles visited by VGB (initialized at $\aux$) before the $N+1$-st visit to $\aux$. Then there is a coupling so that $\Sdmc = \Svgb \setminus \{\aux\}$ (with multiplicity) almost surely. 
\end{theorem}

Intuitively, each particle $a_{1:h}$ visited by $\SMCIND$ corresponds to an interval of time in which the VGB random walk started at $a_{1:h}$, ended at $a_{1:h}$, and never visited the parent node $a_{1:h-1}$ (where we take the parent of $\perp$ to be $\aux$). This demonstrates that the paradigm of backtracking to tame error amplification \cite{rohatgi2025taming} is, in a sense, equivalent to a form of particle filtering.

\paragraph{Outline of appendix} In \cref{sec:vgb} we introduce the VGB random walk (slightly modified for technical reasons). In \cref{sec:smcind} we introduce the \SMCIND{} stochastic process. In \cref{sec:coupling-proof} we prove \cref{thm:dmc-vgb-main}.

\subsection{Background on VGB}\label{sec:vgb}

\newcommand{\Pvgb}{\BP_{\mathsf{VGB}}}

The (non-lazy) VGB random walk \cite{rohatgi2025taming} is defined on state space $\cX = \cX_0 \sqcup \dots \sqcup \cX_H$ where $\cA$ is the token space and $\cX_h := \cA^h$. For technical convenience (specifically, to avoid casework at step $0$ when coupling VGB with \SMCIND), we introduce an auxiliary state $\aux$ that is not in the original definition. We define the transition kernel $\Pvgb$ as follows. First, $\aux$ always transitions to $\perp$:
\[\Pvgb(\perp \mid{} \aux) = 1.\] Second, any $a_{1:h} \in \cA^h$ with $0 \leq h < H$ can transition to either a ``child'' $a_{1:h+1}$ or the ``parent'' $a_{1:h-1}$:
\[\Pvgb(a_{1:h+1} \mid{} a_{1:h}) = \frac{\piref(a_{h+1} \mid{} a_{1:h}) \Vhat(a_{1:h+1})}{\Vhat(a_{1:h}) + \sum_{a_{h+1}'} \piref(a_{h+1}'\mid{} a_{1:h}) \Vhat(a_{1:h},a_{h+1}')},\]
\[\Pvgb(a_{1:h-1} \mid{} a_{1:h}) = \frac{\Vhat(a_{1:h})}{\Vhat(a_{1:h}) + \sum_{a_{h+1}'} \piref(a_{h+1}'\mid{} a_{1:h}) \Vhat(a_{1:h},a_{h+1}')},\]
with the convention that $a_{1:0} = \perp$ (i.e. the empty string) and $a_{1:-1} = \aux$. Third, any $a_{1:H} \in \cA^H$ must transition to its parent:
\[\Pvgb(a_{1:H-1} \mid{} a_{1:H}) = 1.\]
The transition kernel of the original VGB random walk \cite{rohatgi2025taming} is identical conditioned on never visiting $\aux$. Moreover, since $\aux$ is only reachable from $\perp$ (and only transitions to $\perp$), our modification only affects the execution of VGB by stochastically inserting copies of the sequence $(\aux, \perp)$ into VGB's random trajectory.

For notational convenience, it is useful to define a transition kernel $\Pdown$ as the VGB transition kernel conditioned on ``downward'' moves, i.e. 
\[\Pdown(a_{1:h+1} \mid{} a_{1:h}) = \frac{\piref(a_{h+1} \mid{} a_{1:h}) \Vhat(a_{1:h+1})}{\sum_{a_{h+1}'} \piref(a_{h+1}'\mid{} a_{1:h}) \Vhat(a_{1:h},a_{h+1}')}.\]

\subsection{Definition of \SMCIND}\label{sec:smcind}

In \cref{alg:smcind}, we describe how \SMCIND{} stochastically constructs multisets of particles $\cS_0,\dots,\cS_H$ for a given problem instance. Note that there is no specified output mechanism, since the coupling between \SMCIND{} and VGB will apply to the entire multiset of particles.

\begin{algorithm}
  \caption{Sequential Monte Carlo with Independent Sampling (\SMCIND)}\label{alg:smcind}
\begin{algorithmic}
\State \textbf{Input:} Autoregressive language model $\piref$, value function $\Vhat$.
\State \textbf{Parameter:} initial sample size $N\geq 1$.
\State Initialize $\MS_0$ to be $N$ copies of $ \perp$.
\For{$h=1,\cdots,H$}
\State Initialize a multi-set $\cS_h=\crl{}$.
\For{$a_{1:h-1} \in \cS_{h-1}$}
\State Sample $D \sim \Geom(\Pvgb(a_{1:h-2}\mid{}a_{1:h-1}))$ (with convention $a_{1:-1}=\aux$; see \cref{sec:vgb})
\For{$1 \leq i \leq D$}
\State Sample $a_{1:h} \sim \Pdown(\cdot\mid{}a_{1:h-1})$.
\State Update $\cS_{h} \gets \cS_h \cup \{a_{1:h}\}$. 
\EndFor
\EndFor
\EndFor
\end{algorithmic}
\end{algorithm}

\subsection{Coupling VGB with \SMCIND}\label{sec:coupling-proof}

We interpret the executions of \SMCIND{} and VGB as stochastically constructing forests of particles. We then show that these forests are actually the same (up to occurrences of $\aux$) when the randomness is appropriately coupled.

\paragraph{Infinite-particle \SMCIND{} process} The execution of \SMCIND{} with $N=\infty$ naturally defines a rooted forest $\Fdmc$ where each node is labelled by a particle in $\cS_0 \sqcup \dots \sqcup \cS_H$. The roots are the particles in $\cS_0$. For any $h>0$, the parent of a particle in $\cS_h$ is the particle in $\cS_{h-1}$ that birthed it. Note that there may be multiple (or even infinitely many) occurrences of any particle $a_{1:h}$ in $\cS_h$, so $\Fdmc$ may contain multiple nodes with the same label. The parent of a node labelled $a_{1:h}$ always has label $a_{1:h-1}$.

\paragraph{Infinite-time VGB process} We next show how an execution of VGB in the autoregressive tree $\Tauto$ can be interpreted as a rooted tree $\Fvgb$, similar to above. Fix an infinite trajectory $(x_0, x_1,x_2,x_3,\dots)$ of VGB with $x_0=\aux$. Suppose that $\aux$ occurs infinitely often in the trajectory. The root node of $\Fvgb$ is indexed by $\NN$, and labeled $\aux$. Every other node is indexed by an interval $[t,t'] \subset \NN$ with the property that $x_{t-1} = x_{t'+1}$ is the parent of $x_t = x_{t'}$ in $\Tauto$, and this parent does not occur within $\{x_t,x_{t+1},\dots,x_{t'}\}$. This node is labeled with $x_t$, and its parent in $\Fvgb$ is the unique node with label $x_{t-1}$ whose interval contains $[t,t']$.

\begin{lemma}\label{lemma:infty-coupling}
Suppose that VGB almost surely visits $\aux$ infinitely often. There is a coupling of the randomness of $\Fdmc$ and $\Fvgb$ such that $\Fdmc$ is almost surely equal to the forest $\Fvgb \setminus \{\aux\}$.
\end{lemma}

\begin{proof}
We construct the coupling layer by layer. By the lemma assumption, it is clear that the first layer of $\Fdmc$ is almost surely identical to that of $\Fvgb$. Now fix $h \geq 1$ and condition on the forest structure and labels $x_t$ (but not the indices $[t,t']$ themselves) of the first $h$ layers of $\Fvgb$. Fix any node $\iota$ with label $a_{1:h-1}$. 

Consider the law of the VGB random walk between these two visits to $a_{1:h-1}$. The walk is restricted to the subtree rooted at $a_{1:h-1}$, but since the walk is Markovian, its law is otherwise unchanged by the conditioning (and it is also independent of the children of other nodes $\iota' \neq \iota$ at layer $h$). Let $p = \BP(a_{1:h-2}\mid{}a_{1:h-1})$ (with the convention $a_{1:h-2} = \aux$ if $h=1$). The number of children $D'$ of $\iota$ is the number of times that the walk goes to a child of $a_{1:h-1}$ in $\Tauto$ before it returns to $a_{1:h-2}$. Thus, $\Pr[D'=k] = (1-p)^k p$ for any $k \in \{0,1,2,\dots\}$. That is, $D'$ is a geometric random variable with success probability $p$. Moreover, the label of each child of $\iota$ is independent and distributed according to $\Pdown(\cdot\mid{}a_{1:h-1})$. Thus, the coupling can be extended to the children of $\iota$.

\end{proof}

\begin{proof}[Proof of \cref{thm:dmc-vgb-main}]
The first $N$ trees of $\Fdmc$ can be coupled to the forest induced by \SMCIND{} with $N$ initial particles. Similarly, the first $N$ trees of $\Fvgb$ correspond exactly to the prefix of the trajectory of VGB before $\aux$ is visited $N+1$ times. The claim then follows from \cref{lemma:infty-coupling}. 
\end{proof}

\section{Proof of \cref{thm:pf-lb-main}}\label{app:pf-lb}

\newcommand{\HH}{\mathfrak{H}}

In this appendix we formally state the requisite definitions for  \cref{thm:pf-lb-main}, and prove the theorem. In \cref{sec:myopic}, we formalize the definition of a myopic particle filtering algorithm. In \cref{sec:pf-lb-setting} we state \cref{thm:pf-lb}, which directly implies \cref{thm:pf-lb-main}. In \cref{sec:pf-lb-pf}, we provide intuition for and formally prove \cref{thm:pf-lb}.

\subsection{Myopic Particle Filtering Algorithms}\label{sec:myopic}

\begin{definition}\label{def:myopic}
A \emph{myopic particle filtering algorithm with $N$ particles} is defined as follows. The algorithm computes a stochastic sequence of multisets $\cS_0,\dots,\cS_H$ where $\cS_h \subset \cA^h$ for all $0 \leq h \leq H$, satisfying the following properties:
\begin{enumerate}
    \item It holds almost surely that for all $1 \leq h \leq H$, for each $y_{1:h} \in \cS_h$, we have $y_{1:h-1} \in \cS_h$.
    \item $|\cS_{h}| \leq N$ for all $h$ almost surely.
    \item The law of $(\cS_{0},\dots,\cS_{h})$ is determined by $(\piref(a_{1:h}): a_{1:h}\in\cA^h)$ and $(\Vhat(a_{1:k}): a_{1:k}\in \cA^k, k \leq h)$ for all $h$.
\end{enumerate}
The algorithm then (stochastically) outputs any element of $\cS_{H}$.
\end{definition}

The first condition is a natural formalization of what it means to be a particle filtering method, i.e. a particle in $\cS_{h}$ cannot appear ``out of nowhere''; rather, it must have a parent in $\cS_{h-1}$. The second condition asserts that the number of particles is always at most $N$. The third condition is the ``myopic'' condition, which enforces that, at the level of distributions, the algorithm does not use information from later steps $k > h$ to determine the multiset $\cS_{h}$.

\begin{remark}[Sequential Monte Carlo and variants]
It is straightforward from \cref{alg:smc} to see that Sequential Monte Carlo is indeed a myopic particle filtering algorithm, as is Sequential Monte Carlo with Rejection Sampling (\cref{alg:dmc-implementable}).
\end{remark}

\begin{remark}[Computational efficiency]
We remark that \cref{def:myopic} does not enforce any bound on computational efficiency or query complexity; in particular, a myopic particle filtering algorithm is allowed to examine $\Vhat(a_{1:h})$ for all $a_{1:h} \in \cA^h$ before deciding on the set of particles $\cS_{h}$; it is not restricted to e.g. only querying the children of particles in $\cS_{h-1}$. This only makes our impossibility result stronger. That said, it is an interesting open question to prove an analogous result to \cref{thm:pf-lb-main} that applies to \emph{all} algorithms that make $o(H^2)$ queries to $\piref$ and $\Vhat$. Such a result would be formally incomparable to \cref{thm:pf-lb-main}.
\end{remark}

\subsection{Specialized Setting and Formal Theorem Statement}\label{sec:pf-lb-setting}

Since we are proving a lower bound, it suffices to consider a special case of the general setting from \cref{sec:prelim}. Thus, we restrict to the autoregressive generation setting (\cref{remark:autoregressive}) and moreover set $\piref := \Unif(\cA^H)$ where $\cA = \{0,1\}$.

We define a problem instance by specifying some distribution $\mu \in \Delta(\cA^H)$ and functions $(\mu_h)_h$ with $\mu_h:\cA^h \to \RR_{\geq 0}$. These induce $\Vstar,\Vhat$ by
\[\Vstar(a_{1:h}) := \frac{\mu(a_{1:h})}{\piref(a_{1:h})}\]
and
\[\Vhat(a_{1:h}) := \frac{\muhat_h(a_{1:h})}{\piref(a_{1:h})}.\]
Note that the goal distribution $\pistar$ is precisely $\mu$, and the third condition of \cref{def:myopic} precisely enforces that the law of $(\cS_{0},\dots,\cS_{h})$ is determined by $(\muhat_0,\dots,\muhat_h)$. For \cref{asmp:act-cov} to hold with parameter $\Cact$, it suffices to have $\mu(a_{h+1}\mid{} a_{1:h}) \leq \Cact/2$ for all $a_{1:h+1} \in \cA^{h+1}$. Moreover, \cref{asmp:L-inf} holds with parameter $\einf$ if and only if $\muhat$ is $\einf$-accurate for $\mu$ in the following sense:

\begin{definition}
We say that $\muhat$ is $\kappa$-accurate for $\mu \in \Delta(\cA^H)$ if it holds that
\[\frac{\muhat(a_{1:h})}{\mu(a_{1:h})} \in [1/\kappa, \kappa]\]
for all $a_{1:h} \in \cA^h$, and moreover $\muhat(a_{1:H}) = \mu(a_{1:H})$ for all $a_{1:H} \in \cA^H$.
\end{definition}

We can now formally state the main result of this appendix, which directly implies \cref{thm:pf-lb-main}.

\begin{theorem}\label{thm:pf-lb}
Define $N(H) := \log(H)/(4\log \log(H))$ for all $H \in \NN$. There is no myopic particle filtering algorithm $\Alg$ with the following guarantee: fix $H$ sufficiently large. $\Alg$ uses at most $N(H)$ particles. For any $\mu\in\Delta(\cA^H)$, for any $e^3$-accurate $\muhat: \cA^{\leq H} \to \RR_{\geq 0}$, the execution of $\Alg$ on $\muhat$ has output $o_{1:H}$ satisfying
\[\Pr_{o_{1:H} \sim \Alg}[o_{1:H} \in \cE] \geq H^{-1/5}\]
for all events $\cE \subset \cA^H$ with $\mu(\cE) \geq 1/2$.
\end{theorem}

\begin{proof}[Proof of \cref{thm:pf-lb-main}]
Immediate from \cref{thm:pf-lb} and the preceding discussion on \cref{asmp:act-cov,asmp:L-inf}.
\end{proof}

\subsection{Proof of \cref{thm:pf-lb}}\label{sec:pf-lb-pf}

\subsubsection{Intuition and Formal Construction} 

\newcommand{\ol}{\overline}

To prove \cref{thm:pf-lb}, we construct a family of problem instances (each specified by a choice of $\mu$ and $\muhat_0,\dots,\muhat_H$) such that any myopic particle filtering algorithm with at most $N$ particles must fail on at least one instance in the family. We do so recursively, by first constructing an instance for $N=1$, then $N=2$, and so forth. The analysis will be similarly inductive.

Let $\HH:\NN\to\NN$ be a function to be determined. Assume that $\HH(N)$ is a multiple of $\HH(N-1)+2$ for all $N > 1$. For each $N\in\NN$ and $\ystar \in \cA^{\HH(N)}$, we define $\mu\ind{N}(\cdot;\ystar) \in \Delta(\cA^{\HH(N)})$ and $\muhat\ind{N}(\cdot; \ystar): \cA^{\leq \HH(N)} \to \RR_{\geq 0}$ recursively (in $N$) as follows.

\paragraph{Intuition for construction} First, we provide intuition for the $N=1$ construction, which essentially follows Example~3.2 in \citet{rohatgi2025taming}. We will have $\mu\ind{1}(\cdot;\ystar)$ be a product distribution where the marginal $\mu\ind{1}(y_h)$ is biased towards $\ystar_h$. We will then have $\muhat\ind{1}(\cdot;\ystar)$ be ``delayed'' compared to $\mu\ind{1}(\cdot;\ystar)$, i.e. the marginal at step $h$ is biased towards $\ystar_{h-1}$ instead of $\ystar_h$. Thus, the data of $\muhat\ind{1}_1(\cdot;\ystar),\dots,\muhat\ind{1}_h(\cdot;\ystar)$ is independent of $\ystar_h$, and so as formalized in \cref{lemma:base}, any $1$-particle myopic particle filtering method must incur large sampling error.

Next, we provide intuition for the $N>1$ construction. The horizon $H=\HH(N)$ is divided into blocks of size $K = \HH(N-1)+2$. Again $\mu\ind{N}(\cdot;\ystar)$ is a product distribution where the marginal at step $h$ is biased towards $\ystar_h$. In block $[iK+1,(i+1)K]$, we define $\muhat\ind{N}(\cdot;\ystar)$ by ``delaying'' the dependence on $\ystar_{iK+1}$ to step $(i+1)K$, and embedding a copy of the $(N-1)$-particle construction in the intervening steps. Intuitively, any $N$-particle myopic particle filtering algorithm has two choices at step $iK+1$: either irrevocably guess $\ystar_{iK+1}$ (i.e. all particles $y \in \cS_{iK+1}$ will have the same value of $y_{iK+1}$), or split up the particles so that some have different values of $y_{iK+1}$. The basic idea, formalized in \cref{lemma:induction-alg}, is that both choices lead to sampling error: the former because no data about $\ystar_{iK+1}$ is available at step $iK+1$, and the latter by the inductive hypothesis on the $(N-1)$-particle construction. The multiple blocks are needed to amplify the error induced by the first choice.

Making this argument precise requires a careful probabilistic analysis to ensure that the errors in the two cases cannot ``cancel out''. Before proceeding to this analysis, we make the construction formal.

Let $\gamma_1,\gamma_2,\dots \in (0,1/2)$ be a sequence of parameters to be determined later.

\paragraph{Formal construction: base case ($N=1$)} Write $H=\HH(1)$. We define
\[\mu\ind{1}(y_h\mid{}y_{1:h-1};\ystar) := (1/2-\gamma_1)^{\indic[y_h\neq\ystar_h]} (1/2+\gamma_1)^{\indic[y_h=\ystar_h]}\]
and
\[\muhat\ind{1}(y_h\mid{}y_{1:h-1}; \ystar) := \begin{cases} 1/2 & \text{ if } h=1 \\ 
\mu_1(y_{h-1} \mid{} y_{1:h-2}; \ystar) & \text{ if } 1 < h < H \\ 
2\mu_1(y_{H-1:H} \mid{} y_{1:H-2}; \ystar) & \text{ if } h = \HH(1) \end{cases}.\]

\paragraph{Formal construction: recursive case ($N>1$)} Define $K := \HH(N-1) + 2$. For any nonnegative integer $i$ and $h \in [\HH(N)]$ with $iK < h \leq (i+1)K$, we define
\[\mu\ind{N}(y_h\mid{}y_{1:h-1};\ystar) := \begin{cases} 
(1/2-\gamma_N)^{\indic[y_h\neq\ystar_h]} (1/2+\gamma_N)^{\indic[y_h=\ystar_h]} & \text{ if } h=iK+1 \text{ or } h=(i+1)K\\ 
\mu\ind{N-1}(y_h\mid{}y_{iK+2:h-1};\ystar_{iK+2:(i+1)K-1}) & \text{ otherwise}
\end{cases}\]
and
\[\muhat\ind{N}(y_h \mid{} y_{1:h-1};\ystar) := \begin{cases} 
1/2 & \text{ if } h=iK+1 \\ 
\muhat\ind{N-1}(y_h \mid{} y_{iK+2:h-1}; \ystar_{iK+2:(i+1)K-1}) & \text{ if } iK+1 < h < (i+1)K \\ 
2\mu\ind{N}(y_{iK+1} \mid{} y_{1:iK};\ystar)\mu\ind{N}(y_{(i+1)K} \mid{} y_{1:(i+1)K-1};\ystar) & \text{ if } h=(i+1)K
\end{cases}\]

In the following lemma we check that the desired accuracy condition is satisfied for all $N$ and $\ystar$.

\begin{lemma}\label{lemma:pf-construction-accuracy}
For any $N \geq 1$ and $\ystar \in \cA^{\HH(N)}$, it holds that $\muhat\ind{N}(\cdot;\ystar)$ is $\exp(3\sum_{i=1}^N \gamma_i)$-accurate for $\mu\ind{N}(\cdot;\ystar)$.
\end{lemma}

\begin{proof}
We proceed by induction on $N$. Fix $N=1$. Set $H = \HH(1)$. We observe that $\muhat\ind{1}(y_{1:H};\ystar) = \mu\ind{1}(y_{1:H};\ystar)$ and, for any $h < H$,
\begin{align}
\frac{\muhat\ind{1}(y_{1:h};\ystar)}{\mu_1(y_{1:h};\ystar)} 
&= \frac{1}{2\mu\ind{1}(y_h\mid{}y_{1:h-1};\ystar)} \\ 
&\in [1/(1+2\gamma_1), 1/(1-2\gamma_1)] \\ 
&\in [e^{-3\gamma_1}, e^{3\gamma_1}].
\end{align}
Fix $N>1$ and suppose that the claim holds for $\muhat\ind{N-1}$. Set $H = \HH(N)$ and $K = \HH(N-1)+2$. For any $h \in [H]$, if $h = (i+1)K$ for some $i$, then 
\begin{align}
&\muhat\ind{N}(y_{iK+1:(i+1)K}\mid{} y_{1:iK}; \ystar) \\ 
&= \frac{1}{2} \cdot \left(\prod_{h'=iK+2}^{(i+1)K-1} \muhat\ind{N-1}(y_{h'} \mid{} y_{iK+2:h'-1}; \ystar_{iK+2:(i+1)K-1})\right) \\ 
&\qquad\cdot 2\mu\ind{N}(y_{iK+1}\mid{}y_{1:iK};\ystar) \mu\ind{N}(y_{(i+1)K} \mid{} y_{1:(i+1)K-1};\ystar) \\ 
&= \mu\ind{N-1}(y_{iK+2:(i+1)K-1}; \ystar_{iK+2:(i+1)K-1}) \cdot \mu\ind{N}(y_{iK+1}\mid{}y_{1:iK};\ystar) \mu\ind{N}(y_{(i+1)K} \mid{} y_{1:(i+1)K-1};\ystar) \\ 
&= \mu\ind{N}(y_{iK+2:(i+1)K-1} \mid{} y_{1:iK+1}; \ystar) \cdot \mu\ind{N}(y_{iK+1}\mid{}y_{1:iK};\ystar) \mu\ind{N}(y_{(i+1)K} \mid{} y_{1:(i+1)K-1};\ystar) \\ 
&= \mu\ind{N}(y_{iK+1:(i+1)K} \mid{} y_{1:iK} ; \ystar).
\end{align}
Thus, it holds that $\muhat\ind{N}(y_{1:(i+1)K};\ystar) = \mu\ind{N}(y_{1:(i+1)K};\ystar)$, and in particular, $\muhat\ind{N}(y_{1:H};\ystar) = \mu\ind{N}(y_{1:H};\ystar)$. For any $iK+1 \leq h < (i+1)K$, we have
\begin{align}
\muhat\ind{N}(y_{1:h};\ystar) 
&= \mu\ind{N}(y_{1:iK};\ystar) \muhat\ind{N}(y_{iK+1:h}\mid{}y_{1:iK};\ystar) \\ 
&= \mu\ind{N}(y_{1:iK};\ystar) \cdot \frac{1}{2} \cdot \muhat\ind{N-1}(y_{iK+2:h};\ystar_{iK+2:(i+1)K-1}) \\ 
&\in [e^{-3\sum_{n=1}^{N-1}\gamma_n},e^{3\sum_{n=1}^{N-1}\gamma_n}] \cdot \mu\ind{N}(y_{1:iK};\ystar) \cdot \frac{1}{2} \cdot \mu\ind{N-1}(y_{iK+2:h};\ystar_{iK+2:(i+1)K-1}) \\ 
&\in [e^{-3\sum_{n=1}^{N}\gamma_n},e^{3\sum_{n=1}^{N}\gamma_n}] \cdot \mu\ind{N}(y_{1:iK};\ystar) \cdot \mu\ind{N}(y_{iK+1}\mid{}y_{1:iK};\ystar) \cdot \mu\ind{N-1}(y_{iK+2:h};\ystar_{iK+2:(i+1)K-1}) \\ 
&\in [e^{-3\sum_{n=1}^{N}\gamma_n},e^{3\sum_{n=1}^{N}\gamma_n}] \cdot \mu\ind{N}(y_{1:h};\ystar)
\end{align}
which completes the proof.
\end{proof}

To show that any $N$-particle myopic particle filtering method fails to cover $\mu\ind{N}(\cdot;\ystar)$ (for some choice of $\ystar$), we will show that there are events $\cF_N(\ystar)$ that are extremely likely under $\mu\ind{N}(\cdot;\ystar)$ (for any $\ystar$) but extremely unlikely under the execution of any $N$-particle myopic particle filtering method (for a uniformly random choice of $\ystar$). \cref{def:pf-events} formally defines these events. \cref{def:pf-ih} formalizes the claim about these events that we will prove inductively.

\begin{definition}[Key events for analysis]\label{def:pf-events}
Fix $N=1$ and $H = \HH(1)$. We define
\[\cF_1(\ystar) := \{o_{1:H}: \#\{1<h<H: o_h = \ystar_h\} \geq (1/2 + \gamma_1/2)(H-2)\}.\]
Fix $N>1$ and write $H := \HH(N)$ and $K := \HH(N-1)+2$. We define sets $\cF_N(\ystar)$ as follows:
\[\cF^\circ_N(\ystar) := \bigcap_i \left\{o_{1:H}: o_{iK+2:(i+1)K-1} \in \cF_{N-1}(\ystar_{iK+2:(i+1)K-1})\right\}\]
\[\cF^{\square}_N(\ystar) := \{o_{1:H}: \#\{i:o_{iK+1}=\ystar_{iK+1}\} \geq 1/2 + \gamma_N/2\}\]
\[\cF_N(\ystar) := \cF^\circ_N(\ystar) \cap \cF^\square_N(\ystar).\]
\end{definition}

\newcommand{\IH}{\mathsf{IH}}
\begin{definition}[Inductive hypothesis for analysis]\label{def:pf-ih}
Fix $N \geq 1$. For parameter $\delta_{N,1},\delta_{N,2} \in (0,1)$, we let $\IH_N(\delta_{N,1},\delta_{N,2})$ denote the following hypothesis. There are sets $\cF_N(\ystar) \subseteq \cA^{\HH(N)}$ with the following properties:
\begin{itemize}
    \item For all $\ystar \in \cA^{\HH(N)}$, it holds that
    \[\mu\ind{N}(\cF_N(\ystar); \ystar) \geq 1 - \delta_{N,1}\]
    \item For any $N$-particle particle filtering algorithm $\Alg$, it holds that
    \[\EE_{\ystar \sim \Unif(\cA^{\HH(N)})} \Pr_{o \sim \Alg}[o \in \cF_N(\ystar)] \leq \delta_{N,2}\]
    where the probability is over outputs $o$ of $\Alg$ executed on instance $\muhat\ind{N}(\cdot;\ystar)$.
\end{itemize}
\end{definition}

\subsubsection{Base Case for Analysis}

\begin{lemma}\label{lemma:base}
For any $\ystar \in \cA^{\HH(1)}$, it holds that
\[\mu\ind{1}(\cF_1(\ystar);\ystar) \geq 1 - \exp\left(-\gamma_1^2 \frac{\HH(1)-2}{16}\right).\]
Moreover, for any $1$-particle particle filtering algorithm $\Alg$,
\[\EE_{\ystar\sim\Unif(\cA^{\HH(1)})} \Pr_{o\sim\Alg}[o\in\cF_1(\ystar)] \leq \exp\left(-\gamma_1^2 \frac{\HH(1)-2}{12}\right)\]
where $\Alg$ is executed on instance $\muhat\ind{1}(\cdot;\ystar)$.
\end{lemma}

\begin{proof}
Let $X \sim \mu\ind{1}(\cdot;\ystar)$. Then $X_1,\dots,X_H$ are independent with $\Pr[X_h = \ystar_h] = 1/2+\gamma_1$ for each $h \in [H]$. It follows from a Chernoff bound that
\begin{align}
\mu\ind{1}(\cF_1(\ystar);\ystar)
&= \Pr[\#\{1<h<H: X_h=\ystar_h\} > (1/2+\gamma_1/2)H] \\ 
&\geq 1 - \exp(-\gamma_1^2(\HH(1)-2)/16)
\end{align}
which proves the first claim of the lemma. Next, fix any $1$-particle particle filtering algorithm $\Alg$. Sample $\ystar \sim \Unif(\cA^{\HH(1)})$ and let $\cS_{0},\dots,\cS_{H}$ be the stochastic sequence of multisets obtained by executing $\Alg$ on $\muhat\ind{1}(\cdot;\ystar)$. Let $o$ be the output of $\Alg$. Then $\cS_{h} = \{o_{1:h}\}$ for all $0 \leq h \leq H$. Since the law of $\cS_{0:h}$ is determined by $(\muhat\ind{1}_0(\cdot;\ystar),\dots,\muhat\ind{1}_h(\cdot;\ystar))$, it holds from the definition of $\muhat\ind{1}(\cdot;\ystar)$ that $(\cS_{0:h},\ystar_{1:h-1})$ is independent of $\ystar_h$ for each $1 < h < H$. Thus, $\ystar_h$ is independent of $\cS_{h}$ conditioned on $(\cS_{0:h-1},\ystar_{1:h-1})$, and so we have
\[\Pr[o_h = \ystar_h \mid{} \cS_{0:h-1},\ystar_{1:h-1}] = \frac{1}{2}\]
for all $1 < h < H$. The second claim of the lemma then follows from Azuma's inequality.
\end{proof}

\subsubsection{Inductive Step for Analysis}

\begin{lemma}\label{lemma:induction-mu}
Fix $N>1$. Assume that $\IH_{N-1}(\delta_{N-1,1},\delta_{N-1,2})$ holds. Suppose $\HH(N-1)\geq 2$. For any $\ystar \in \cA^{\HH(N)}$, it holds that
\[\mu\ind{N}(\cF_N(\ystar);\ystar) \geq 1 - \frac{\HH(N)}{\HH(N-1)} \delta_{N-1,1} - \exp\left(-\gamma_N^2 \frac{\HH(N)}{16\cdot \HH(N-1)}\right)\]
\end{lemma}

\begin{proof}
Set $H := \HH(N)$ and $K := \HH(N-1)+2$. Let $o \sim \mu\ind{N}(\cdot;\ystar)$. Then for each $i$, the law of $o_{iK+2:(i+1)K-1}$ is precisely $\mu\ind{N-1}(\cdot;\ystar_{iK+2:(i+1)K-1})$. Thus,
\begin{align} 
&\mu\ind{N}(\cF^\circ_N(\ystar);\ystar) \\ 
&\geq 1 - \sum_i \mu\ind{N}(o_{iK+2:(i+1)K-1} \not\in \cF_{N-1}(\ystar_{iK+2:(i+1)K-1})) \\ 
&= 1 - \frac{H}{K}\left(1 -  \mu\ind{N-1}(\cF_{N-1}(\ystar_{iK+2:(i+1)K-1});\ystar_{iK+2:(i+1)K-1})\right) \\ 
&\geq 1 - \frac{H}{K}\delta_{N-1,1} \\ 
&\geq 1 - \frac{\HH(N)}{\HH(N-1)} \delta_{N-1,1}.
\end{align}
Next, note that the random variables $(o_{iK+1})_i$ are independent and $\Pr[o_{iK+1}=\ystar_{iK+1}] = 1/2 + \gamma_N$ for each $i$. Thus,
\[\mu\ind{N}(\cF^\square_N(\ystar);\ystar)\geq 1 - \exp(-\gamma_N^2 (H/K)/8) \geq 1 - \exp(-\gamma_N^2 (\HH(N)/\HH(N-1))/16\]
by Hoeffding's inequality, so long as $\HH(N-1)\geq 2$. The lemma follows from the union bound.
\end{proof}

\begin{lemma}\label{lemma:induction-alg}
Fix $N>1$ and $\delta_{N-1,1},\delta_{N-1,2},\delta_{N,2} \in (0,1/2)$. Assume that $\IH_{N-1}(\delta_{N-1,1},\delta_{N-1,2})$ holds. Set $H := \HH(N)$ and $K := \HH(N-1)+2$. If 
    \[\gamma_N > 4\delta_{N-1,2} + \sqrt{\frac{8K\log(2/\delta_{N,2})}{H}} + \frac{8K}{H}\log(4/\delta_{N,2}),\]
    then for any $N$-particle particle filtering algorithm $\Alg$, it holds that
    \[\EE_{\ystar \sim \Unif(\cA^{\HH(N)})} \Pr_{o \sim \Alg}[o \in \cF_N(\ystar)] \leq \delta_{N,2}\]
    where $\Alg$ is executed on instance $\muhat\ind{N}(\cdot;\ystar)$.
\end{lemma}

\begin{proof}
Consider the random process induced by sampling $\ystar \sim \Unif(\cA^{\HH(N)})$ and executing $\Alg$ on $\muhat\ind{N}(\cdot;\ystar)$. Fix $i$ and condition on $\cS_{0:iK}$. Let $\cE^i$ be the event that $y_{iK+1} \neq y'_{iK+1}$ for some $y,y' \in \cS_{(i+1)K-1}$. For each $b \in \{0,1\}$, define process $\cT^{(0,b)},\dots,\cT^{(K-2,b)}$ where $\cT^{(0,b)} = \{\emptyset\}$ and, for each $h'>0$,
$\cT^{(h',b)}$ consists of the first $\leq N-1$ elements of \[\{y_{iK+2:h'+iK+1}: y_{1:h'+iK+1} \in \cS_{h' + iK + 1} \land y_{iK+1}=b \land y_{iK+2:h'+iK} \in \cT^{(h'-1,b)}\}\]
according to some arbitrary ordering of $\cA^{h'}$. Observe that $\cT^{(0,b)},\dots,\cT^{(K-2,b)}$ describes the execution of an $(N-1)$-particle myopic particle filtering algorithm on $\muhat\ind{N-1}(\cdot; \ystar_{iK+2:(i+1)K-1})$. To see why, we check the conditions of \cref{def:myopic}. First, it holds almost surely that $|\cT^{(h',b)}| \leq N-1$ for all $h'$. Second, by construction, if $y_{1:h'} \in \cT^{(h',b)}$ then $y_{1:h'-1} \in \cT^{(h'-1,b)}$. Third, for any $0 \leq h \leq K-2$, the law of $\cT^{(0,b)},\dots,\cT^{(h',b)}$ is determined by the law of $\cS_{0:h'+iK+1}$, which by definition of $\Alg$ is determined by $(\muhat\ind{N}_h(\cdot;\ystar): h \leq h'+iK+1)\}$, which by definition of $\muhat\ind{N}$ is determined by $(\muhat\ind{N-1}_{k}(\cdot;\ystar_{iK+2:(i+1)K-1}): 0 \leq k \leq h')$ (together with randomness independent of $\ystar_{iK+2:(i+1)K-1}$).

It follows from the inductive hypothesis that 
\[\Pr[\exists y \in \cT^{(K-2,b)} \cap \cF_{N-1}(\ystar_{iK+2:(i+1)K-1})] \leq \delta_{N-1, 2}.\]
In the event $\cE^i$, it holds almost surely that \[\{y_{iK+2:(i+1)K-1}: y_{1:(i+1)K-1} \in \cS_{(i+1)K-1}\} = \cT^{(K-2,0)} \cup \cT^{(K-2,1)}.\] Moreover, the output $o_{1:H}$ of $\Alg$ satisfies $o_{1:(i+1)K-1} \in \cS_{(i+1)K-1}$ almost surely, so
\begin{align}
&\Pr[o_{iK+2:(i+1)K-1} \in \cF_{N-1}(\ystar_{iK+2:(i+1)K-1}) \mid{} \cS_{0:iK}] \\
&\leq \Pr[\exists y_{1:(i+1)K-1}\in\cS_{(i+1)K-1}: y_{iK+2:(i+1)K-1} \in \cF_{N-1}(\ystar_{iK+2:(i+1)K-1}) \mid{} \cS_{0:iK}] \\ 
&\leq \Pr[\ol \cE^i \mid{} \cS_{0:iK}] + \Pr[\exists y \in \cT^{(K-2,0)} \cup \cT^{(K-2,1)}: y_{iK+2:(i+1)K-1} \in \cF_{N-1}(\ystar_{iK+2:(i+1)K-1}) \mid{} \cS_{0:iK}] \\ 
&\leq \Pr[\ol \cE^i \mid{} \cS_{0:iK}] + 2\delta_{N-1,2}.
\end{align}
Next, observe that in the event $\ol \cE^i$, it holds that $o_{iK+1} = \text{mode}\{y_{iK+1}: y \in \cS_{(i+1)K-1}\}$. Observe that $\cS_{(i+1)K-1}$ is independent of $\ystar_{iK+1}$, so that
\[\Pr[\text{mode}\{y_{iK+1}: y \in \cS_{(i+1)K-1}\} = \ystar_{iK+1} \mid{} \cS_{0:iK}] = \frac{1}{2}.\]
Thus, we have
\[\Pr[o_{iK+1} = \ystar_{iK+1} \mid{} \cS_{0:iK}] \leq \frac{1}{2} + \Pr[\cE^i \mid{} \cS_{0:iK}].\]
It follows that
\begin{align} 
&\Pr[o_{iK+1} = \ystar_{iK+1} \mid{} \cS_{0:iK}]  \\ 
&\leq \frac{1}{2} + \Pr[o_{iK+2:(i+1)K-1} \not \in \cF_{N-1}(\ystar_{iK+2:(i+1)K-1}) \mid{} \cS_{0:iK}] + 2\delta_{N-1,2}. 
\end{align}
Define random variables
\[W_i := \indic[o_{iK+1} = \ystar_{iK+1}] - \Pr[o_{iK+1} = \ystar_{iK+1} \mid{} \cS_{0:iK}].\]
Then $W_0,\dots,W_{H/K}$ is a martingale difference sequence under filtration induced by $(\cS_{0:iK})_{i=0}^{H/K-1}$, so for any $\delta \in (0,1)$,
\[\Pr\left[\sum_{i=0}^{H/K-1} W_i > \sqrt{8(H/K)\log(1/\delta)}\right] \leq \delta.\]
Thus, with probability at least $1-\delta$, it holds that 
\begin{align}
&\frac{K}{H}\sum_{i=0}^{H/K-1} \indic[o_{iK+1}=\ystar_{iK+1}]  \\
&\leq \frac{1}{2} + 2\delta_{N-1,2} + \sqrt{\frac{8K\log(1/\delta)}{H}} \\ 
&\qquad+ \frac{K}{H}\sum_{i=0}^{H/K-1} \Pr[o_{iK+2:(i+1)K-1} \not \in \cF_{N-1}(\ystar_{iK+2:(i+1)K-1}) \mid{} \cS_{0:iK}].
\end{align}
By Freedman's inequality, it holds with probability at least $1-\delta$ that 
\begin{align}
&\sum_{i=0}^{H/K-1} \Pr[o_{iK+2:(i+1)K-1} \not \in \cF_{N-1}(\ystar_{iK+2:(i+1)K-1}) \mid{} \cS_{0:iK}] \\ 
&\leq 2 \sum_{i=0}^{H/K-1} \indic[o_{iK+2:(i+1)K-1} \not \in \cF_{N-1}(\ystar_{iK+2:(i+1)K-1})] + 8\log(2/\delta).
\end{align}
We conclude that with probability at least $1-2\delta$,
\begin{align}
&\frac{K}{H}\sum_{i=0}^{H/K-1} \indic[o_{iK+1}=\ystar_{iK+1}]  \\
&\leq \frac{1}{2} + 2\delta_{N-1,2} + \sqrt{\frac{8K\log(1/\delta)}{H}} + \frac{8K}{H}\log(2/\delta)\\ 
&\qquad+ \frac{2K}{H}\sum_{i=0}^{H/K-1} \indic[o_{iK+2:(i+1)K-1} \not \in \cF_{N-1}(\ystar_{iK+2:(i+1)K-1})].
\end{align}
Condition on this event. If we moreover have $o_{1:H} \in \cF_N(\ystar)$, then we have both
\[\frac{K}{H}\sum_{i=0}^{H/K-1} \indic[o_{iK+1}=\ystar_{iK+1}] \geq \frac{1}{2} + \frac{\gamma_N}{2}\]
and
\[\frac{K}{H}\sum_{i=0}^{H/K-1} \indic[o_{iK+2:(i+1)K-1} \not \in \cF_{N-1}(\ystar_{iK+2:(i+1)K-1})] = 0.\]
But then 
\[\frac{\gamma_N}{2} \leq 2\delta_{N-1,2} + \sqrt{\frac{8K\log(1/\delta)}{H}} + \frac{8K}{H}\log(2/\delta).\]
Set $\delta := \delta_{N,2}/2$. Then this bound contradicts the lemma assumption, so in fact $o_{1:H} \not \in \cF_N(\ystar)$ almost surely (under the preceding condition). It follows that $\Pr[o_{1:H} \in \cF_N(\ystar)] \leq 2\delta = \delta_{N,2}$ as claimed.
\end{proof}

\subsubsection{Completing the Induction}

Fix some $N$. Define $\HH(n) := N^{4n}$ for all $n$. Define $\gamma_n := 1/N$ for all $n$.

\begin{lemma}\label{lemma:pf-induction}
It holds that
\[\mu\ind{N}(\cF_N(\ystar);\ystar) \geq 1 - \exp(-\Omega(N^2))\]
for all $\ystar$, and
\[\EE_{\ystar\sim\Unif(\cA^{\HH(N)})} \Pr_{o\sim \Alg}[o\in\cF_N(\ystar)] \leq 4\exp(-N).\]
\end{lemma}

\begin{proof}
We have from \Cref{lemma:base,lemma:induction-mu} that
\[\mu\ind{N}(\cF_N(\ystar);\ystar) \geq 1 - \sum_{n=1}^N \frac{\HH(N)}{\HH(n-1)} \exp(-\gamma_n^2 \cdot \Omega(\HH(n)/\HH(n-1))) \geq 1-N^{4N+1}\exp(-\Omega(N^2)).\]
Next, observe that if we set $\delta_{n,2} := 4\exp(-N)$ then the condition of \Cref{lemma:induction-alg} is satisfied for each $n$, so long as $N$ exceeds a sufficiently large constant. Thus, applying \Cref{lemma:induction-alg} together with induction on $N$ (with \Cref{lemma:base} providing the guarantee for $N=1$) proves that 
\[\EE_{\ystar\sim\Unif(\cA^{\HH(n)})} \Pr_{o\sim \Alg}[o\in\cF_n(\ystar)] \leq 4\exp(-N)\]
for each $n \leq N$ as needed.
\end{proof}

\begin{proof}[Proof of \Cref{thm:pf-lb}]
Immediate from \Cref{lemma:pf-induction,lemma:pf-construction-accuracy}, together with the fact that the $N$-particle construction has horizon $H = \HH(N) = N^{4N}$, and hence $N \geq \log(H)/(4\log\log(H))$.
\end{proof}

\section{Additional Experimental Details}
In this section, we provide further details regarding our experimental setup.

\subsection{Additional Evaluation Approaches}
\label{sec:additional-evaluations}
First, we give additional context and motivation for our choice of the evaluation metric of logprob-discrepancy introduced in \cref{sec:prompt-switching}. 

One approach to evaluate SMC, which has been used in prior work, is to use the fact that the value $\What_H$ computed in \cref{alg:smc} satisfies $\E[\What_H] = Z$ and moreover the difference $\log Z - \E[\log \What_H]$ is an upper bound on the KL divergence $\Dkl{\E[\nuhat_H]}{\pistar_H}$ between the output distribution of SMC, $\E[\nuhat_H]$, and the target distribution $\pistar_H$ (Proposition 5 of \citet{zhao2024twisted}).\footnote{Further, in \cref{lem:SMC-TV} we show that the difference $|\What_H - Z|$ upper bounds the \emph{total variation distance} $\Dtv{\E[\nuhat_H]}{\pistar_H}$.}  %
However, there are several pitfalls with this approach:  first, this is only an \emph{upper bound} and is not guaranteed to be tight; further, it only applies to SMC and so does not allow us to compare the performance of SMC to other baselines. Finally, even for the relatively simple setting in which we operate, the number $N$ of particles is too small to get actual \emph{distributional} closeness between $\E[\nuhat_H]$ and the target $\pistar_H$ (for instance, the action-level coverage is too large). %

Thus, we aimed to use a weaker notion of distributional closeness that we can both (a) measure efficiently and (b) hope to have some bearing on the quality of samples produced by SMC. The $\MP$-logprob discrepancy certainly satisfies (a). Moreover, since its error is one-sided in nature (i.e., $\mathrm{LPD}_{\MP}(\nu, \nu) = 0$) and it takes on a  positive value in all of our experiments, we believe it is a good candidate in the context of the prompt switching task.

\subsection{Details for \cref{sec:exp-prm-accuracy}}
\label{sec:exp-prm-accuracy-details}
For the prompt switching task in \cref{sec:exp-prm-accuracy}, the prompts $\promptref,\prompt\^i$ were all generated with GPT-5.2 (in particular, the $\text{gpt-5.2\_2025-12-11}$ release). 

To estimate $\Dkl{\pistar_h}{\pihat_h\^i}$, we write
\begin{align}
\Dkl{\pistar_h}{\pihat_h\^i} &= \E_{a_{1:h} \sim \pistar_h} \left[ \log \frac{\pistar_h(a_{1:h})}{\pihat_h\^i(a_{1:h})} \right] \nonumber\\
&= \E_{a_{1:h} \sim \pistar_h} \left[ (1-h/H) \cdot \alpha \cdot \log \frac{\pistar(a_{1:h})}{\pi\^i(a_{1:h})} \right] + \log \Zhat\^i_h\label{eq:kldiv-expand},
\end{align}
where 
\begin{align}
\Zhat\^i_h := \E_{\piref}[\Vhat\^i(a_{1:h})] = \E_{\pistar} \left[ \left(\frac{\pi\^i_h(a_{1:h})}{\pistar_h(a_{1:h})} \right)^{(1-h/H) \cdot \alpha}\right].\label{eq:zhat-expand}
\end{align}
 We estimate the first term in \cref{eq:kldiv-expand} via Monte Carlo rollouts from $\pistar_h$, and estimate $\log \Zhat\^i_h$ via Monte Carlo rollouts from $\pistar$, per \cref{eq:zhat-expand} (each with 1000 samples). %

 Notice that our estimate of of $\Zhat_h\^i$ is likely very noisy, since there is no logarithm inside the expectation. When we repeated the experiments without adding (our estimate of) the term $\log \Zhat_h\^i$ at all in \cref{eq:kldiv-expand}, we observed similar results. Thus, the noise in our estimate of $\Zhat_h\^i$ seems unlikely to have affected our conclusions. %

 \paragraph{Details for the figures} In \cref{fig:inv-phi-intro}, we use the PRM defined in \cref{eq:vhat-interp} with $\alpha = 2$. Moreover, the KL divergences are measured at step $h = 32$ (recall that $H = 64$); we observed similar results with other values of $h$. The experiments were run on 32GB NVIDIA V100 GPUs. The Pearson correlation coefficient for the points plotted in \cref{fig:inv-phi-intro} is 0.81. 

 In \cref{fig:invphi-alpha2-all}, we plot additional runs of the same experiment as in \cref{fig:inv-phi-intro}, with 3 additional different choices of the dataset (i.e., different choices of $\promptref$ and $\prompt\^i$). Finally, in \cref{fig:invphi-alpha1-all}, we repeat the same experiment for all 4 datasets but with $\alpha = 1$. 

 \subsection{Details for \cref{sec:exp-prm-action}}
 \label{sec:exp-prm-action-details}
In \cref{fig:actcov-all}, we repeat the experiment in \cref{fig:mean-ratio-intro} with 4 additional choices for the dataset (i.e., choices of $\promptref$ and $\prompt\^i$, $i \in [k]$). The Pearson correlation coefficients $r$ are shown in the subcaptions of each subfigure. The experiments were run on 32GB NVIDIA V100 GPUs.

\subsection{SMC vs.\ Best-of-$N$ on AIME}\label{sec:smc-bon-aime-details}
In our experiments on math problems in \cref{sec:smc-math-experiments}, we used Qwen2.5-1B-Instruct as the base model $\piref$ and Qwen2.5-Math-PRM-7B as the PRM $\Vhat$. For both SMC and Best-of-$N$, we used $N=32$ particles/samples. We sampled from the base model with temperature $0.8$, and used systematic resampling in the implementation of SMC. We generated for a maximum of 3072 tokens for AIME and a maximum of 2048 tokens for Math500. 

As opposed to implementing SMC with $\MA$ being the space of tokens (as in our experiments on prompt switching), we took $\MA$ to be the space of \emph{blocks} of some number $B$ of tokens; doing so significantly improves the efficiency of SMC, as calls to the PRM are the main computational bottleneck. We set $B = 64$ for Math500 and $B = 128$ for AIME. Moreover, we observed that the reward model $\Vhat(a_{1:h})$ often gave unreliable answers when passed as input a partial completion $a_{1:h}$ which was truncated in the middle of a sentence or mathematical expression. Thus, we in fact truncated all blocks at the last instance of a delimiter character (e.g., new line, period). 

\begin{figure}[ht]
\centering
\begin{subfigure}[b]{0.35\textwidth}
\centering
\includegraphics[width=\textwidth]{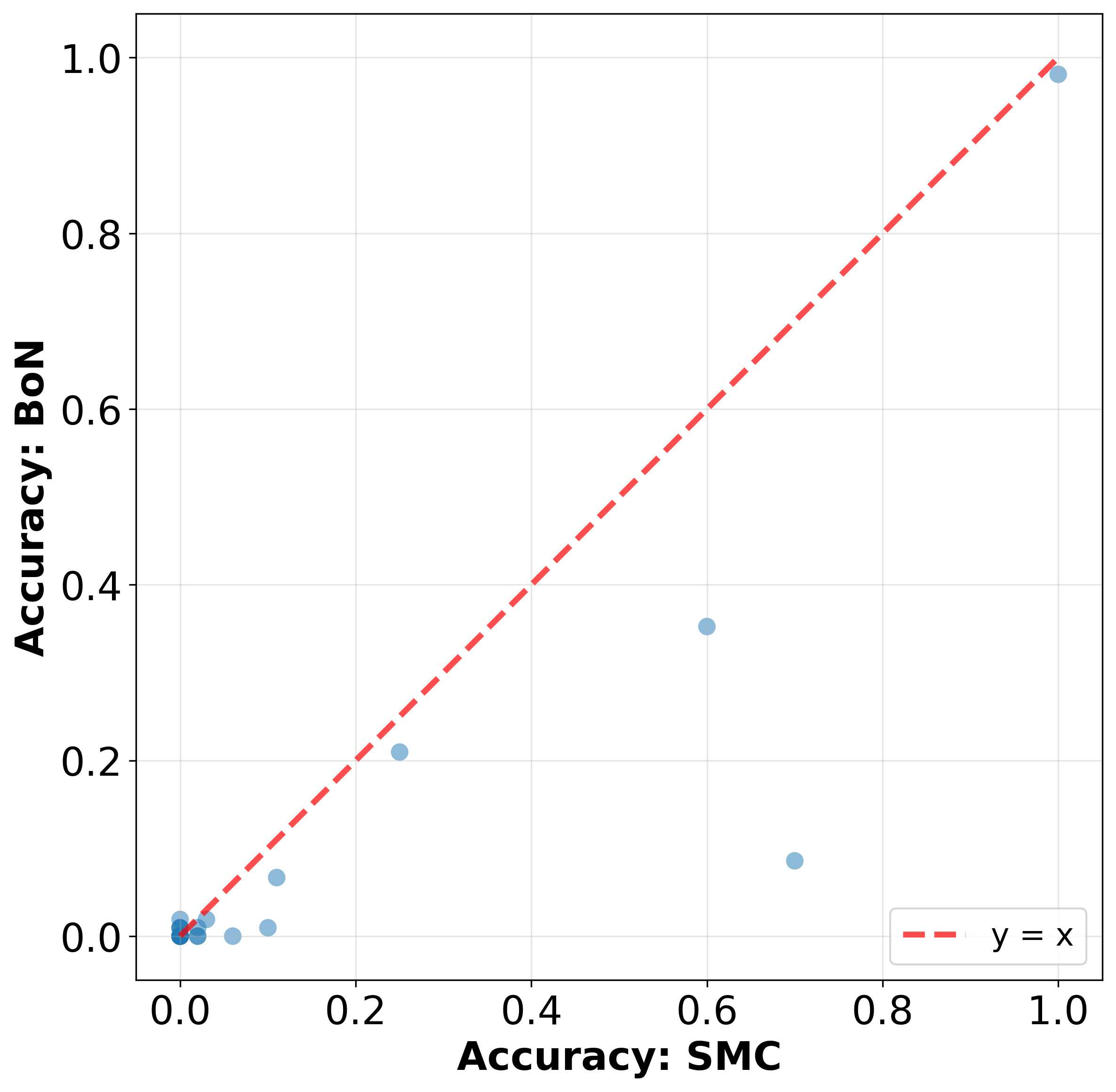}
\caption{AIME 2024}
\label{fig:smc-vs-bon-aime24}
\end{subfigure}
\begin{subfigure}[b]{0.35\textwidth}
\centering
\includegraphics[width=\textwidth]{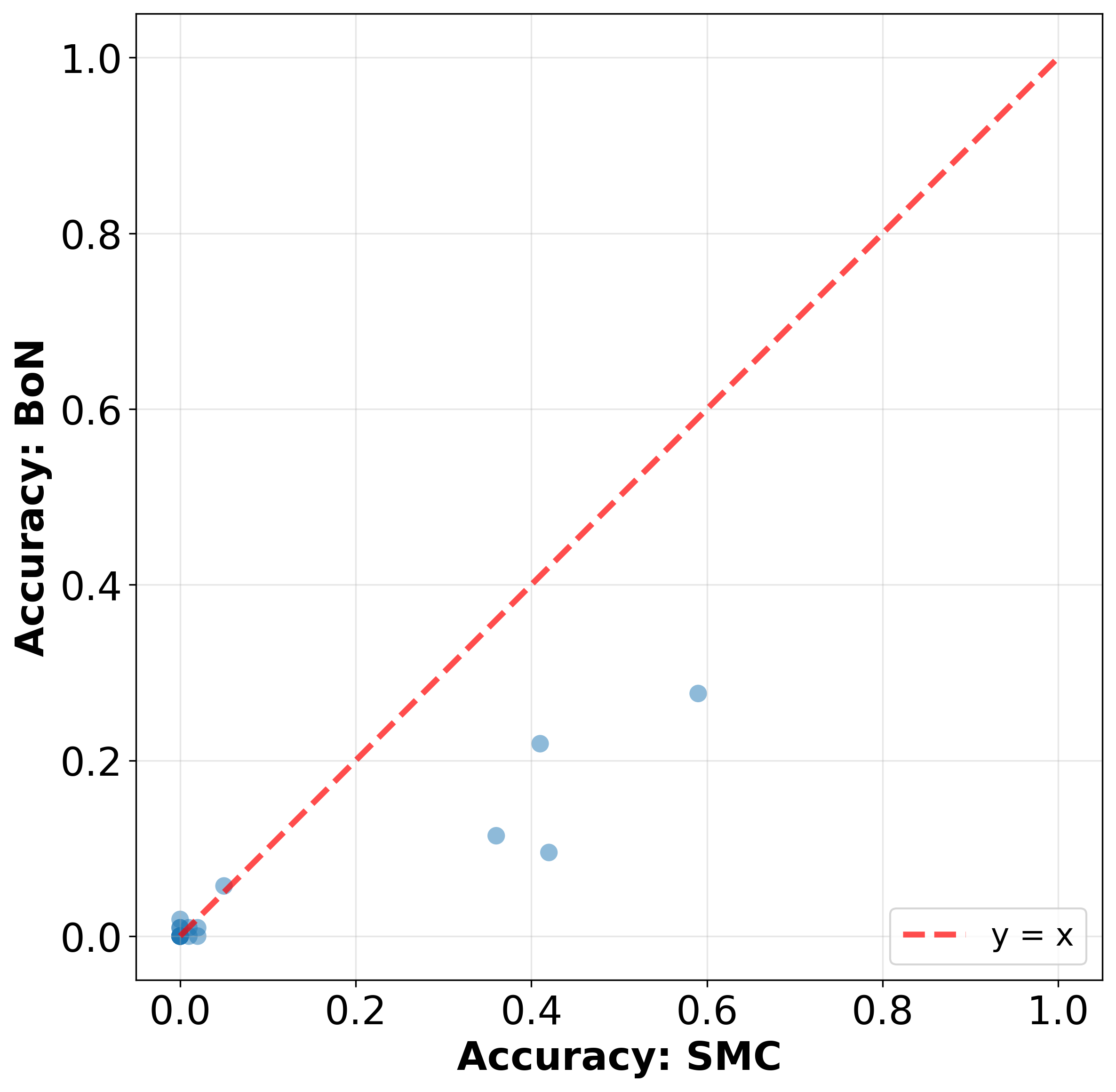}
\caption{AIME 2025}
\label{fig:smc-vs-bon-aime25}
\end{subfigure}
\caption{Performance of SMC vs.\ Best-of-$N$ on AIME problems (each point is a different problem). Similar to \cref{fig:smc-vs-bon}, the majority of points lie below the line $y=x$, indicating that SMC improves performance over Best-of-$N$ on most problems.}
\label{fig:smc-vs-bon-aime}
\end{figure}

\begin{figure}[ht]
\centering
\begin{subfigure}[b]{0.3\textwidth}
\centering
\includegraphics[width=\textwidth]{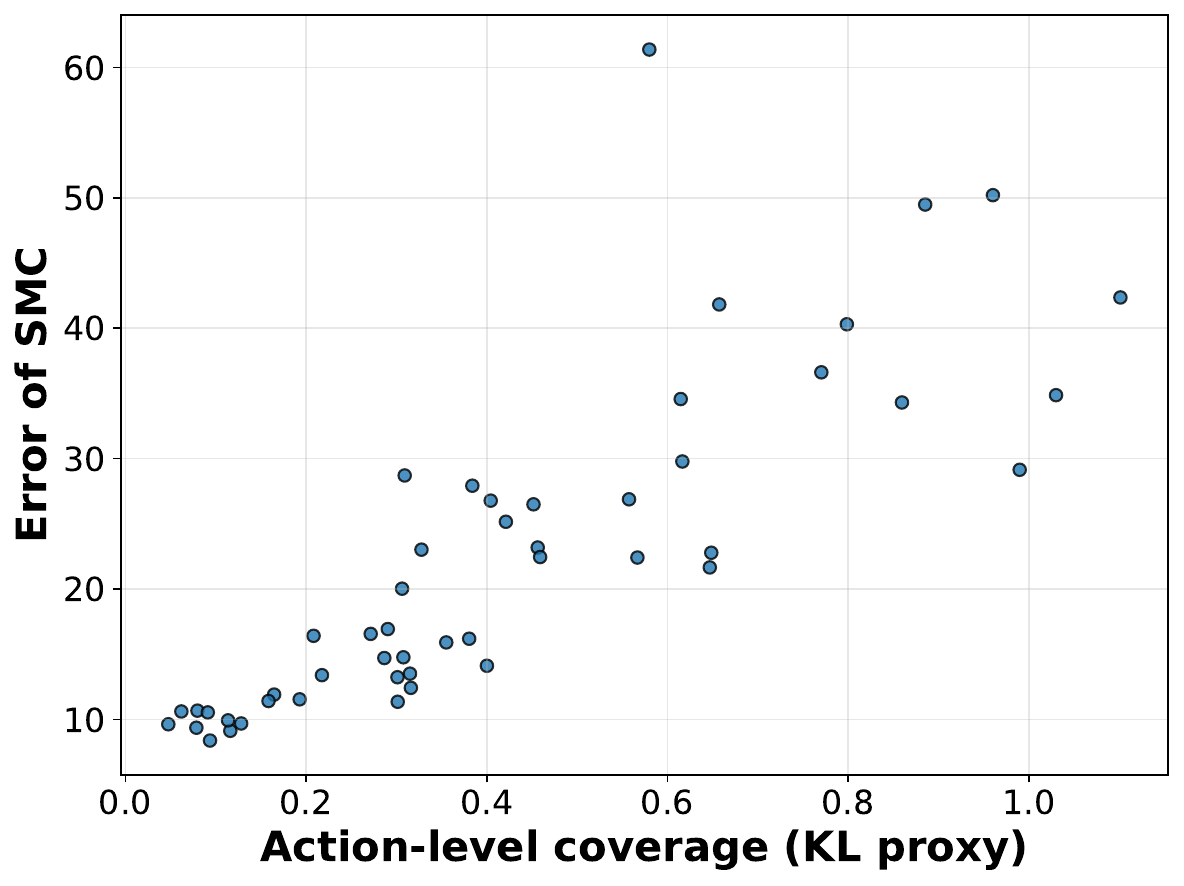}
\caption{$r = 0.83$}
\label{fig:actcov-1}
\end{subfigure}\hfill
\begin{subfigure}[b]{0.3\textwidth}
\centering
\includegraphics[width=\textwidth]{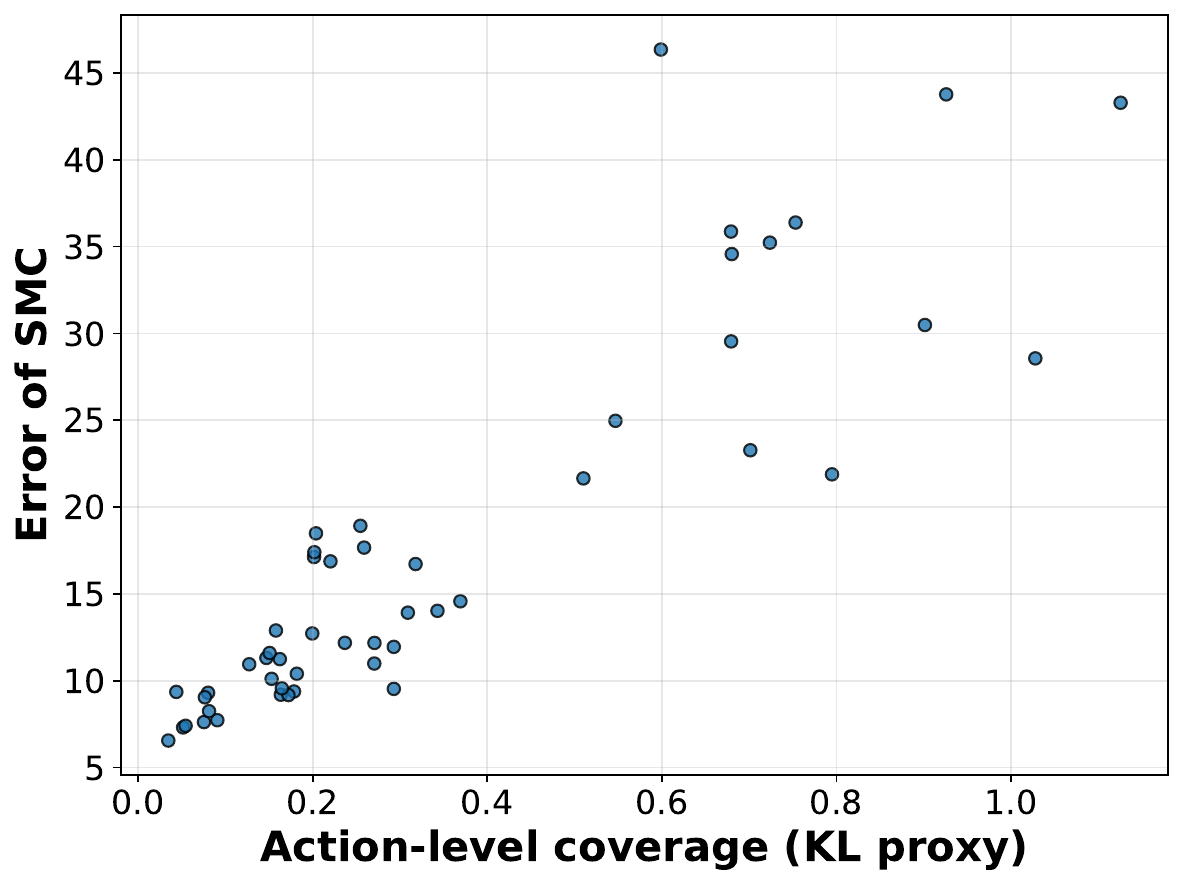}
\caption{$r = 0.89$}
\label{fig:actcov-2}
\end{subfigure}
\begin{subfigure}[b]{0.3\textwidth}
\centering
\includegraphics[width=\textwidth]{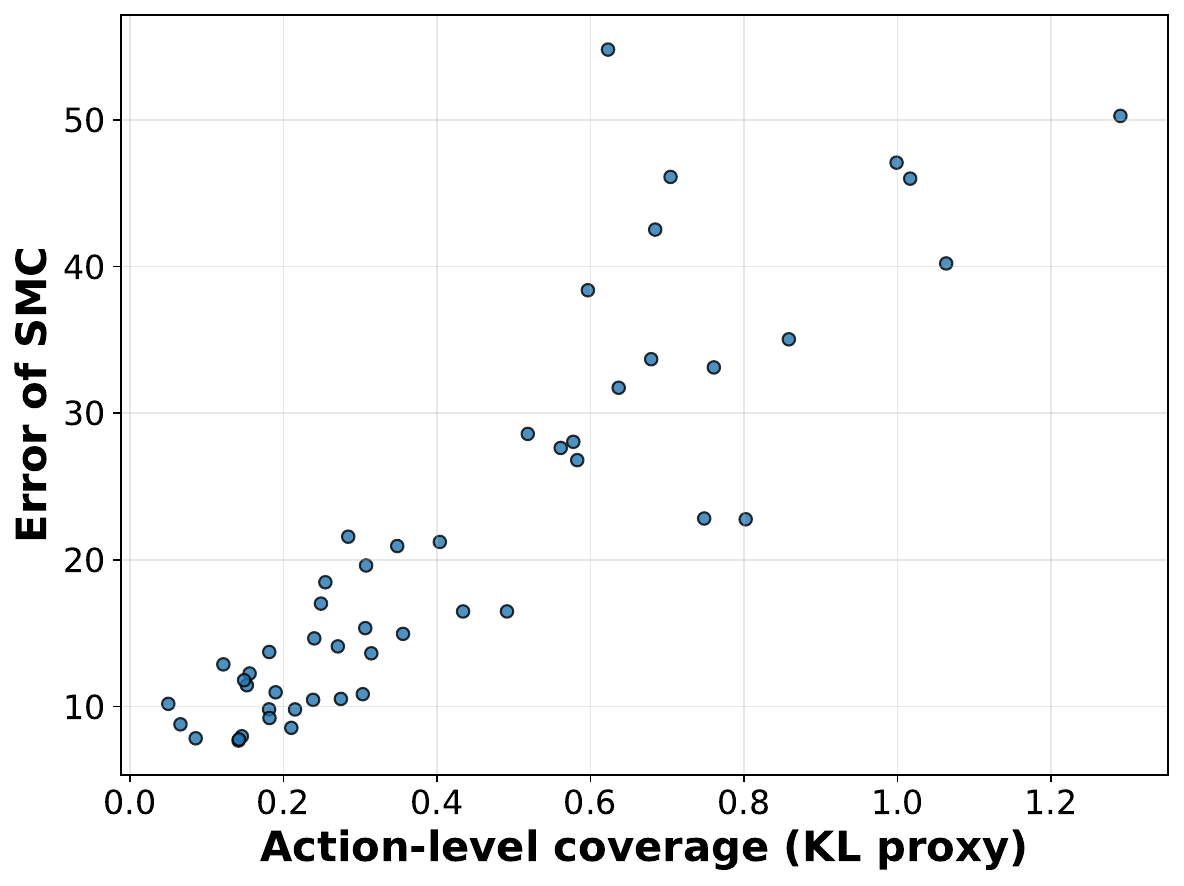}
\caption{$r = 0.88$}
\label{fig:actcov-3}
\end{subfigure}\hfill
\begin{subfigure}[b]{0.3\textwidth}
\centering
\includegraphics[width=\textwidth]{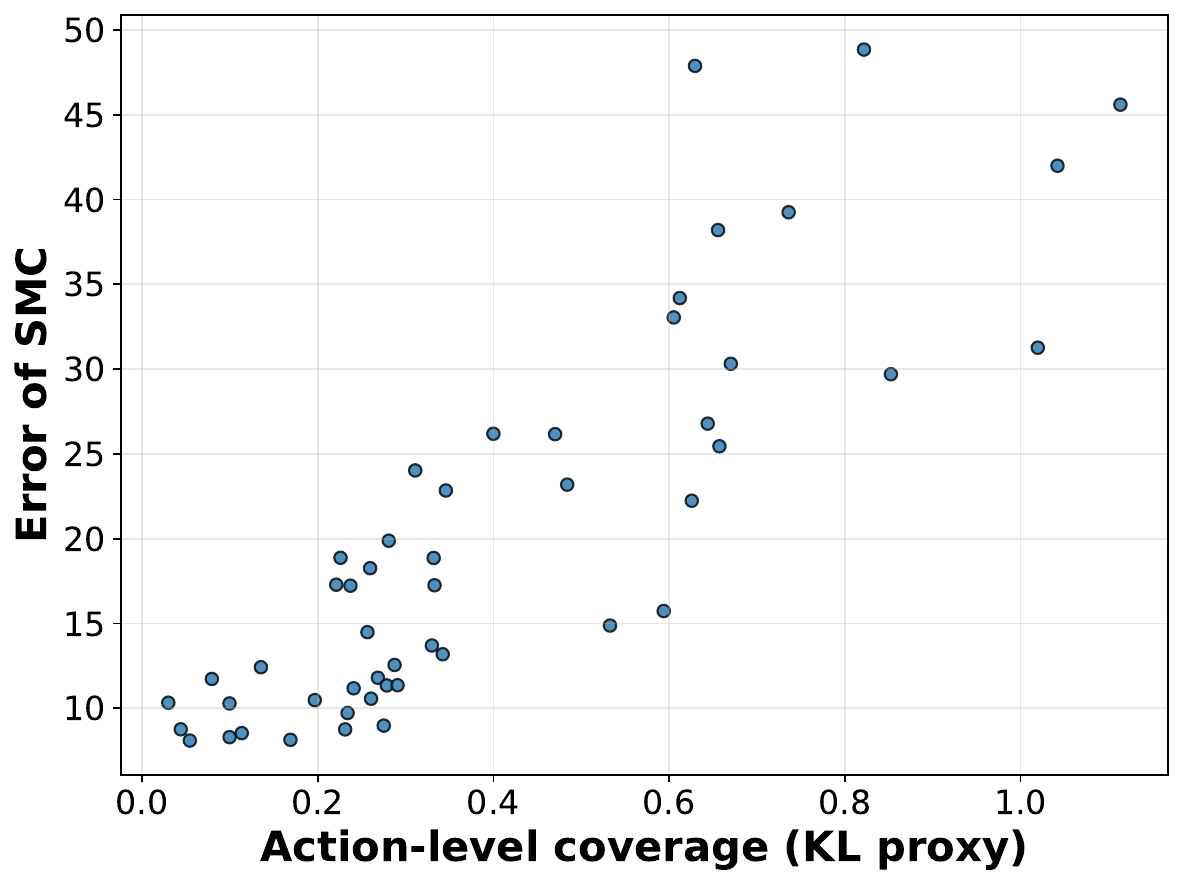}
\caption{$r = 0.86$}
\label{fig:actcov-4}
\end{subfigure}
\caption{Additional instantiations of the experiment shown in \cref{fig:mean-ratio-intro}, with different datasets (generated by GPT-5.2). Captions show Pearson correlation coefficients.}
\label{fig:actcov-all}
\end{figure}

\begin{figure}[ht]
\centering
\begin{subfigure}[b]{0.3\textwidth}
\centering
\includegraphics[width=\textwidth]{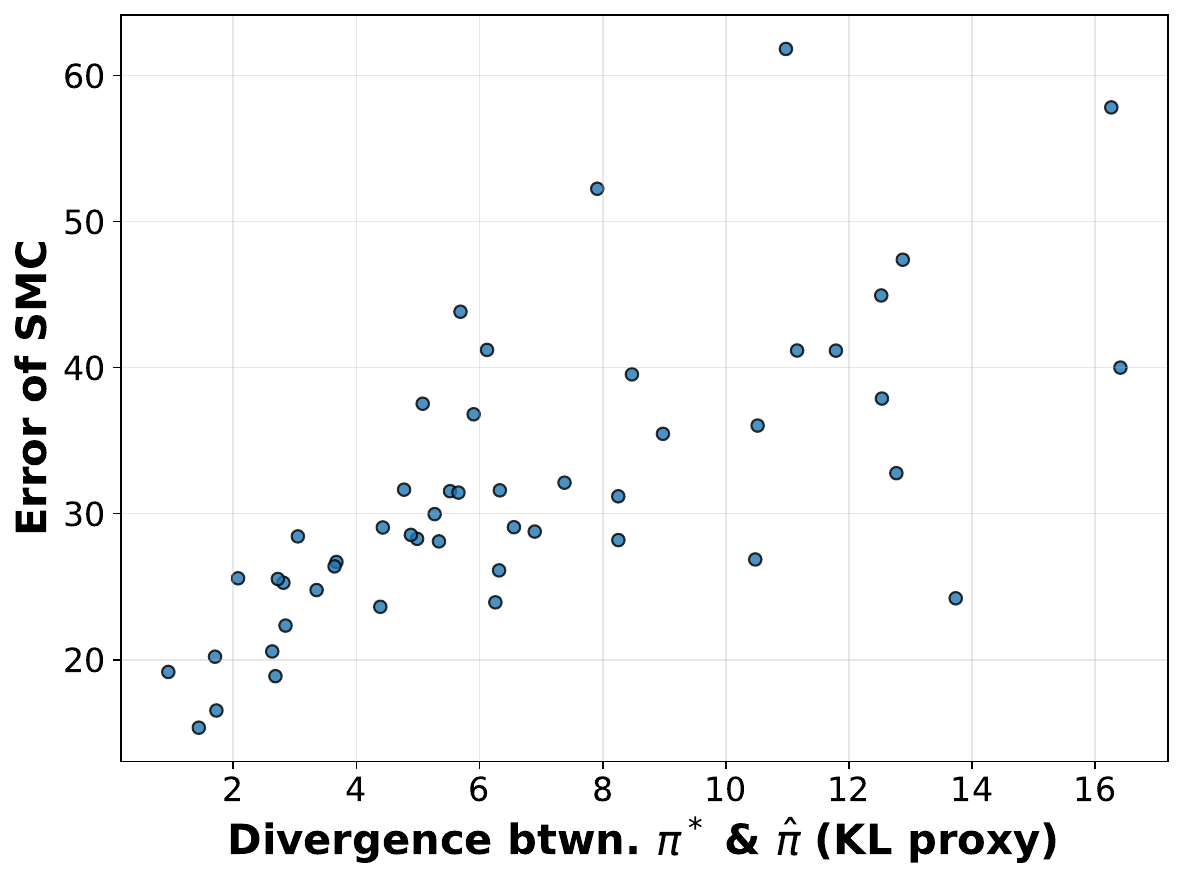}
\caption{$r = 0.70$}
\label{fig:invphi-alpha2-1}
\end{subfigure}
\begin{subfigure}[b]{0.3\textwidth}
\centering
\includegraphics[width=\textwidth]{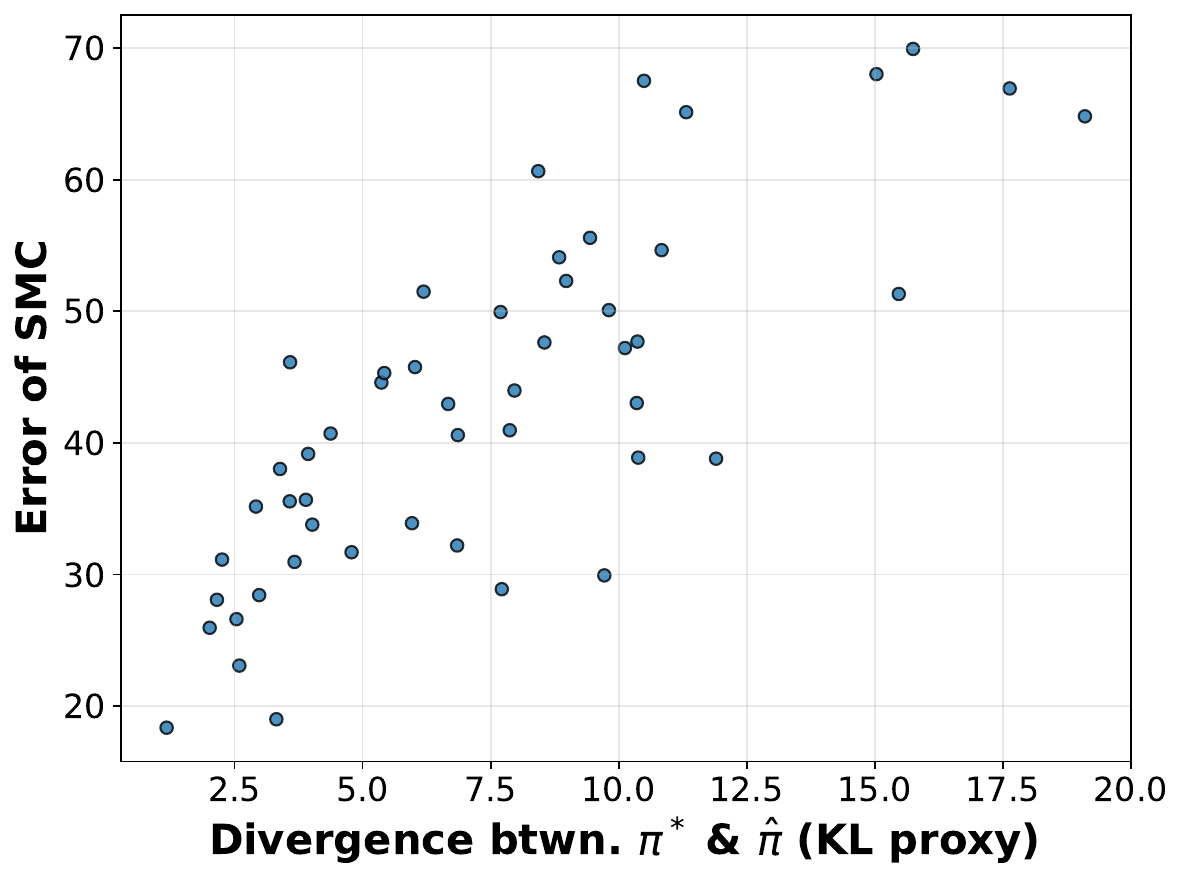}
\caption{$r = 0.79$}
\label{fig:invphi-alpha2-2}
\end{subfigure}
\begin{subfigure}[b]{0.3\textwidth}
\centering
\includegraphics[width=\textwidth]{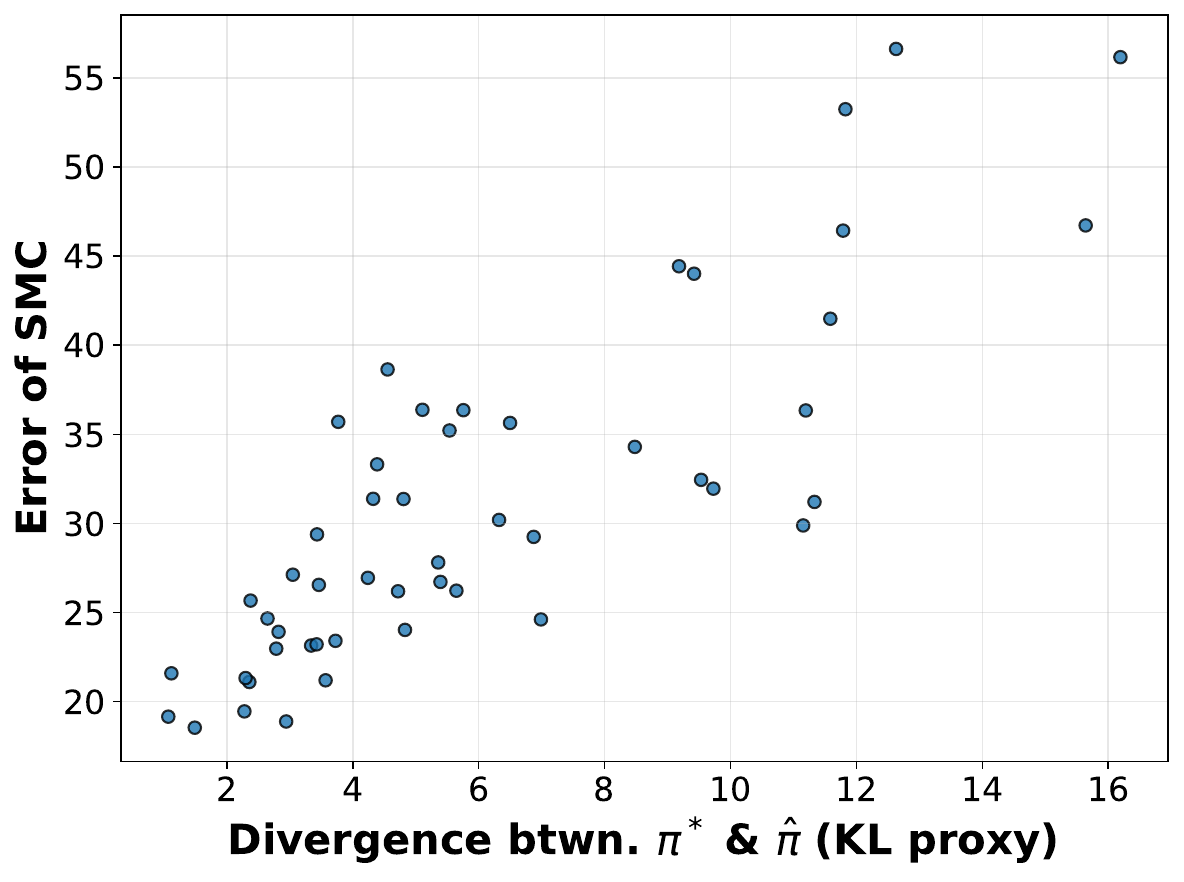}
\caption{$r = 0.83$}
\label{fig:invphi-alpha2-3}
\end{subfigure}
\caption{Additional instantiations of the experiment shown in \cref{fig:inv-phi-intro}, with different datasets (generated by GPT-5.2). Captions show Pearson correlation coefficients.}
\label{fig:invphi-alpha2-all}
\end{figure}

\begin{figure}[ht]
\centering
\begin{subfigure}[b]{0.3\textwidth}
\centering
\includegraphics[width=\textwidth]{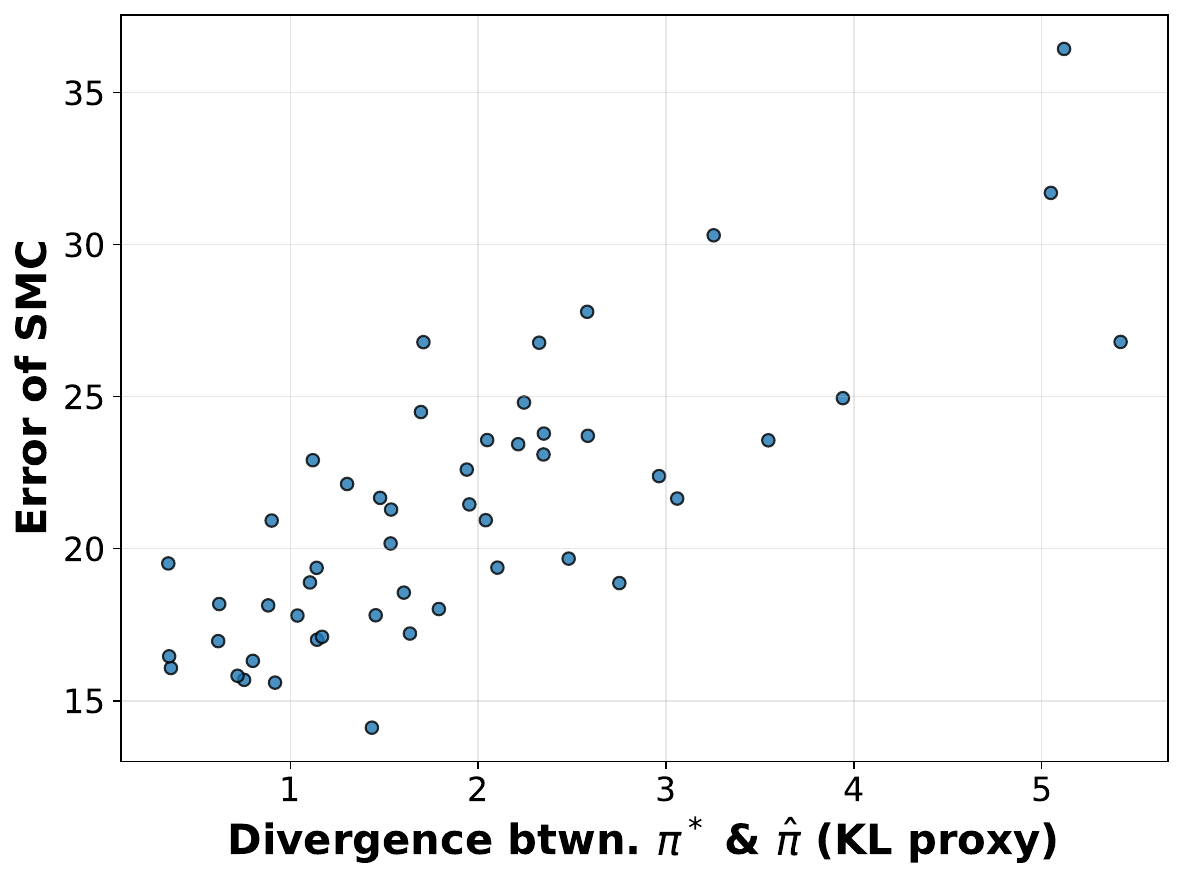}
\caption{$r = 0.78$}
\label{fig:invphi-alpha1-0}
\end{subfigure}
\begin{subfigure}[b]{0.3\textwidth}
\centering
\includegraphics[width=\textwidth]{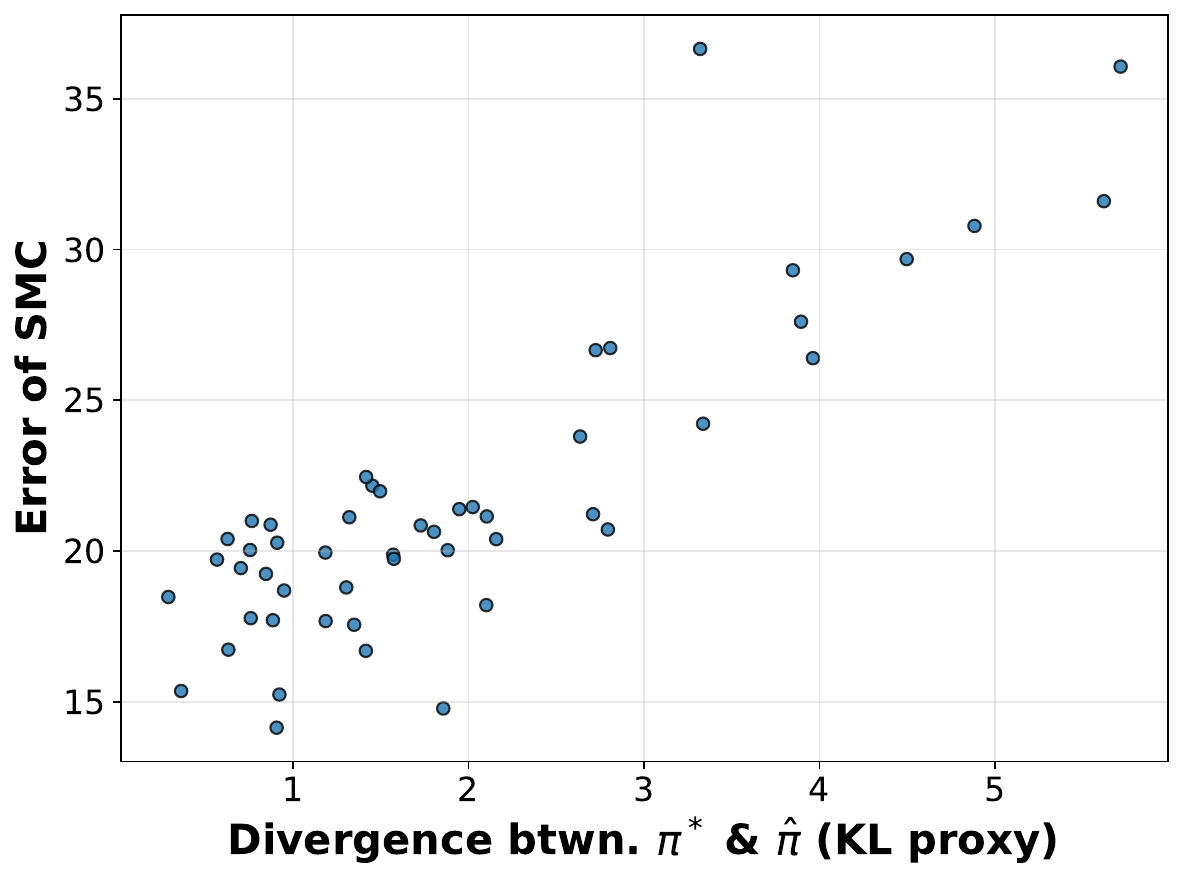}
\caption{$r = 0.85$}
\label{fig:invphi-alpha1-1}
\end{subfigure}
\begin{subfigure}[b]{0.3\textwidth}
\centering
\includegraphics[width=\textwidth]{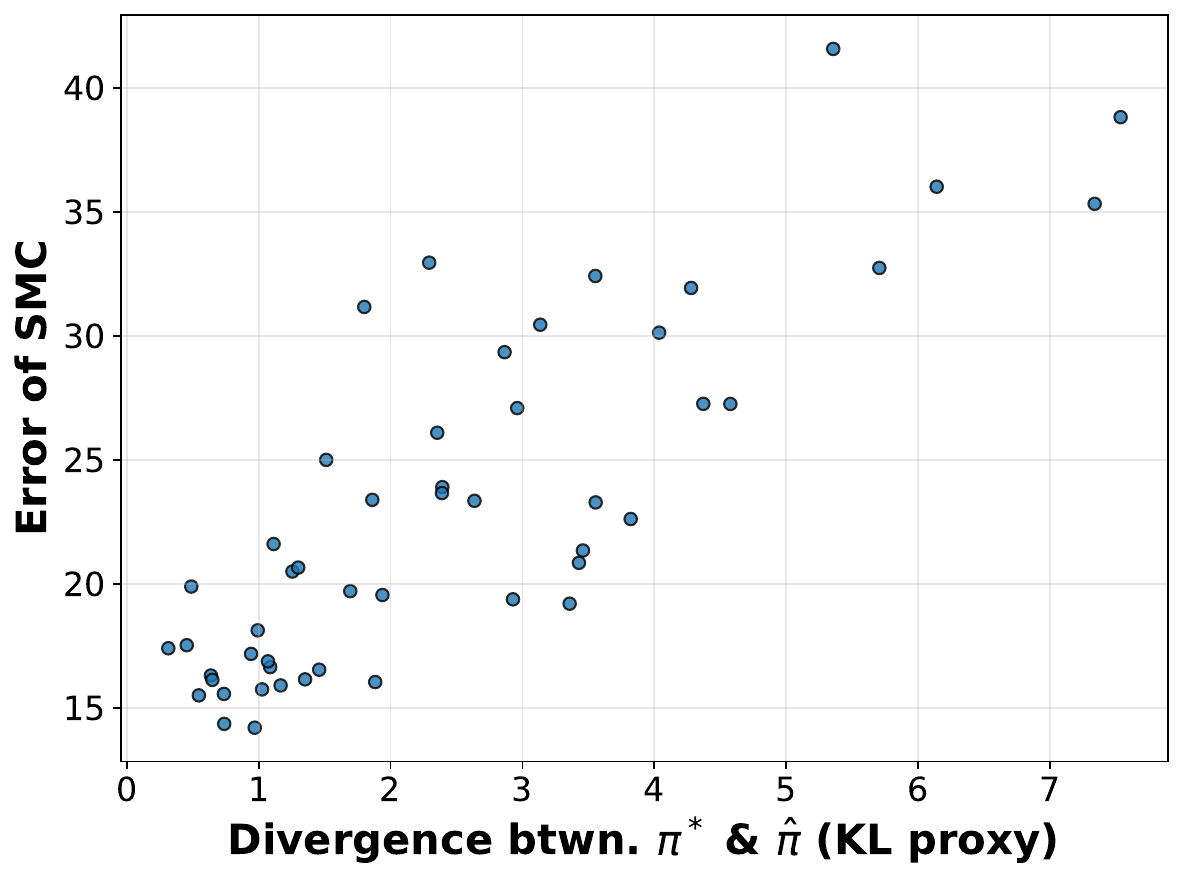}
\caption{$r = 0.83$}
\label{fig:invphi-alpha1-2}
\end{subfigure}
\begin{subfigure}[b]{0.3\textwidth}
\centering
\includegraphics[width=\textwidth]{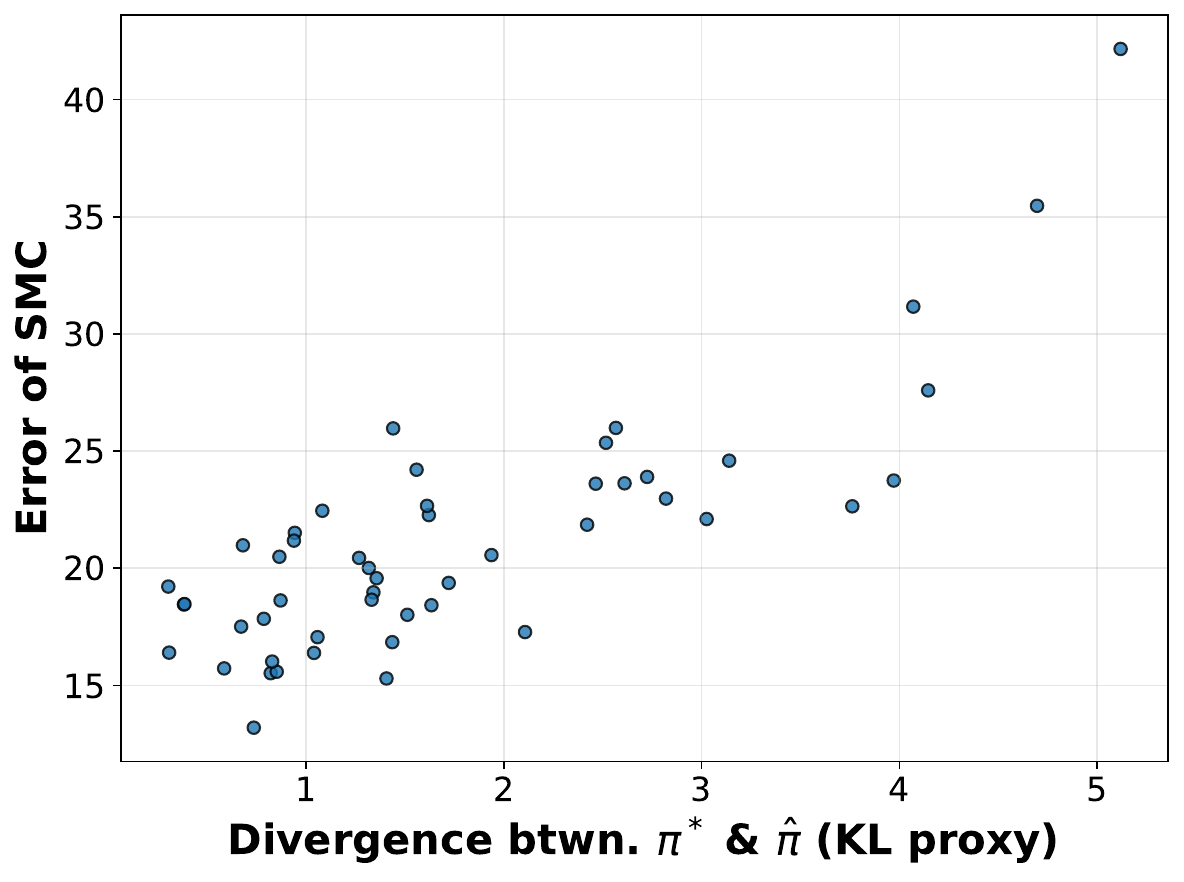}
\caption{$r = 0.82$}
\label{fig:invphi-alpha1-3}
\end{subfigure}
\caption{Additional instantiations of the experiment shown with the 4 datasets in \cref{fig:inv-phi-intro,fig:invphi-alpha2-all}, except with $\alpha = 1$. Captions show Pearson correlation coefficients.}
\label{fig:invphi-alpha1-all}
\end{figure}

\end{document}